\documentclass{article}

\usepackage{microtype}
\usepackage{graphicx}
\usepackage{subfigure}
\usepackage{booktabs} 

\usepackage[accepted]{icml2024}

\usepackage{amsmath}
\usepackage{amssymb}
\usepackage{mathtools}
\usepackage{amsthm}
\usepackage{array}
\usepackage{color, colortbl}
\usepackage{lipsum}
\usepackage{cancel}
\usepackage[capitalize,noabbrev]{cleveref}

\usepackage{bbm}
\def\1{\mathbbm{1}}
\def\d{{\rm d}}

\def\eps{\varepsilon}
\def\XX{\mathcal{X}}

\def\FF{\mathcal{F}}

\def\RR{\mathcal{R}}
\def\TT{\mathcal{T}}

\def\W{\mathcal{W}}
\def\kde{{\rm KDE}}
\def\R{\mathbb{R}}
\def\m{\boldsymbol{m}}
\DeclareMathOperator{\lgst}{{\rm logistic}}
\DeclareMathOperator{\relu}{{\rm ReLU}}
\DeclareMathOperator{\svm}{{\rm svm}}
\DeclareMathOperator*{\expt}{\mathbb{E}}
\DeclareMathOperator{\ranges}{\bf Ranges}
\DeclareMathOperator{\range}{\bf range}

\def\t{\boldsymbol{t}}
\def\T{\boldsymbol{T}}

\DeclareMathOperator*{\pr}{\mathbb{P}}

\def\vc{{\rm VCdim}}

\def\L{\mathcal{L}}
\def\u{\boldsymbol{u}}
\def\K{\mathsf{K}}
\def\s{\boldsymbol{s}}
\theoremstyle{plain}
\newtheorem{theorem}{Theorem}[section]

\newtheorem{lemma}[theorem]{Lemma}
\newtheorem{corollary}[theorem]{Corollary}
\theoremstyle{definition}
\newtheorem{definition}[theorem]{Definition}

\theoremstyle{remark}
\newtheorem{remark}[theorem]{Remark}

\usepackage[textsize=tiny]{todonotes}

\icmltitlerunning{No Dimensional Sampling Coresets for Classification}

\begin{document}

\twocolumn[
\icmltitle{No Dimensional Sampling Coresets for Classification}




\begin{icmlauthorlist}
\icmlauthor{Meysam Alishahi}{Utah}
\icmlauthor{Jeff M. Phillips}{Utah,Leipzig}
\end{icmlauthorlist}

\icmlaffiliation{Utah}{Kahlert School of Computing, University of Utah, Salt Lake City, Utah, USA}
\icmlaffiliation{Leipzig}{visiting ScaDS.AI, University of Leipzig and MPI for Math in the Sciences, Leipzig, Germany}

\icmlcorrespondingauthor{Meysam Alishahi}{meysam.alishahi@utah.edu}
\icmlcorrespondingauthor{Jeff M. Phillips}{jeffp@cs.utah.edu}

\icmlkeywords{Machine Learning, ICML}

\vskip 0.3in
]



\printAffiliationsAndNotice{\icmlEqualContribution} 

\begin{abstract}
We refine and generalize what is known about coresets for classification problems via the sensitivity sampling framework. Such coresets seek the smallest possible subsets of input data, so one can optimize a loss function on the coreset and ensure approximation guarantees with respect to the original data.  Our analysis provides the first no dimensional coresets, so the size does not depend on the dimension.  Moreover, our results are general, apply for distributional input and can use iid samples, so provide sample complexity bounds, and work for a variety of loss functions.  A key tool we develop is a Radamacher complexity version of the main sensitivity sampling approach, which can be of independent interest.  
\end{abstract}

\section{Introduction}
In machine learning, coresets~\citep{Phi16} are small subsets of input data that act as proxy for the full set while \emph{guaranteeing} not to deviate much in accuracy.  By reducing data size, they 
improve scalability and efficiency.   
A common approach for constructing coresets involves sampling data points with probabilities proportional to their so-called \emph{sensitivity score}, a bound on the worst-case impact that a point can have on the property of interest.  
In classification problems, the property optimized is a loss function $\ell$ which is typically a smooth and convex approximation of the misclassification rate.  
However, the ultimately goal is not just to optimize the loss function over the observed data, but to understand how well the classifier will perform using new data drawn iid from the same (unknown) distribution $P$.

To formalize this, consider a probability distribution $P$ over a set $\XX$.  
Let $\W$ be the parameter space of potential models, and then $\ell: \XX \times \W \longrightarrow [0, \infty)$ is a loss (or cost) function defined on a single $x \in \XX$.  Let $\int_{x \in \XX} \ell(x,w) 
\d P(x)$ be the full loss function evaluated at $w \in \W$, which we will seek to minimize.  
Now a (finite) set $Y \subseteq \XX$ accompanied by a measure (or weight function) $\nu$, is termed an \emph{$\eps$-coreset} for $(\XX,P,\W, \ell)$ for $\eps \in (0,1)$ when it approximates $P$ in the following way for each $w \in \W$:
\[
\Big| \int_{\XX} \ell(x,w) \d P(x)  - \sum_{y \in Y} \nu(y)   \ell(y,w) \Big|  
    \leq  \eps \hspace{-1.5mm}\int_{\XX} \ell(x,w) \d P(x) 
\]


Algorithms for coresets are often described with respect to a finite sample $X \subset \XX$ where we then assume $P$ has support limited to and uniform over $X$.  We will return later to the integral of $P$ definition to state more general bounds.  Then a very common approach to create a coreset is via sensitivity sampling~\cite{FL2011}.  This samples points from $X$ proportional to their sensitivity score $s(x)$.  Via an importance sampling argument, one can show the size of the sample $Y$ needed to induce an $\eps$-coreset grows proportionally to the total sensitivity $S = \sum_{x \in X} s(x)$. 

For classification, consider a non-increasing function $\phi: \R \longrightarrow(0, \infty)$ and a set $\mathcal{W} \subseteq \mathbb{R}^d$. Define $\ell_\phi(x, w) = \phi(\langle x, w \rangle)$; this aligns with scenarios like Logistic regression when $\phi(t) = \log(1+e^{-t})$. However, in~\cite{Munteanu18}, it was proven that certain types of non-increasing functions $\phi$, specifically those where $\frac{\phi(x)}{\phi(-x)}$ approaches zero as $x$ goes to infinity, can lead to $\eps$-coresets that encompass all points in $X$
 (see Theorem 3.1 in \citep{pmlr-v162-tolochinksy22a} which refines \cite{Munteanu18}). 
To construct a smaller coreset, some additional problem structure needs to be introduced.
We mostly focus on the natural way of regularizing the problem~\citep{shsh2014}, initiated in the context of coresets by \citet{pmlr-v162-tolochinksy22a}.  For a positive integer $k$, define $\ell_{k,\phi}(x,w) = \phi(\langle x, w \rangle) + \frac{1}{k}\|w\|^2_2$. An $\eps$-coreset for $(X, p, \mathcal{W}, \ell_{k,\phi})$ was termed a \emph{coreset for a monotonic function $\phi$}~\citep{pmlr-v162-tolochinksy22a, pmlr-v108-samadian20a}.   
Others approached this issue in varied ways. 
For example, \citet{Munteanu18} introduced a complexity measure $\mu(X)$ that quantifies the hardness of compressing a dataset for logistic regression, and showed sublinear-sized coresets can be founding depending on this parameter.  Their results were polynomially improved by \citet{mai2021coresets} in terms of $\mu(X)$ and $d$.  
Other work \citep{mirzasoleiman2020coresets} considers greedily selected coresets for learning classifiers under incremental gradient descent where regularization is again needed, but now to ensure strong convexity.  

The most studied examples of $\phi$ are: 
\vspace{-3mm}
\begin{itemize}
    \setlength\itemsep{-0.5em}
    \item \emph{sigmoid function}: $\sigma(t) = \frac{1}{1+e^{t}}$,
    \item \emph{logistic function}: $\lgst(t) = \log(1+e^{-t})$,
    \item \emph{svm loss function}: $\svm(t) = \max(0,1-t)$.
    \item \emph{ReLU function}: $\relu(t) = \max(0,t)$
\end{itemize}
\paragraph{\bf Main Results.}  
Before we describe our results in more technical detail, we first summarize our main contributions.  
\begin{itemize}
    \setlength\itemsep{0em}
    \item We provide the first \emph{no dimensional} coreset for monotone functions, of size $O(k^3/\eps^2)$; it has no dependence on dimension $d$.  A special case of our results are in Table~\ref{tab:main_comparison1} with comparisons to prior work; this includes a $(\sqrt{\log k})$ factor improvement for logistic loss.   
    \item These bounds use a uniform sample from $X$, \emph{and} apply to iid samples from unknown distribution $P$; so this provides a true sample complexity bound for regularized classification.  
    \item We provide a new general sensitivity sampling bound based on Radamacher complexity (Theorem \ref{thm:main_rademacher1}), which leads to these no dimensional bounds, and we believe will be of general interest.  
    \item We provide a, in our opinion, simpler proof of the main sensitivity sampling claim using VC dimension \citep{FSS2020, braverman2022new}, which can handle issues of weighted loss functions needed for these results.  
\end{itemize}
{\footnotesize 
\begin{table*}[h]
    \centering
    \caption{Comparison with prior state-of-the-art results.  
    The notation $O^*(\cdot)$ hides some logarithmic factor in terms of given parameters.
    }
    \begin{tabular}{ccccc}
    \hline
       Function $\phi$ & Assumption & Sampling  & Coreset size & Reference\\
    \hline 
         $\lgst$ 
         & bounded $\mu(X)$
         & sqrt lev. scores 
         & $\begin{array}{c}
              O^*\left(\frac{d^3\mu^3(X)}{\eps^4}\right)\\
              O^*\left(\frac{\sqrt{n}d^{\frac{3}{2}}\mu(X)}{\eps^2}\right)  
         \end{array}$  
         & \cite{Munteanu18}\\ 
    \hline 
         $\lgst, \svm,  \relu$ 
         & 
         bounded $\mu(X)$
         & $\ell_l$ Lewis  
         & $\begin{array}{c}
              O\left(\frac{\mu^2(X)d\log(d/\eps)}{\eps^2}\right)  
         \end{array}$  
         & \cite{mai2021coresets}\\ 
    \hline
         $\lgst, \svm$ 
         & 
         $\begin{array}{c}
            \|x_i\|_2 \leq 1\ \forall i\\
            \text{reg. term: } \frac{1}{k}\|w\|_2^2\\
            \|w\|_2\leq 1
         \end{array}$ 
         & sensitivity 
         & $\begin{array}{c}
              O^*\left(\frac{kd^2\log n\log k}{\eps^2}\right)  
         \end{array}$  
         & \cite{pmlr-v162-tolochinksy22a}\\ 
    \hline 
         $\sigma$ 
         & 
            $\text{reg. term: } \frac{1}{k}\|w\|_2^2$
         & sensitivity 
         & $\begin{array}{c}
              O^*\left(\frac{kd^2\log n\log k}{\eps^2}\right)  
         \end{array}$  
         & \cite{pmlr-v162-tolochinksy22a}\\ 
    \hline 
         $\lgst, \svm$ 
         & 
         $\begin{array}{c}
            \|x_i\|_2 \leq 1\ \forall i\\
            \text{reg. term: } \frac{1}{k}\|w\|_2^2\\[.1cm]
            \frac{1}{k}\|w\|_2, \text{ or } 
            \frac{1}{k}\|w\|_1
         \end{array}$ 
         & uniform  
         & $\begin{array}{c}
              O\left(\frac{dk\log k}{\eps^2}\right)  
         \end{array}$  
         & \cite{pmlr-v108-samadian20a}\\ 
    \hline
    \rowcolor{lightgray}
         $\sigma, \svm, \relu$ 
         & 
         $\begin{array}{c}
            \|x_i\|_2 \leq 1\ \forall i\\
            \text{reg. term: } \frac{1}{k}\|w\|_2^2
         \end{array}$ 
         & uniform  
         & 
         $\begin{array}{c}
              O\left(\frac{dk\log k}{\eps^2}\right)  \\[.2cm]
              O\left(\frac{k^3}{\eps^2}\right) 
         \end{array}$ 
         & $\begin{array}{c} \text{Theorems~\ref{thm:coreset_for_sigmoid},\ref{thm:coreset_for_sigmoid_R},~\ref{thm:coreset_for_svm},}\\ \text{~\ref{thm:coreset_for_svm_d}, and Section~\ref{sec:relu}}
         \end{array}$\\ 
    \hline 
    \rowcolor{lightgray}
         $\lgst$ 
         & 
         $\begin{array}{c}
            \|x_i\|_2 \leq 1\ \forall i\\
            \text{reg. term: } \frac{1}{k}\|w\|_2^2
         \end{array}$ 
         & uniform  
         & $\begin{array}{c}
              O\left(\frac{dk\sqrt{\log k}}{\eps^2}\right)  \\[.2cm]
              O\left(\frac{k^3}{\eps^2\sqrt{\log k}}\right) 
         \end{array}$  
         & Theorems~\ref{thm:main_coreset_for_logistic},~\ref{thm:main_coreset_for_sigmoid_R_d}\\ 
    \hline \\
    \end{tabular}
\label{tab:main_comparison1}
\end{table*}
}
\subsection{Preliminaries}\label{subsec:preliminaries}
Throughout the paper, we maintain the assumption that $P$ is a probability measure over $\XX$. If $f:\XX\longrightarrow [0,\infty)$ is $P$-integrable, then the \emph{expectation of $f$ with respect to $P$} is defined as $\expt_{x\sim P}(f(x)) =\int_\XX f(x)\d P$. While this framework is quite broad, the most natural instances stem from the following two examples. 

\textbf{Discrete: } 
Let $p$ be a discrete distribution defined over a countable set $\XX=\{x_1,x_2,\ldots\}\subseteq \R^d$. For any subset $A\subseteq \XX$, the probability measure $P(A)$ is computed as the sum of probabilities of elements in $A$, given by $P(A) = \sum_{a\in A}p(a)$. In this scenario, given a function $f:\XX\longrightarrow [0,\infty)$, the expectation $\expt_{x\sim P}(f(x))$ is expressed as $\int_\XX f(x)\d P = \sum_{i=1}^\infty p(x_i)f(x_i)$. The case that $\XX$ is finite falls under this framework when the distribution $p$ possesses finite support.

\textbf{Continuous: } 
Let $p$ be a continuous probability density distribution defined over $\R^d$, and $\m$ represents the Lebesgue measure (or Borel measure) over $\R^d$.  Then the probability measure $P(A)$ for a measurable set $A$ is computed as $P(A) = \int_A p(x) \d\m$.

When $p(x)$ represents a probability density function, the notation $x\sim P$ denotes sampling based on $p$. 

If $\FF$ is a set of $P$-integrable functions from $\XX$ to $[0,\infty)$ such that $\int f(x)\d P(x) >0$ for all $f \in \FF$,
then the tuple $(\XX, P, \FF)$ is termed as \emph{positive definite}.  
A $P$-integrable function $s\colon \XX\longrightarrow(0,\infty)$ is called 
an \emph{upper sensitivity function} for a positive definite tuple $(\XX, P, \FF)$ if 
\begin{equation}\label{ineq:main_sensitivity}
    \sup_{f\in \FF} \frac{f(x)}{\int f(z)\d P(z)}\leq s(x)\quad\quad \text{for almost all } x\in \XX.
\end{equation}


The value $S = \int s(x)\d P(x)$ will be referred to as the \emph{total sensitivity} for $(\XX, P, \FF)$ with respect to the function $s(\cdot)$. The \emph{sensitivity normalized probability measure} for $(\XX, P, \FF)$ is defined as 
$\d Q = \frac{s(x)}{S}\d P$ or equivalently $Q(A) = \int_A \frac{s(x)}{S}\d P(x)$ for each $P$-measurable  $A$. 
Given that $S = \int s(x)\d P(x)$, it follows that $Q$ represents a probability measure over $\XX$ since $\int_\XX \d Q(x) = \int_\XX \frac{s(x)}{S}\d P(x) = 1$.
An \emph{$s$-sensitivity sample} from $\XX$ with a size of $m$ is a set of $m$ iid draws $x_1\ldots,x_m$ from $\XX$ based on $Q$ which will be expressed as as  $x_{1:m}\sim Q$. To simplify, one can imagine iid sampling from $\XX$ using the probability distribution $q(x) = \frac{s(x)}{S}p(x)$. 
Define the \emph{$s$-augmented family of $(\XX, P,\FF)$} as
\begin{equation}\label{def:main_F_G}
\mathcal{T}_\FF = \left\{T_f(\cdot) = \frac{Sf(\cdot)}{s(\cdot)\int_\XX f(x)\d P(x)}\colon f\in \FF\right\}.
\end{equation}
For each $f\in \FF$, $\sup_{x\in\XX}T_f(x)\leq S$ and 
$$\expt_{x\sim Q}(T_f(x)) 
= \int_\XX \frac{Sf(x)}{s(x)\int_\XX f(z)\d P(z)} \frac{s(x)}{S}\d P(x) = 1.
    $$
Thus, for any $f\in \FF$,   
\begin{equation}\label{equation:main_T_f}
\begin{aligned}
  \Big|
    \int_\XX f(x)\d P(x) - & \sum_{i=1}^m \frac{Sf(x_i)}{m s(x_i)}
    \Big| \leq \eps\int_\XX f(x)\d P(x)  \\
  \Longleftrightarrow \quad & 
    \left|1 - \frac{1}{m}\sum_{i=1}^m T_f(x_i)
    \right|\leq \eps.  
\end{aligned}
\end{equation}
The {\it $m$-th Rademacher Complexity} of a family $\FF$ of functions from $\XX$ to $\R$ with respect to a probability measure $P$ over $\XX$ is given by 
$$\RR^P_m(\FF) = \operatornamewithlimits{\expt}_{x_{1:m}\sim P}
    \operatornamewithlimits{\expt}_{\sigma_{1:m}\sim \{-1,1\}}
    \sup_{f\in \FF}\left[
    \frac{1}{m}\sum_{i=1}^m \sigma_i f(x_i)\right],$$
were $\sigma_{1:m}\sim \{-1,1\}$ denotes $m$ iid samples $\sigma_1,\ldots,\sigma_m$ uniformly drawn from  $\{-1,1\}.$
Rademacher Complexity somehow serves as a more detailed dimension than VC-dimension.  


\subsection{Our Contributions}
Our first main result is the following theorem. 
\begin{theorem}\label{thm:main_rademacher1}
    Let $(\XX, P, \FF)$ be a positive definite tuple and $s(\cdot)$ be an upper sensitivity function with the total sensitivity $S$. For any $t>0$, an $s$-sensitivity sample from $\XX$ of size $m$, 
    with probability at least $1-2\exp \left(-{\frac {2mt^{2}}{S}}\right)$, satisfies 
    \begin{align*}
        \Big|\int f(x)\d P(x) 
        & - \sum_{i=1}^m \frac{Sf(x_i)}{m s(x_i)}
    \Big|\\
   \leq  & (2\RR^q_m(\mathcal{T}_\FF) + t)\int f(x)\d P(x)\quad\forall f\in\FF.
    \end{align*}
\end{theorem}
This theorem, which is pivotal in our new results, is an adaptation of the central result from~\cite{FL2011, braverman2022new, FSS2020}, extensively employed in coreset construction results, but now centered on Rademacher Complexity instead of VC-dimension. 

Given a function $f:\XX\longrightarrow [0,\infty)$, for each $r\geq 0$, 
the set $\range(f, \succ, r)$ is defined as:  
$$\range(f, \succ, r) = \{x \in \XX \colon f(x)> r\}.$$ 
Now, considering a family of non-negative $P$-measurable functions $\FF$, the \emph{$\FF$-linked range space} (see~\citealt{JKPV2011}) is defined as  $(\XX, \ranges(\FF, \succ))$ where 
$$\ranges(\FF, \succ) =\{\range(f, \succ, r)\colon f\in \FF, r \geq 0\}.$$
Consider a range space \((\XX, \RR)\), where $\XX$ is the ground set and $\RR$ is a collection of subsets of $\XX$. A subset $Y \subseteq \XX$ is said to be shattered by $\RR$ if every subset $Z \subseteq Y$ can be expressed as $Z = Y \cap R$ for some $R \in \RR$. The \emph{VC-dimension of $(\XX, \RR)$} is defined as the cardinality of the largest subset $Y \subseteq \XX$ that can be shattered by $R$.
\begin{theorem}\label{thm:main_braverman2022new}
    Let $(\XX, P, \FF)$ be a positive definite tuple, and $s\colon \XX\longrightarrow(0,\infty)$ be an upper sensitivity function for it with the total sensitivity $S=\int s(x)\d P(x)$. There is a universal constant $C$ such that, for any $0<\eps,\delta<1$, if $m\geq \frac{CS}{\eps^2}\left(
    \mathsf{VC}\log S + \log\frac{1}{\delta}
    \right))$, where $\mathsf{VC} = \vc(\ranges(\mathcal{T}_\FF, \succ))$, 
    then, with probability at least $1-\delta$,  any $s$-sensitivity sample $x_1\ldots,x_m$ from $\XX$
    satisfies 
    \begin{align*}
        \Big|\int f(x)\d P(x) & - \frac{1}{m}\sum_{i=1}^m \frac{Sf(x_i)}{s(x_i)}\Big|\\
        \leq & \eps \int f(x)\d P(x)\quad\quad \forall f\in\FF.
    \end{align*}
\end{theorem}
This theorem, with variants previously shown in~\citep{braverman2022new, FSS2020}, are usually constrained by the assumption that $\mathcal{X}$ is finite, which has been extensively used to construct coresets. However, in this work we remove this restriction. This general case was also handled  in~\cite{LS2010}. 
Section~\ref{prf:ofThmvcDim} showcases a novel and simple proof for this theorem, leveraging  Fubini's Theorem~\cite{fubini1907sugli}\footnote{Fubini's Theorem is applicable in our context, given we are dealing with $\sigma$-finite measures.}.
Moreover, in instances where Rademacher complexity refines VC-dimension, our new result Theorem~\ref{thm:main_rademacher1} provides improved bounds. 
Next we describe several applications stemming from these two theorems. 

\paragraph{Well-behaved distributions.}
We next describe essential problem parameters, generalizing ones used elsewhere.   

\begin{definition}\label{def:main_well_beaived_P_R}
Let $P$ be a probability measure over $\R^d$, $\phi\colon \R\longrightarrow[0,\infty)$ a non-increasing Lipschitz function.  We say $P$ is \emph{well-behaved} (w.r.t. $\phi$) if 
\begin{equation}
    \expt\limits_{x\sim P}\|x\|_2^2 \leq E^2_1 \quad \text{ and } \quad 
    \expt\limits_{x\sim P}\phi\left(\frac{\|x\|_2}{2E_1}\right) \geq E_2,
\end{equation}
for positive constants $E_1, E_2$.  For constant $k>0$ define
\begin{equation}\label{main_Phi_family}
    \L_{\phi, k} = \left\{\ell_{\phi, k}(w,\cdot) = \phi(\langle w, \cdot\rangle) + \frac{\|w\|_2^2}{k} \colon w\in\XX\right\}.
\end{equation} 
\end{definition}

Given $\expt(t^2)\geq \expt^2(t)$, 
we can conclude that $\expt_{x\sim P}\left(\|x\|_2\right)\leq E_1$.
The assumption on the expectation of $\|x\|_2^2$ is fairly standard, such as in bounded or normalized data where it is an absolute constant.  On the other hand, when $p$ is the Gaussian distribution $\mathcal{N}_d(\boldsymbol{0}, \boldsymbol{I})$, then $E_1 = d$.  
The lower bound for $\expt\limits_{x\sim P}\phi\left(\frac{\|x\|_2}{2E_1}\right)$ is a constant bounded by $\phi(1/2)$ when $\phi$ is convex; 
then we can apply Jensen's inequality so
\begin{equation}\label{eq:main_E_2_for_convex_phi}
    \expt\limits_{x\sim P}\phi\left(\frac{\|x\|_2}{2E_1}\right)
    \geq \phi\left(\expt\limits_{x\sim P}\frac{\|x\|_2}{2E_1}\right)\geq \phi\left(\frac{1}{2}\right)=E_2. 
\end{equation}

\paragraph{Connecting to coresets.}
Assume that $s(\cdot):\XX\longrightarrow(0,\infty)$ is an upper sensitivity function for a family of functions $\FF$
with total value $S=\int_\XX s(x)\d P$. Since $\frac{f(x)}{\int f(z)\d P(z)}\leq s(x)$ for each $f\in \FF$, we have  
$1= \int \frac{f(x)}{\int f(z)\d P(z)}\d P(x) \leq \int s(x)\d P(x) = S.$ If we define $s'(\cdot) = s(\cdot)+1$, then $s':\XX\longrightarrow(1,\infty)$ is also an upper sensitivity function for that family with total value $S'=  1 + S = O(S)$, since $S \geq 1$.  
\begin{theorem}\label{thm:main_monotonic_coreset_R}
Let $P$ be a well-behaved probability measure over $\R^d$ and $\phi$ be an $L$-Lipschitz function.  If
$s(\cdot):\R^d\longrightarrow(1,\infty)$ is an upper sensitivity function for $\L_{\phi, k}$ with total sensitivity $S$ and $m\geq \frac{2S}{\eps^2}\left(8C^2 + S\log\frac{2}{\delta})\right)$,  then any $s$-sensitivity sample $x_1,\ldots,x_m$ from $\XX$ with 
weights $u_i = \frac{S}{m s(x_i)}$
provides an $\eps$-coreset for $(\R^d,\ P,\ \R^d,\ \ell_{\phi, k})$ with probability at least $1-\delta$,
where $C=  \left(2LE_1+\phi(0)\right)\max\left(2E_1k, \frac{1}{E_2}\right) + 8LkE_1 + 1$.
\end{theorem}
When $E_1$ is independent of dimension, such as bounded or normalized data, this theorem gives an $\eps$-coreset for $(\R^d,\ P,\ \R^d,\ \ell_{\phi, k})$ with size \emph{independent of dimension $d$ and input size $n$}. 
In contrast, the coresets for \cite{pmlr-v162-tolochinksy22a}, also considering bounded data, used VC-dimension (via Theorem~\ref{thm:main_braverman2022new}), so their size depends on $d$ and $n$.   

\paragraph{Bounds for specific $\phi$.}
Before describing our bounds for common choices of $\phi$ we examine the context of prior work by \citet{pmlr-v162-tolochinksy22a}.  They considered absolutely bounded data with $\|x\|_2^2 \leq 1$; so $E_1 = 1$ and $E_2 = \phi(1)$.  For simpler notation we define $R$ so that $\W_R = \{x\in\R^d\colon \|w\|_2\leq R\}$. Up to a few small typos, we summarize their results:

\begin{theorem}[\citealp{pmlr-v162-tolochinksy22a}]\label{thm:main_TJFlogisticsigmoin1}
    Let $X=\{x_1,\ldots,x_n\}$ be a set of $n$ points in the unit ball of $\R^d$,
    $\u$ be the uniform probability measure over $X$, and 
    $\eps,\delta\in(0,1)$, $R, k>0$ where $k$ is a sufficiently large constant.  Then $m = O(\frac{t}{\eps^2} (d^2 \log t + \log \frac{1}{\delta})$ sensitivity samples is, with probability at least $1-\delta$, an $\eps$-coreset for 
    \begin{itemize} 
        \item $(X,\u, \R^d,  \ell_{k, \sigma})$  with $t = (1+k) \log n$.  
        \item $(X,\u, \W_R,  \ell_{k, \lgst})$  with $t = R^2(1+Rk) \log n$.
        \item $(X,\u, \W_R,  \ell_{k, \svm})$  with $t = R(1+Rk) \log n$.  
    \end{itemize}
\end{theorem}

It should be emphasized that the last two results in Theorem~\ref{thm:main_TJFlogisticsigmoin1} was successfully improved to $O(\frac{k}{\eps^2}(d\log k +\log\frac{1}{\delta})$ for $\W=\R^d$ and $k=\Theta(n^{1-\kappa})$ for some $\kappa\in(0,1)$ by~\citet{pmlr-v108-samadian20a}. 
The strategy utilized by~\citet{pmlr-v162-tolochinksy22a} and~\citet{pmlr-v108-samadian20a} to demonstrate this theorem can be encapsulated as: (1) find an upper bound for VC-dimension,
(2) approximate a sensitivity upper bound and thus the total sensitivity, 
(3) sample $m$ points form $X$ according to an appropriate distribution, and 
(4) use Theorem~\ref{thm:main_braverman2022new} to complete the argument.  
The argument by~\citet{pmlr-v162-tolochinksy22a}, however, encounters a critical issue. To apply Theorem~\ref{thm:main_braverman2022new}, we must replace $d$ with the VC-dimension of the linked-range space of the $s$-augmented family of $(\XX, P,\L_{\phi, k})$, i.e. $\vc(\ranges(\mathcal{T}_{\L_{\phi, k}}, \succ))$, 
while they directly used the VC-dimension of linked-range space of $\FF$ (instead of $\mathcal{T}_{\L_{\phi,k}}$). 
Our work addresses this issue in that it leads to sensitivity functions for which we can compute the VC-dimension of the augmented function space.  Furthermore, by instead bounding Radamacher complexity in step (1), we can then use our new Theorem~\ref{thm:main_rademacher1} in step (4) to obtain new bounds.  These are summarized in Table~\ref{tab:main_mainresults_R} and the next theorem; details are in appendix.  

{\footnotesize 
\begin{table*}[h]
    \centering
    \caption{Coreset size for $(\R^d, P, \R^d, \L_{\phi, k})$; it is common for $E_1, A = 1$.}  
    \begin{tabular}{ccccccc}
    \hline
        function $\phi$ & data assumption & total sens. $S$ 
       & constant $C$ & sampling  & Coreset size & reference\\
    \hline
         $\sigma$ & $\expt\limits_{x\sim p}\|x\|_2^2 \leq E^2_1$ & $O(E_1^2k)$ & $O(E_1^2k)$ & $p$ & $\frac{2S}{\eps^2} (8C^2 + 
         S\log\frac{2}{\delta})$ & Thm~\ref{thm:coreset_for_sigmoid}\\
         $\sigma$ & $\expt\limits_{x\sim p}\|x\|_2^2 \leq E^2_1$ & $O(E_1^2k)$ & $O(1)$ & $p$ & $\frac{CS}{\eps^2}\left(d\log S + \log\frac{1}{\delta}\right)$ & Thm~\ref{thm:coreset_for_sigmoid_R}\\ 
    \hline
        $\lgst$ & $P\left(\|x\|_2\geq A\right)= 0$ & $O(\frac{A^2k}{\sqrt{\log(A^2k)}})$ & $O(A^2k)$ & $p$ & $\frac{2S}{\eps^2} (8C^2 + S\log\frac{2}{\delta})$ &  Thm~\ref{thm:coreset_for_logistic} \\
        $\lgst$ & $P\left(\|x\|_2\geq A\right)= 0$ & $O(\frac{A^2k}{\sqrt{\log(A^2k)}})$ & $O(1)$ & $p$ & $\frac{CS}{\eps^2}\left(d\log S + \log\frac{1}{\delta}\right)$ & Thm~\ref{thm:main_coreset_for_sigmoid_R_d} \\[0.15cm]
        $\lgst$ & $\expt\limits_{x\sim p}\|x\|_2^2 \leq E_1$ & $O(\frac{E_1^2k}{\sqrt{\log(E_1^2k)}})$ & $O(E_1^2k)$ & $\frac{s}{S}p$ & $\frac{2S}{\eps^2}\left(8C^2 + S\log\frac{2}{\delta})\right)$ & Thm~\ref{thm:coreset_for_logistic_unbounded} \\
        $\lgst$ & $\expt\limits_{x\sim p}\|x\|_2^2 \leq E_1$ & $O(\frac{E_1^2k}{\sqrt{\log(E_1^2k)}})$ & $O(1)$ & $\frac{s}{S}p$ & $\frac{CS}{\eps^2}\left(d^2\log S + \log\frac{1}{\delta}\right)$ & Thm~\ref{thm:coreset_for_logistic_unbounded_d} \\ 
    \hline
        $\svm$ & $P\left(\|x\|_2\geq A\right)= 0$ & $O(A^2k)$ & $O(A^2k)$ & $p$ & $\frac{2S}{\eps^2} (8C^2 + 
         S \log\frac{2}{\delta})$ & Thm~\ref{thm:coreset_for_svm} \\
        $\svm$ & $P\left(\|x\|_2\geq A\right)= 0$ & $O(A^2k)$ & $O(1)$ & $p$ & $\frac{CS}{\eps^2}\left(d\log S + \log\frac{1}{\delta}\right)$ & Thm~\ref{thm:coreset_for_svm_d} \\ [.15cm]
        $\svm$ & $\expt\limits_{x\sim p}\|x\|_2^2 \leq E_1$ & $O(E_1^2k)$ & $O(E_1^2k)$ & $\frac{s}{S}p$ & $\frac{2S}{\eps^2}\left(8C^2 + 
         S \log\frac{2}{\delta})\right)$ & Thm~\ref{thm:coreset_for_svm_unbounded} \\ 
        $\svm$ & $\expt\limits_{x\sim p}\|x\|_2^2 \leq E_1$ & $O(E_1^2k)$ & $O(1)$ & $\frac{s}{S}p$ & $\frac{CS}{\eps^2}\left(d^2\log S + \log\frac{1}{\delta}\right)$ & Thm~\ref{thm:coreset_for_svm_unbounded_d} \\
    \hline
        $\relu$ & $P\left(\|x\|_2\geq A\right)= 0$ & $O((1+A)^2k)$ & $O((1+A)^2k)$ & $p$ & $\frac{2S}{\eps^2} (8C^2 + 
         S\log\frac{2}{\delta})$ & Sec.~\ref{sec:relu} \\
        $\relu$ & $P\left(\|x\|_2\geq A\right)= 0$ & $O((1+A)^2k)$ & $O(1)$ & $p$ & $\frac{CS}{\eps^2}\left(d\log S + \log\frac{1}{\delta}\right)$ & Sec.~\ref{sec:relu} \\[.15cm]  
        $\relu$ & $\expt\limits_{x\sim p}\|x\|_2^2 \leq E_1$ & $O((1+E_1)^2k)$ & $O((1+E_1)^2k)$ & $\frac{s}{S}p$ & $\frac{2S}{\eps^2}\left(8C^2 + 
         S\log\frac{2}{\delta})\right)$ & Sec.~\ref{sec:relu} \\ 
        $\relu$ & $\expt\limits_{x\sim p}\|x\|_2^2 \leq E_1$ & $O((1+E_1)^2k)$ & $O(1)$ & $\frac{s}{S}p$ & $\frac{CS}{\eps^2}\left(d^2\log S + \log\frac{1}{\delta}\right)$ & Sec.~\ref{sec:relu} \\
    \hline
    \end{tabular}
    \label{tab:main_mainresults_R}
\end{table*}
}

\begin{theorem}\label{thm:main_summery_R}
    Let $P$ be a probability measure over $\R^d$. Given the conditions outlined in the initial $5$ columns of Table~\ref{tab:main_mainresults_R}, any $s$-sensitivity sample size $m$ (sampling according to the distribution provided in Column 6) guarantees an $\eps$-coreset for $(\R^d, P, \R^d, \L_{\phi, k})$ with a probability of at least $1-\delta$.
\end{theorem}


\paragraph{Sample complexity.}
Note that if the sampling column of Table~\ref{tab:main_mainresults_R} uses $p$, that indicates the total sensitivity has a constant upper bound, and we can use iid samples from the unknown distribution $P$.  Not only can we avoid the computational issue of estimating $s$, but these provide sample complexity bounds for the full (continuous) setting.  

\section{Proofs of Main Results}
\subsection{Proof of Theorem~\ref{thm:main_rademacher1}}
The proof follows by McDiarmid's inequality (see lemma~\ref{lem:GenMcDiarmid}), which is a strong concentration inequality that bounds the difference between the sampled mean and the true mean of a function satisfying the bounded differences property: the effect of changing a single observation.
Here we outline the proof of Theorem~\ref{thm:main_rademacher1}; full proof in Appendix \ref{proof:radamahcer_main_thm}.
\begin{proof}[Sketch of proof of Theorem~\ref{thm:main_rademacher1}.]
    In view of Equation~\eqref{equation:main_T_f}, we need to show that, 
    for any iid sample $x_1,\ldots,x_m\in\XX$ according to the sensitivity normalized distribution $q = p(x)
\frac{s(x)}{S}$,
    with probability at least $1-2\exp \left(-{\frac {2mt^{2}}{S^2}}\right)$, 
    $$\sup_{f\in \FF}\Big|1 - \sum_{i=1}^m \frac{T_f(x_i)}{m}
    \Big|\leq 2\RR^q_m(\mathcal{T}_\FF) + t.$$
    For every $x\in \XX$ and $f\in\FF$, we have $\frac{f(x)}{\int_\XX f(z)\d P(z)}\leq s(x)$ which implies 
    $0\leq T_f(x)= \frac{Sf(x)}{s(x)\int_\XX f(z)\d P(z)}\leq S$. Therefore, 
    $\sup\limits_{x\in\XX} T_f(x)\leq S$. 
    Now, define 
    $$g(x_1,\ldots,x_m) = \sup_{f\in \FF}
    \Big[
    1 - \sum_{i=1}^m \frac{T_f(x_i)}{m}
    \Big].$$
    For each $i\in[m]$ and each $(x_1,\ldots, x_m), (x_1',\ldots, x_m')\in \XX^m$ such that $x'_j = x_j$ for $j \neq i$, 
    observe that  
    \begin{align*}
    |g(x_1,\ldots,x_m) - g(x_1',\ldots, x_m')| \leq\frac{S}{m}.
    \end{align*} 
    Therefore, using McDiarmid's Inequality, with a probability at least 
    $1-\exp \left(-{\frac {2mt^{2}}{S^2}}\right)$, we have 
    \begin{align*}
    \sup_{f\in \FF}
    \Big|
    1 - \sum_{i=1}^m \frac{T_f(x_i)}{m}
    \Big|& 
    \leq \hspace{-2mm} \expt_{x_{1:m}\sim q}\sup_{f\in \FF} 
    \Big|1-\sum_{i=1}^m \frac{T_f(x_i)}{m}\Big| + t\\
    & \leq 2\RR^q_m(\mathcal{T}_\FF) + t \qedhere.
    \end{align*}
\end{proof}

\subsection{Proof of Theorem~\ref{thm:main_braverman2022new}}\label{prf:ofThmvcDim}

This approach to get relative error goes through VC-dimension, which exploits combinatorial properties of range spaces.  
A \emph{$P$-range space} is a tuple $(\XX, P, \RR)$ where $(\XX, P)$ is a probability measure space and $\RR$ is a subset of $2^\XX$ whose members are $P$-measurable.  Given a $P$-range space, the best relative error possible in general is conditioned on a small additive parameter $\eta>0$ as follows.  
For an $\eps,\eta\in(0,1)$, the measure $\nu$ on $\XX$ is called a \emph{relative $(\eps, \eta)$-approximation for $(\XX, P, \RR)$} if 
each $R\in \RR$ is $\nu$-measurable and 
$\left| 
   P(R) - \nu(R)
\right|\leq \eps \max(\eta, P(R)).$
Using ideas from \citet{LI2001516}, then \citet{HS11} showed 
that a sufficiently large sample according to $P$ provides a relative $(\eps, \eta)$-approximation for $(\XX, P, \RR)$, if the VC-dimension is bounded as $\mathsf{VC}$.  Specifically, 
%
there is a universal constant $C$ such that, for any $\eta>0$, and $\eps,\delta\in (0,1)$, with probability at least $1-\delta$, iid sample $X = \{x_1,\ldots,x_m\}$ according to $P$ with $m\geq \frac{C}{\eta\eps^2}\left(
    \mathsf{VC}\log \frac{1}{\eta} + \log\frac{1}{\delta}
    \right)$ satisfies 
    \[
    \left| P(R) - \frac{|R\cap X|}{m}\right|\leq \eps\max\left(\eta, P(R)\right)\quad\quad \forall R\in \RR.
    \]

In order to use this relative error conditioned on $\eta$ to obtain an unconditioned relative error, the key insight is that using an $s$-sensitivity sample for an upper sensitivity function $s$ with total sensitivity $S$, then by setting $\eta = 1/S$, we obtain unconditioned relative error.  This is formalized in the following in the next lemma -- a discrete version of which is implicit in the proof of Theorem 31 in \cite{FSS2020}.  

\begin{lemma}\label{thm:main_g_eta}
    Let $(\XX, P, \FF)$ be a positive definite tuple and $s(\cdot)$ be an upper sensitivity function for it with the total sensitivity $S$.  
	If the measure $\nu$ is a relative $(\eps, \eta)$-approximation for $(\XX, Q, \ranges(\TT_\FF, \succ))$, where $\d Q(x) = \frac{s(x)}{S}\d P(x)$, 
	then for each $f\in \FF$, 
        \begin{align*}\Big|\int_{x \in \XX} f(x)\d P(x) - & \int_{x \in \XX} \frac{Sf(x)}{s(x)}\d \nu(x)\Big | \\
         & \leq (\eps + S\eta\eps)\int_{x \in \XX} f(x)\d P(x).
         \end{align*}
\end{lemma}

With this, the proof of Theorem~\ref{thm:main_braverman2022new} is straightforward.  

\begin{proof}[Proof of Theorem~\ref{thm:main_braverman2022new}]
    Since $(\XX, Q, \ranges(\TT_\FF, \succ))$ is a $Q$-range, 
    for an $\eps\in (0,1)$ and  $\eta = \frac{1}{S}$, \citet{HS11}'s sampling result implies that there is a universal constant $C$ such that, any iid sample $X = \{x_1,\ldots,x_m\}$ according to $Q$ with $m\geq \frac{4CS}{\eps^2}\left(
    \mathsf{VC} \log S + \log\frac{1}{\delta}
    \right)$ satisfies 
    $$\left| Q(R) - \frac{|R\cap X|}{m}\right|\leq \frac{\eps}{2}\max\left(\frac{1}{S}, Q(R)\right)\quad\quad \forall R\in \RR,$$
    with probability at least $1-\delta$.
    This implies that the uniform probability measure over $X$ serves as a relative $\left(\frac{\eps}{2}, \frac{1}{S}\right)$-approximation for $(\XX, Q, \ranges(\TT_\FF, \succ))$.
    The proof concludes by using Lemma \ref{thm:main_g_eta} to derive the error co-efficient on the right hand side as $(\frac{\eps}{2} + S \frac{1}{S}\frac{\eps}{2}) = \eps$.  
\end{proof}

We next provide what we believe is a simpler and more direct proof of Lemma \ref{thm:main_g_eta}.  
A key improvement is the use of Fubini's theorem to transform back and forth between integrating over $\XX$ to over the range of the function.
We first define some useful shorthand notation:
$Q_\succ(g,t) = Q(\range(g,\succ,t)) = \int_{x \in \XX} \1_{g(x) > t} \d Q$.

\begin{lemma}\label{lem:fubini_application}
For any $g = T_f\in\TT_\FF$ and $r\geq 0$,
    $$\int_{x \in \XX} \min\{r, g(x)\}\d Q(x) = \int_{0}^{r} Q_\succ(g,t))\d t.$$
\end{lemma}
\begin{proof}
\begin{align*}
\int_{x \in \XX} &\min\{r, g(x)\}\d Q(x)
\\ & = 
\int_{x \in \XX}  \left(\int^{r}_{t=0} \1_{_{t < g(x)}} \d t\right) \d Q(x)
\\& =  \int^{r}_{t=0} \left(\int_{x \in \XX}  \1_{_{t < g(x)}} \d Q(x)\right) \d t \quad\quad \text{(*)}
\\ & =  
\int_{t=0}^{r} Q_\succ(g,t)\d t,
\end{align*}
where $(*)$ is true due to Fubini's Theorem.  
\end{proof}


\begin{proof}[Proof of Lemma~\ref{thm:main_g_eta}]
Dividing by $P(f) = \int_\XX f(z)\d P(z)$, we can now restate the goal as proving 
\[
\frac{1}{P(f)} \left| P(f) - \int_\XX  \frac{Sf(x)}{s(x)}\d\nu(x) \right| \leq \eps (1+S\eta).
\]
For an arbitrary $g = T_f\in\TT_\FF$, using that $\int_\XX g(x)\d\nu = \frac{1}{P(f)}
\int_{\XX}\frac{Sf(x)}{s(x)}\d \nu(x),$ and then $\int_\XX g(x)\d Q(x) = 1$ we can transform the left-hand side
\begin{align}
\MoveEqLeft{\frac{1}{P(f)} \left| P(f) - \int_\XX  \frac{Sf(x)}{s(x)}\d\nu(x) \right| } \nonumber
\\ &=
\left| 1 - \int_{x \in \XX} g(x) \d\nu(x) \right| \nonumber
\\ &=
\left| \int_{x \in \XX} g(x) \d Q(x) - \int_{x \in \XX} g(x) \d\nu(x)
\right| \nonumber
\\ & = 
\Big |\int_{x \in \XX} \min \{ S,g(x) \} \d Q(x)  \nonumber
\\ & \hspace{20mm} -
 \int_{x \in \XX} \min \{ S,g(x) \} \d\nu(x) \Big |,
\label{eq:pre-fubini}
\end{align}
where the last line follows  $\max_x g(x) \leq S$.    
Now applying 
Lemma \ref{lem:fubini_application} twice on both $Q$ and $\nu$ we have
\begin{align}
\eqref{eq:pre-fubini} &= \left| \int_{t = 0}^S Q_\succ(g,t) - \int_{t = 0}^S \nu_\succ(g,t) \d t \right|  \nonumber
\\ & \leq
\int_{t = 0}^S \left| Q_\succ(g,t) - \nu_\succ(g,t) \right| \d t \nonumber
\\ & \leq
\eps \int_{t = 0}^S \max \{ \eta, Q_\succ(g,t) \} \d t,  \label{eq:pre-split}
\end{align}
where the last line follows by $\nu$ being an $(\eps, \eta)$-approximation of $(\XX, Q, \ranges(\TT_\FF, \succ))$.  
We can then split this integral on $t$ from $0$ to $S$ into two parts at a point $r_g =\sup\{r\geq 0\colon Q_\succ(g,t) \geq \eta\}$.  
Note that for each $t\in[0,r_g)$,  $Q_\succ(g,t) \geq \eta$ and for each $t> r_g$, $Q_\succ(g,t) < \eta$.  
Taking these observations into account, we obtain 
\begin{align*}
\int_{t=0}^{r_g} \max \{ \eta, Q_\succ(g,t) \} \d t
& = 
\int_{t=0}^{r_g} Q_\succ(g,t)\d t 
\\ \textrm{(Lemma \ref{lem:fubini_application})} & = 
\int_{x \in \XX} \min \{ r_g, g(x) \} \d Q(x)
\\ & \leq 
\int_{x \in \XX}  g(x)\d Q(x)
     = 1.
\end{align*}
and
\[
\int_{t=r_g}^S \max\{\eta, Q_\succ(g,t)\} \d t  
= \int^S_{t=r_g} \eta \d t 
\leq S \eta.
\]
Combining these two parts, we have 
\[
\eqref{eq:pre-split} \leq \eps (1 + S \eta)
\]
completing the proof.  
\end{proof}

\subsection{Proof of Theorem~\ref{thm:main_monotonic_coreset_R}}
When dealing with Rademacher complexity, we can leverage the helpful lemma known as Talagrand’s Contraction Lemma (refer to Theorem 4.12 in \cite{alma991043798149103276}). This lemma, in particular, establishes that $\RR_m(\phi\circ\FF)\leq L\RR_m(\FF)$, where $\phi\colon \R\longrightarrow \R$ is an $L$-Lipschitz function, and $\phi\circ \FF = \{\phi\circ f\colon f\in\FF\}$. Our requirement extends beyond this lemma to the following generalization (for the proof, see Appendix~\ref{apdx:contraction_lemma}).
\begin{corollary}\label{cor:Talagrand}
    Let $\boldsymbol{T}\subset \R^m$ be a bounded set, and $\eps_1,\ldots,\eps_m >0$. For any $L$-Lipschitz functions $\phi_i\colon \R\longrightarrow \R$ with $\phi_i(0) = 0$ for $i\in[m]$, we have 
    \begin{align*}
        \expt_{\varrho_{_{1:m}}\sim \boldsymbol{\eps}}
    & \sup_{\t\in \T}\Big|
    \sum_{i=1}^m \varrho_{_i}\phi_i(t_i)
    \Big| \leq 
    2L\expt_{\varrho_{_{1:m}}\sim \boldsymbol{\eps}}
    \sup_{\t\in \T}\Big|
    \sum_{i=1}^m \varrho_{_i} t_i
    \Big|,
    \end{align*}
    where $\boldsymbol{\eps} = \prod_{i=1}^m\{\pm\eps_i\}$. 
\end{corollary}
To establish the proof of Theorem~\ref{thm:main_monotonic_coreset_R}, we first need to bound the Radamacher complexity. 
\begin{lemma}\label{lem:main_rademacherupper}
For a well-behaved measure $P$ (see 
Definition~\ref{def:main_well_beaived_P_R}), 
if $s(\cdot):\R^d\longrightarrow(1,\infty)$ is an upper sensitivity function for $\L_{\phi, k}$ with total sensitivity $S$, then  $$\RR^q_m(\mathcal{T}_{\L_{\phi, k}})\leq C\sqrt{\frac{S}{m}}$$
for $C= \left(2LE_1+\phi(0)\right)\max\left(4E_1k, \frac{1}{E_2}\right) 
        + 8LkE_1 + 1$.
\end{lemma}
As the proof is detail-involved, we present a sketch of it here. The complete proof can be found in Appendix~\ref{apx:proof_radamacher_thm}.
\begin{proof}[Sketch of proof.]
    we start with the two following observations
\begin{equation}\label{ineq:main_upper}
    \begin{aligned}
        &\expt_{x_{1:m}\sim q}\expt_{\sigma_{1:m}\sim \{-1,1\}}
        \Big\|\sum_{i=1}^m\sigma_i\frac{x_i}{s(x_i)}\Big\|_2\leq \sqrt{\frac{m}{S}}E_1  \\
        &\expt_{x_{1:m}\sim q}\expt_{\sigma_{1:m}\sim \{-1,1\}}
        \Big|\sum_{i=1}^m\sigma_i\frac{1}{s(x_i)}\Big|\leq \sqrt{\frac{m}{S}}. 
    \end{aligned}
\end{equation}
Next, for $\alpha(w)=\frac{S}{\int_\XX \ell_{\phi,k}(w,x)\d P(x)}$, we deduce  that $\alpha(w)\|w\|_2^2\leq Sk$ and 
$$\alpha(w)\leq \left\{\begin{array}{cc}
     \max\left(4SE_1k, \frac{S}{E_2}\right)&  \|w\|_2\leq 1\\
     \frac{kS}{2^{2n}} & 2^n\leq \|w\|_2\leq 2^{n+1}.
\end{array}\right.$$
For $\bar{\phi}(\cdot) = \phi(\cdot) - \phi(0)$, we obtain 
\begin{align*}
        & \RR^q_m(\mathcal{T}_{\L_{\phi,k}})\leq  \\
        & \quad \underbrace{\expt_{
        \begin{array}{c}
             x_{1:m}\sim q  \\
             \sigma_{1:m}\sim \{-1,1\} 
        \end{array}}
        \sup_{w\in \XX} \alpha(w) \left|
        \frac{1}{m}\sum_{i=1}^m \sigma_i \frac{\bar{\phi}(\K(w,x_i))}{s(x_i)}\right|}_{=M_1}\\ 
        & \quad\quad + \underbrace{\expt_{
        \begin{array}{c}
             x_{1:m}\sim q  \\
             \sigma_{1:m}\sim \{-1,1\} 
        \end{array}}
        \sup_{w\in \XX} \alpha(w) \left|
        \frac{1}{m}\sum_{i=1}^m \sigma_i \frac{\phi(0)}{s(x_i)}\right|}_{=M_2}\\
         & 
        \quad \quad +\frac{1}{k}\underbrace{\expt_{
        \begin{array}{c}
             x_{1:m}\sim q  \\
             \sigma_{1:m}\sim \{-1,1\} 
        \end{array}
        }
        \sup_{w\in \XX} \alpha(w)\|w\|_2^2 \left|
        \sum_{i=1}^m  \frac{\sigma_i}{ms(x_i)}\right|}_{= N} 
    \end{align*}
    Using the provided upper bound for $\alpha(w)$ and Equation~\eqref{ineq:main_upper}, we conclude 
    $M_1\leq 2LE_1 \max\left(2E_1k, \frac{1}{E_2}\right) \sqrt{\frac{S}{m}}
        + 8LkE_1\sqrt{\frac{S}{m}}$ and 
    $M_2\leq \sqrt{\frac{S}{m}}\max\left(4E_1k, \frac{1}{E_2}\right)\phi(0)$.
    Utilizing the bound $\alpha(w)\|w\|_2^2\leq Sk$ and Equation~\eqref{ineq:main_upper}, we derive 
    $N\leq \sqrt{\frac{S}{m}}k$ which completes the proof. 
\end{proof}

The proof of Theorem~\ref{thm:main_monotonic_coreset_R} now follows immediately from Theorem~\ref{thm:main_rademacher1} and Lemma~\ref{lem:main_rademacherupper}.


\subsection{Proof of results in Table~\ref{tab:main_mainresults_R}}
We here only outline the proofs for $\phi = \lgst$. The proofs for other cases share similarities and are presented in full detail in Appendix~\ref{sec:coresets_for_monotonic_functions}. 
As these results are yielded from Theorems~\ref{thm:main_braverman2022new} and~\ref{thm:main_monotonic_coreset_R}, we need to know the upper sensitivity function. We start with a simple observation, providing us with a method to calculate an upper sensitivity function. 
\begin{lemma}\label{lem:main_sensitivity_beta}
For a probability measure $P$ over $\XX$, 
assume that $\W\subseteq \XX$, $\ell: \XX\times \XX \longrightarrow(0,\infty)$, and 
    $\gamma:[0,\infty)\longrightarrow [0,\infty)$ such that $0<\int_\XX \gamma(\|x\|_2)\d P(x)<\infty$ 
     and   
    \begin{equation}\label{def:beta_property}
        \gamma(\|x\|_2)\leq \frac{\ell(x,w)}{\ell(y,w)} \quad\quad \forall x,y\in \XX, w\in \W.
    \end{equation}
Then 
$s(y)= \frac{1}{\int_\XX \gamma(\|x\|_2)\d P(x)}$
is an upper sensitivity function for $(\XX, P, \L_\W)$,
where $\L_\W = \{\ell(\cdot,w)\colon w\in\W\}$.
\end{lemma}

Then we need the following extension of a lemma by \citet{pmlr-v162-tolochinksy22a}.
\begin{lemma}\label{lem:main_beta_property}
    Let $\phi:\R\longrightarrow(0,\infty)$ be a non-increasing function such that  
    \begin{equation}\label{beta_property}
        \frac{\phi(-\alpha z)+\frac{z^2}{k}}{\phi(\alpha z)+\frac{z^2}{k}}\leq \beta(\alpha)\quad 0\leq \alpha  \leq B_1, 0\leq z\leq B_2.
    \end{equation}
    If we set $M=\phi(-B_1B_2)$, then, for each $x,y, w\in \XX$ with $\|x\|, \|y\|\leq B_1, \|w\|\leq B_2$, we have   
    $$\frac{\phi(0)}{M \beta(\|x\|)}\leq \frac{\ell_{\phi,k}(x,w)}{\ell_{\phi,k}(y,w)}.$$
    If $\phi$ is universally bounded by $M$, then 
    we do not need upper bounds $B_1, B_2$ for $\alpha, z$ and consequentially do not need upper bounds for $\|x\|,\|y\|$, and $\|w\|$, and the same statement holds.  
\end{lemma}
With a detailed proof (see  Lemma~\ref{lem:beta_for_logistic}), 
we observe that 
$$\frac{\lgst(-\alpha z)+\frac{z^2}{k}}{\lgst(\alpha z)+\frac{z^2}{k}}\leq \left\{
    \begin{array}{cc}
        \frac{85\alpha^2k}{\log(\alpha^2k)}  &  \alpha^2 k > e\\[.2cm] 
        85               &  \alpha^2 k\leq e,
    \end{array}\right.$$ 
for each $\alpha, z \geq 0.$ Using Lemmas~\ref{lem:main_sensitivity_beta} and~\ref{lem:main_beta_property}, this concludes $s(y) = 1 + \frac{170(1+kA^2)}{\sqrt{\max(1,\log(kA^2))}}$ is an upper sensitivity for $(\R^d,\ P,\ \L_{\lgst,k})$ with probability measure $P$ with $P\left(\left\{x\in\XX\colon \|x\|_2\geq A\right\}\right)= 0$ (see Lemma~\ref{lem:sensitivity_logistic}). 
\begin{theorem}\label{thm:main_coreset_for_logistic}
    Assume that $P$ is probability measure over $\R^d$ such that 
    $P\left(\left\{x\in \XX\colon \|x\|_2\geq A\right\}\right)= 0$,  $S=1 + \frac{170(1+kA^2)}{\sqrt{\max(1,\log(kA^2))}}$, and $C= (2A+1)\max(4Ak, 2.5)+8Ak+1$. 
    For $m \geq \frac{2S}{\eps^2} (8C^2 + S\log\frac{2}{\delta})$, any iid sample $x_1,\ldots,x_m$ from $P$ with weights $u_i = \frac{1}{m}$ provides an $\eps$-coreset for $(\R^d, P, \R^d, \ell_{\lgst, k})$ with probability at least $1-\delta$.
\end{theorem}
\begin{proof}
    We observed that  $s(y) = 1 + \frac{170(1+kA^2)}{\sqrt{\max(1,\log(kA^2))}}$ serves as an upper sensitivity function for $(\R^d,\ P,\ \L_{\lgst,k})$. Since it is a constant function, $s$-sensitivity sampling is equivalent to sampling according to $P$. It is worth noting that $\lgst$ is a $1$-Lipschitz, convex, and decreasing function. 
    For $E_1=A$,  
    we have 
     $$\expt\limits_{x\sim p}\lgst\left(\frac{\|x\|_2}{2E_1}\right)\geq 
     \lgst\left(\frac{1}{2}\right)\geq \frac{2}{5} = E_2.$$
    This implies that Definition~\ref{def:main_well_beaived_P_R} is satisfied for $\L_{\lgst, k}$ and thus  
    Theorem~\ref{thm:main_monotonic_coreset_R} concludes the statement for $m\geq \frac{2S}{\eps^2}\left(8C^2 + S\log\frac{2}{\delta}\right)$. 
\end{proof}
As the upper sensitivity function for $\L_{\lgst, k}$ is constant,  $\ranges(\mathcal{T}_{\L_{\lgst, k}}, \succ) = \ranges(\L_{\lgst, k}, \succ)$: every function in $\mathcal{T}_{\L_{\lgst, k}}$ is a function in $\L_{\lgst, k}$ scaled by a positive constant. Thus
    $\vc(\ranges(\mathcal{T}_{\L_{\lgst, k}}, \succ)) = \vc(\ranges(\L_{\lgst, k}, \succ)).$ 
For an $f_w\in \L_{\lgst, k}$ and $r\geq 0$,
\begin{align*}
    \range(& f_w, \succ, r) 
     = \left\{x \in \R^d \colon f_w(x)> r\right\}\\
    & = \left\{x \in \R^d \colon \lgst(\langle x,w\rangle) + \frac{\|w\|^2}{k} > r\right\}\\
    & = \left\{x \in \R^d \colon \log\left(1+e^{-\langle x,w\rangle}\right) > \underbrace{r - \frac{\|w\|^2}{k}}_{=t}\right\}\\
    & = \left\{\begin{array}{ll}
      \R^d   &  t\leq 0\\ 
       \left\{x\in\R^d\colon \langle x,w\rangle < \log (e^t-1)\right\}  &  t > 0,
    \end{array}\right.
\end{align*}
which concludes that $\ranges(\L_{\lgst, k}, \succ)$ only includes half-spaces and the whole space $\R^d$.  Therefore, by Radon's theorem, $\vc(\ranges(\L_{\lgst, k}, \succ))\leq d+1$. 
Leveraging Theorem~\ref{thm:main_braverman2022new}, we obtain the following theorem.  
\begin{theorem}\label{thm:main_coreset_for_sigmoid_R_d}
    Let $P$ be a probability measure over $\R^d$ such that 
    $P\left(\left\{x\in\XX\colon \|x\|_2\geq A\right\}\right)= 0$ and 
    $S=1 + \frac{170(1+kA^2)}{\sqrt{\max(1,\log(kA^2))}}$.
    There is an $m = O\left(\frac{S}{\eps^2}\left(d\log S + \log\frac{1}{\delta}\right)\right)$ such that 
    any iid sample $x_1,\ldots,x_m$ from $P$ with weights $u_i = \frac{1}{m}$ provides an $\eps$-coreset for $(\R^d,\ P,\ \R^d,\ \ell_{\lgst, k})$ with probability at least $1-\delta$.
\end{theorem}
    
\section{Conclusion and Experimental results}
This paper provides the first no dimensional sampling coresets for classification; they provides relative error for standard loss functions on linear classification with regularization, and an expectation bound on the data norm.  Some results apply to iid samples from a continuous distribution, and hence imply sample complexity bounds.  The key new ingredient is a Radamacher complexity bound for sensitivity sample coresets, which we expect will find further use.  

The appendix shows how these results can be applied to kernelized versions of these problems, and also recover the best sampling bounds for KDE coresets.  

While we do not provide new experimental evidence of the claims, our results are consistent with simulations in many previous papers.   For example  \citet{pmlr-v162-tolochinksy22a} show non-uniform sensitivity sampling slightly out-performing uniform samples; but this does not contradict our results using iid samples since, for example, in those experiments doubling the iid sample size improves upon the non-uniform sample results.


\section*{Acknowledgements}
Thanks to support from NSF CCF-2115677, CDS\&E-1953350, IIS-1816149, and IIS-2311954.  




\section*{Impact Statement}
This paper advances the field of Machine Learning, particularly in understanding the convergence rate and potential compression of classic problem formulations.  While this has potential for a wide variety of impacts, including societal ones, there are none in particular that we feel demand to be highlighted here as specific consequences.

\nocite{langley00}


\newpage
\appendix
\onecolumn
 
\section{Generalized Framework: Reproducing kernel Hilbert Space}
We can establish the validity of our findings in a broader framework by employing reproducing kernels. In particular, the above-mentioned results are automatically derived when considering $\XX = \R^d$ with the linear kernel $\K(x,w) = \langle x, w\rangle$ as a reproducing kernel Hilbert space $H$. For those less familiar with reproducing kernel Hilbert spaces, maintaining this latter setting throughout the entire paper is helpful. 

\subsection{Kernelization of the Results}
The kernel method is known as a powerful technique in machine learning that enables the handling of non-linear relationships in data by implicitly mapping it to a higher-dimensional space. This flexibility and efficiency make it a widely used approach, especially in scenarios where linear methods may not be sufficient.
Reproducing kernel Hilbert spaces (RKHS) are intimately connected to kernel methods. An RKHS is a type of Hilbert space associated with a positive definite kernel function. In the context of machine learning, we can think that the input space is mapped into an RKHS through a feature map. Broadly, the RKHS framework provides a mathematical foundation for understanding how non-linear relationships in data can be effectively captured through implicit mappings into higher-dimensional spaces. The key insight in the context of kernel methods is that the kernel function implicitly represents the inner product in the Hilbert space, for more see~\cite{Mercer1909, aronszajn50reproducing, BGV1992,SC2008,SS2018}. The following definition extends Definition~\ref{def:main_well_beaived_P_R} to the context of RKHS cases. 

\begin{definition}\label{def:well_beaived_P1}
Let $H$ be a reproducing kernel Hilbert space of real-valued functions on $\XX$ with kernel $\K\colon \XX\times \XX\longrightarrow \R$, $P$ a probability measure over $\XX$, $\phi\colon \R\longrightarrow [0,\infty)$ a non-increasing $L$-Lipschitz function, and $k>0$ a constant.  We say $P$ is \emph{well-behaved} (w.r.t $\phi$ and $\K$) if 
\begin{equation}
    \expt\limits_{x\sim P}\K(x, x) \leq E^2_1 \qquad\qquad \text{and} \qquad\qquad 
    \expt\limits_{x\sim P}\phi\left(\frac{\sqrt{\K(x, x)}}{2E_1}\right) \geq E_2,
\end{equation}
where $E_1, E_2$ are two positive constants.  
Define 
\begin{equation}\label{Phi_family_H}
    \L^H_{\phi, k} = \left\{\ell^H_{\phi, k}(w,\cdot) = \phi(\K(w, \cdot)) + \frac{1}{k}\K(w,w) \colon w\in\XX\right\}
\end{equation} 
and 
\begin{equation}\label{Phi_bar_family_H}
    \bar{\L}^H_{\phi, k} = \left\{\bar{\ell}^H_{\phi, k}(w,\cdot) = \phi(-\K(w, \cdot)) + \frac{1}{k}\K(w,w) \colon w\in\XX\right\}
\end{equation} 
To keep things simpler, we occasionally employ $\|X\|_H$ and $\langle \cdot, \cdot \rangle_H$ instead of $\sqrt{\K(x, x)}$ and $\K(\cdot,\cdot)$ respectively.
\end{definition} 

\begin{remark}\label{rem:minus_K}
 Throughout the paper, we mainly focus on coresets for $\L^H_{\phi, k}$. However, we will see that our analysis effortlessly extends to $ \bar{\L}^H_{\phi, k}$. This consideration is insignificant when dealing with the standard inner product as the kernel, given that  $-\langle w,\cdot\rangle = \langle\cdot, -w\rangle$. However, this property does not hold for kernels in general.   
\end{remark}

\begin{theorem}[Theorem~\ref{thm:main_monotonic_coreset_R}, Kernelized Representation]\label{thm:monotonic_coreset1}
For a well-behaved probability measure $P$ with respect to a Hilbert space $H$ of real-valued functions defined on $\XX$ with kernel $\K\colon \XX\times \XX\longrightarrow \R$, if
$s(\cdot):\XX\longrightarrow(1,\infty)$ is an upper sensitivity function for $\L^H_{\phi, k}$ with total sensitivity $S$ and $m\geq \frac{2S}{\eps^2}\left(8C^2 + S\log\frac{2}{\delta})\right)$,  then any $s$-sensitivity sample $x_1,\ldots,x_m$ from $\XX$ with 
weights $u_i = \frac{S}{m s(x_i)}$
provides an $\eps$-coreset for $(\XX,\ P,\ \XX,\ \ell^H_{\phi, k})$ with probability at least $1-\delta$,
where $$C=  \left(2LE_1+\phi(0)\right)\max\left(2E_1k, \frac{1}{E_2}\right) + 8LkE_1 + 1.$$ 
The statement remains true if we replace $\L^H_{\phi, k}$ and  $(\XX,\ P,\ \XX,\ \ell^H_{\phi, k})$ by  $\bar{\L}^H_{\phi, k}$ and $(\XX,\ P,\ \XX,\ \bar{\ell}^H_{\phi, k})$.
\end{theorem}

Considering $\R^d$ as a Hilbert space equipped with linear kernel, we derive Theorem~\ref{thm:main_monotonic_coreset_R} as a corollary of this theorem.  In the following, we present various applications derived from this theorem.

\begin{theorem}
    Let us consider $H$, a reproducing kernel Hilbert space of real-valued functions defined on $\XX$ with kernel $\K\colon \XX\times \XX\longrightarrow \R$, and $P$, a probability measure over $\XX$. Given the conditions outlined in the initial $4$ columns of Table~\ref{tab:mainresults}, any $s$-sensitivity sample of size $m$ (sampling according to the distribution provided in Column 5) guarantees an $\eps$-coreset for $(\XX, P, \XX, \L^H_{\phi, k})$ with a probability of at least $1-\delta$.
\end{theorem}
The first five columns of each row in Table~\ref{tab:mainresults} outline assumptions concerning $H$ and $P$, while the last two columns detail sample complexity and the respective theorem in the paper where the result is introduced. Specifically, under the conditions specified in the initial four columns, a sample of size $m$ (as per the sampling distribution in Column 5) produces an $\eps$-coreset with a probability of $1-\delta$. 
It is worth emphasizing once more that the utilization of 
$p$ in Column 5 signifies that sampling is directly conducted from the data distribution, specifically through uniform sampling from the provided data points. This eliminates the necessity for sensitivity computations in this context. 

{\footnotesize 
\begin{table}
    \centering
    \caption{Coreset size for $(\XX, P, \XX, \L^H_{\phi, k})$ where $P$ is a probability measure over $\XX$ and $H$ is an RKHS.}
    \begin{tabular}{cccccccc}
    \hline
       function $\phi$ & data assumption & total sens. $S$ 
       & constant $C$ & sampling  & Coreset size & reference\\
    \hline
         $\sigma(\pm\cdot)$ & $\expt\limits_{x\sim P}\|x\|_H^2 \leq E^2_1$ & $O(E_1^2k)$ & $O(E_1^2k)$ & $p$ & $\frac{2S}{\eps^2} (8C^2 + 
         S\log\frac{2}{\delta})$ & Thm~\ref{thm:coreset_for_sigmoid}\\ 
    \hline 
        $\lgst(\pm\cdot)$ & $P\left(\|x\|_H\geq A\right)= 0$ & $O(\frac{A^2k}{\sqrt{\log(A^2k)}})$ & $O(A^2k)$ & $p$ & $\frac{2S}{\eps^2} (8C^2 + S\log\frac{2}{\delta})$ &  Thm~\ref{thm:coreset_for_logistic} \\
        $\lgst(\pm\cdot)$ & $\expt\limits_{x\sim P}\|x\|_H^2 \leq E_1$ & $O(\frac{E_1^2k}{\sqrt{\log(E_1^2k)}})$ & $O(E_1^2k)$ & $\frac{s}{S}p$ & $\frac{2S}{\eps^2}\left(8C^2 + S\log\frac{2}{\delta})\right)$ & Thm~\ref{thm:coreset_for_logistic_unbounded} \\ 
    \hline
        $\svm(\pm\cdot)$ & $P\left(\|x\|_H\geq A\right)= 0$ & $O(A^2k)$ & $O(A^2k)$ & $p$ & $\frac{2S}{\eps^2} (8C^2 + S\log\frac{2}{\delta})$ & Thm~\ref{thm:coreset_for_svm} \\[.1cm]
        $\svm(\pm\cdot)$ & $\expt\limits_{x\sim P}\|x\|_H^2 \leq E_1$ & $O(E_1^2k)$ & $O(E_1^2k)$ & $\frac{s}{S}p$ & $\frac{2S}{\eps^2}\left(8C^2 + S\log\frac{2}{\delta})\right)$ & Thm~\ref{thm:coreset_for_svm_unbounded}\\
    \hline 
    \end{tabular}
    \label{tab:mainresults}
\end{table}
}
\subsubsection{\bf Kernel Density Estimate}
Let $P$ be a probability measure over $\R^d$ and $\K:\R^d\times \R^d\longrightarrow \R$ be a kernel, for instance the Gaussian kernel $\K(x,w) = \exp(-\|x-w\|^2)$.
At any point $w\in\R^d$, the kernel density estimate is defined ad $\kde_P(w) = \int_{x\in \R^d}\K(x,w)\d P(x).$ 
When we have finite number of points $X=\{x_1,\ldots,x_n\}$ given as the data set, we can assume that $P$ is a uniform probability measure over $X$, i.e.,
$P(x=x_i) = \frac{1}{n}$. In this scenario, 
we use $\kde_X$ instead of $\kde_P$ and and the expression for $\kde_X$ is given by  
$$\kde_X(w) = \frac{1}{n} \sum_{x\in X} \K(x, w).$$
The evaluation of $\kde_X$ demands O(n) time, which can become impractical for massive datasets. Therefore, a frequently employed approach is to substitute $X$ with a significantly smaller dataset $Y$, allowing $\kde_Y$ to function as an approximation for $\kde_X$. Formally, for a given $\eps\in(0,1)$, we look for a small size $Y$ such that $$\sup_{w\in\R^d}|\kde_X(w) - \kde_Y(w)|\leq \eps.$$ This challenge has been thoroughly explored and investigated in various studies such as~\cite{PT2020, Tai2022, charikar2024quasimonte, BMT2017, Taylor2018, GBR2018, PWZ2015, RW2010, Scott2015, Silverman86, ZP2015}. Our primary result in this context is presented in the following theorem, restated as Corollary~\ref{cor:kde} along with its proof. 
\begin{theorem}\label{thm:kde}
    Let $\K:\R^d\times \R^d\longrightarrow (0,1]$ be a reproducing kernel, i.e., a kernel associated with an RKHS, and $P$ be a probability measure over $\R^d$. For $\eps,\delta\in(0,1)$, there exists a universal constant $C$ (independent of $d$ and $\K$) such that if $m \geq \frac{C}{\eps^2}\log \frac{1}{\delta}$, then, with probability at least $1-\delta$, for any iid random sample $X=\{x_1,\ldots,x_m\}$ based on $P$, we have 
    $$\sup_{w\in\R^d}|\kde_P(w) - \kde_X(w)|\leq \eps.$$
\end{theorem}
Note that here, we treat $d$ as a variable, and importantly, our bound is independent of $d$. 
The Gaussian, Laplacian, and Exponential kernels are typically chosen for study.
Compared to the state-of-the-art results, we can reproduce the findings of~\cite{pmlr-v37-lopez-paz15} and~\cite{lacostejulien:hal-01099197}. Nevertheless, there are some improvements in these results, such as works by~\cite{PT2018, PT2020, pmlr-v99-karnin19a}, and more recent papers by~\cite{Tai2022, charikar2024quasimonte}. It is worth noting that some of these results consider $d$ as a constant, while others treat it as a variable. 

\section{Related Works}
In machine learning, a common scenario involves having a set of $n$ data points and the goal of minimizing a data-driven objective loss function.
At times, for scalability reasons, it becomes necessary to choose a smaller subset of $m \ll n$ points. The aim is to minimize the objective function on these points, potentially using non-uniform weights for the selected points. This approach aims to achieve a near-optimal solution for the entire data set.
Coresets are an important tool in scalable machine learning, providing a means to achieve this objective. 
Coresets have found application in various problem domains, such as  clustering~\citep{BHP2002, HM2004, GC2005, FS2008, FL2011, FS2012, FSS2013, BLUZ2019, HV2020, FSS2020}, linear regression~\citep{DMM2006, DDH2009, CWW2019}, principal component analysis~\citep{CEMMP2015, FSS2020}, and more~\citep{Bachem2017PracticalCC, sener2018active, Munteanu18, SW2018, PhT2018, HJLW2018, ABB2019, Mussay2020Data-Independent}.

Improving the prior result by~\cite{FL2011} which suggests that common sensitivity sampling requires $O(dS^2)$ samples, \cite{braverman2022new} and~\cite{FSS2020} proved that $O(dS\log S)$ would be sufficient. Here, $S$ represents the total sensitivity, and $d$ is some VC-dimension associated with the function space. The proof of these results benefits from a connection between 
approximating a class of functions and approximating their linked-ranges spaces, the same concept established by~\cite{JKPV2011} in the study of $\eps$-coresets for kernel density estimates. It is important to note that these results are subject to the assumption that the space is finite. We eliminate this constraint in Theorem~\ref{thm:main_braverman2022new} with a different and more straightforward proof outlined in Section~\ref{prf:ofThmvcDim}. It is notable to highlight that the coreset bounds derived from this theorem are all dependent on the dimension. More importantly, we provide a counterpart of this theorem (see Theorem~\ref{thm:main_rademacher1}) using Rademacher complexity in place of VC-dimension, marking the first result of this kind. As a result, we are able to present several coreset bounds independent of the dimension -- the first such results of this type (see Table~\ref{tab:main_mainresults_R}). In the following, we focus more on the related works pertaining to these findings. 

Our work shares some connections with the research conducted by~\cite{Munteanu18, mai2021coresets, pmlr-v162-tolochinksy22a}, and~\cite{pmlr-v108-samadian20a}. \cite{Munteanu18} demonstrated the existence of datasets for which coresets of sublinear size do not exist. They introduced a complexity measure $\mu(X)$ for data points $X$, quantifying the difficulty associated with compressing a dataset for logistic regression. They demonstrated the existence of coresets with size $O^*(\frac{d^{3/2}\mu(X)\sqrt{|X|}}{\eps^2})$ and $O^*(\frac{d^3\mu^2(X)}{\eps^4})$ through a random sampling procedure ($\O^*(\cdot)$ hides some logarithmic factors in terms of the problem's parameters). As a downside of their method, it is not clear how to compute $\mu(X)$ and also they conjectured that computing the value of $\mu(X)$ in general is hard. Moreover, their second bound is roughly quadratic in terms of $\mu(X)$ and $\frac{1}{\eps^2}$ when the data possesses a small $\mu$-complexity. 
Following their work, \cite{mai2021coresets} improved these bounds to $O^*(\frac{d \mu^2(X)}{\eps^2})$ which is linear in terms of $\frac{1}{\eps^2}$. Their method relies on subsampling data points with probabilities proportional to their $\ell_1$ Lewis weights. As they utilize the complexity measure $\mu(X)$ in their context, which cannot be directly interpreted in terms of our variables, presenting a meaningful comparison between our findings and theirs becomes challenging. Nevertheless, as an advantage, some of our bounds are independent of dimension, surpassing their results in this particular aspect.   

The most closed works to ours are~\cite{pmlr-v162-tolochinksy22a} and~\cite{pmlr-v108-samadian20a}. With an almost similar setting as ours, but restricted to Euclidean space and with a bounded parameter space, \cite{pmlr-v162-tolochinksy22a} found the coresets of size $O^*\left(\frac{kd^2\log n\log k}{\eps^2}\right) $ for 
$\phi = \lgst, \sigma, \svm$. Enhancing some of these results, \cite{pmlr-v108-samadian20a} conducted an analysis of sampling-based coreset constructions designed for regularized loss minimization problems (logistic regression or SVM). Their context is slightly distinct from ours. To rephrase their setting in the context of ours, we can express that in our setting, the regularization term is $\frac{1}{k}$ whereas they roughly use $\frac{1}{n^{1-\kappa}}$ for it (up to a constant).
In the scenario where $k$ is proportional to $n^{1-\kappa}$ for a fixed $\kappa \in (0,1)$, they have demonstrated that a uniform sample of size $O^*\left(\frac{dn^{1-\kappa}}{\epsilon^2}\right)$ functions as a coreset with high probability. Here, $n = |X|$ and  $d$ represents a type of VC-dimension linked to the function space (refer to the paper for the precise definition).They have further established the tightness of uniform sampling, accurate up to poly-logarithmic factors. A special case of our results not only reproduces their results but also improves upon theirs for logistic regression. It is still noteworthy that their bounds are dimension-dependent, whereas we present some bounds independent of dimension. For a comprehensive comparison, refer to Table~\ref{tab:main_comparison1}.

A few other variants of coresets for classification exist; we briefly mention a few.    
\citet{pmlr-v139-munteanu21a} applies sketching to create data compression for logistic regression.  Their methods are data oblivious and show improvements for sparse high-dimensional vectors.  These results are still polynomial in data parameter $\mu(X)$ and dimension $d$.  
\citet{mirzasoleiman2020coresets} considers coresets for regularized classification in the context of incremental gradient descent.  They greedily select a coreset to be used in IGD, and their results use that the $\frac{1}{k}\|w\|^2$ regularizer makes loss function strongly convex as a function of $k$.  
Coresets can also be found for the $0/1$ mis-classification function of size $O(d/\eps^2)$ via classic VC-dimension arguments, but as this cost function is non-convex over the model parameters $w$, the best known algorithms \cite{matheny2021approximate} to solve for the approximately optimal solution are still exponential in dimension $O^*(1/\eps^d)$.  
Or if a dataset is known to be linearly separable, then greedy coresets based on Frank-Wolfe can be found of size roughly $O(1/\eps)$ which approximately preserve that separation margin~\cite{tsang2005core,gartner2009coresets}.

\section{Main Tools}
This section consists of our primary tools and their accompanying proofs essential for establishing our other key findings.
\subsection{Proof of Theorem~\ref{thm:main_rademacher1}}\label{proof:radamahcer_main_thm}
We begin by revisiting McDiarmid's inequality, which serves as a concentration inequality that bounds the difference between the sampled mean and the true mean of a specific function. If $(\XX_1, \Sigma_1, P_1),\ldots, (\XX_n, \Sigma_1, P_n)$ are probability measure spaces, then $P=\prod_{i=1}^n P_i$ is a probability measure over $\prod_{i=1}^n\XX_i$. 
A function $f:\prod_{i=1}^n\XX_i\longrightarrow \R$ satisfies the bounded differences property if there are constants 
$\displaystyle c_{1},c_{2},\dots ,c_{n}$ such that for all 
$\displaystyle i\in [n]$, and for all 
$\displaystyle x_{1}\in {\mathcal {X}}_{1},\,x_{2}\in {\mathcal {X}}_{2},\,\ldots ,\,x_{n}\in {\mathcal {X}}_{n}$
$$\displaystyle \sup _{x_{i}'\in {\mathcal {X}}_{i}}\left|f(x_{1},\dots ,x_{i-1},x_{i},x_{i+1},\ldots ,x_{n})-f(x_{1},\dots ,x_{i-1},x_{i}',x_{i+1},\ldots ,x_{n})\right|\leq c_{i}.$$

\begin{lemma}[ \citealt{mcdiarmid_1989}]\label{lem:GenMcDiarmid}
    Let $\displaystyle f:\prod_{i=1}^n\XX_i\longrightarrow \mathbb {R}$ satisfy the bounded differences property with bounds $\displaystyle c_{1},c_{2},\dots ,c_{n}$. 
    Consider independent random variables $X_{1},X_{2},\dots ,X_{n}$ where $\displaystyle X_{i}\in {\mathcal {X}}_{i}$ for all $i$. Then, for any $\varepsilon >0$, $$\displaystyle {\text{P}}\left(f(X_{1},X_{2},\ldots ,X_{n})-\mathbb {E} [f(X_{1},X_{2},\ldots ,X_{n})]\geq \varepsilon \right)\leq \exp \left(-{\frac {2\varepsilon ^{2}}{\sum _{i=1}^{n}c_{i}^{2}}}\right).$$
\end{lemma}

We are all set to demonstrate our findings for this subsection 
\begin{theorem}[Theorem~\ref{thm:main_rademacher1}, Restated]\label{thm:rademacher}
    Let $(\XX, P, \FF)$ be a positive definite tuple and $s(\cdot)$ be an upper sensitivity function for it with the total sensitivity $S$. Any $s$-sensitivity sample from $\XX$ of size $m$, 
    with probability at least $1-2\exp \left(-{\frac {2mt^{2}}{S^2}}\right)$, satisfies 
    $$\left|\int_x f(x)\d P(x) - \sum_{i=1}^m \frac{S}{m s(x_i)}f(x_i)
    \right|\leq (2\RR^q_m(\mathcal{T}_\FF) + t)\int_x f(x)\d P(x)\quad\quad\forall f\in\FF.$$
\end{theorem}
\begin{proof}
    In view of Equation~\eqref{equation:main_T_f}, we need to show that, 
    for any iid sample $x_1,\ldots,x_m\in\XX$ according to the sensitivity normalized distribution $q(x) = \frac{s(x)}{S}p(x)$,
    with probability at least $1-2\exp \left(-{\frac {2mt^{2}}{S^2}}\right)$, 
    $$\sup_{f\in \FF}\left|1 - \frac{1}{m}\sum_{i=1}^m T_f(x_i)
    \right|\leq 2\RR^q_m(\mathcal{T}_\FF) + t.$$
    For every $x\in \XX$ and $f\in\FF$, we have $\frac{f(x)}{\int_\XX f(z)\d P(z)}\leq s(x)$ which implies 
    $0\leq T_f(x)= \frac{Sf(x)}{s(x)\int_\XX f(z)\d P(z)}\leq S$. Therefore, 
    $\sup\limits_{x\in\XX} T_f(x)\leq S$. 
    Now, define 
    $$g(x_1,\ldots,x_m) = \sup_{f\in \FF}
    \left[
    1 - \frac{1}{m}\sum_{i=1}^m T_f(x_i)
    \right].$$
    For each $i\in[m]$ and each $(x_1',\ldots, x_m')\in \XX^m$ such that $x'_j = x_j$ for $j \neq i$, 
    observe that  
    \begin{align*}
    g(x_1,\ldots,x_m) - g(x_1',\ldots, x_m') & = 
    \sup_{f\in \FF}\left [1 - 
    \frac{1}{m}\sum_{i=1}^m T_f(x_i)\right] - 
    \sup_{f\in \FF}\left[1 - 
    \frac{1}{m}\sum_{i=1}^m T_f(x'_i)\right]\\
    & \leq \sup_{f\in \FF}\left[\left(
    1 - \frac{1}{m}\sum_{i=1}^m T_f(x_i)
    \right ) - \left (
    1 - \frac{1}{m}\sum_{i=1}^m T_f(x'_i)
    \right )\right]\\
    & \leq \sup_{f\in \FF}\left(
    \frac{1}{m}\sum_{i=1}^m T_f(x_i) - \frac{1}{m}\sum_{i=1}^m T_f(x'_i)
    \right)\\
    & \leq \sup_{f\in \FF}\left |\frac{T_f(x_i)-T_f(x'_i)}{m}\right |\leq\frac{S}{m}.
    \end{align*}
    We can repeat the above argument to show  
    $g(x_1,\ldots,x_m) - g(x_1',\ldots, x_m')\leq \frac{S}{m}$.
    Thus $|g(x_1,\ldots,x_m) - g(x_1',\ldots, x_m')|\leq \frac{S}{m}$.
    Therefore, using McDiarmid's Inequality, with a probability at least 
    $1-\exp \left(-{\frac {2mt^{2}}{S^2}}\right)$, we have 
    $$
    \sup_{f\in \FF}
    \left[
    1 - \frac{1}{m}\sum_{i=1}^m T_f(x_i)
    \right]\leq
    \underbrace{\expt_{x_{1:m}\sim q}\left(
    \sup_{f\in \FF}
    \left[
    1 - \frac{1}{m}\sum_{i=1}^m T_f(x_i)
    \right]
    \right)}_{=A} + t.$$
    Similarly, we can prove that, with probability at least $1-\exp \left(-{\frac {2mt^{2}}{S}}\right)$, we have 
    $$
    \sup_{f\in \FF}
    \left[
    \frac{1}{m}\sum_{i=1}^m T_f(x_i) - 1
    \right]\leq
    \underbrace{\expt_{x_{1:m}\sim q}\left(
    \sup_{f\in \FF}
    \left[
    \frac{1}{m}\sum_{i=1}^m T_f(x_i) - 1
    \right]
    \right)}_{=B} + t.$$
    Consequently, since $A,B\geq 0$, with probability at least $1-2\exp \left(-{\frac {2mt^{2}}{S^2}}\right)$ we have 
    $$\sup_{f\in \FF}
    \left|1 - 
    \frac{1}{m}\sum_{i=1}^m T_f(x_i)
    \right|\leq \max\{A,B\} +t.$$
    Let us now deal with  
    $\max\{A,B\}$.
    As $\expt_{x\sim q}(T_f(x)) = 1$, we can write  
    \begin{align*}
    \expt_{x_{1:m}\sim q}\sup_{f\in \FF} & 
    \left[1 - \frac{1}{m}\sum_{i=1}^m T_f(x_i)\right]\\ 
    &= 
    \expt_{x_{1:m}\sim q}\sup_{f\in \FF}\left[
    \expt_{x'_{1:m}\sim q}\frac{1}{m}\sum_{i=1}^m T_f(x'_i) - \frac{1}{m}\sum_{i=1}^m T_f(x_i)\right]\\
    & \leq 
    \expt_{x_{1:m}\sim q}\expt_{x'_{1:m}\sim q}\sup_{f\in \FF}\left[
    \frac{1}{m}\sum_{i=1}^m T_f(x'_i) - \frac{1}{m}\sum_{i=1}^m T_f(x_i)\right]\\
     (*)\qquad & = 
    \expt_{x_{1:m}\sim q}\expt_{x'_{1:m}\sim q}
    \expt_{\sigma_{1:m}\sim \{-1,1\}}
    \sup_{f\in \FF}\left[
    \frac{1}{m}\sum_{i=1}^m \sigma_i\left(T_f(x'_i) - T_f(x_i)\right)\right]\\
    & \leq  
    \expt_{x_{1:m}\sim q}\expt_{x'_{1:m}\sim q}
    \expt_{\sigma_{1:m}\sim \{-1,1\}}\left(
    \sup_{f\in \FF}\left[
    \frac{1}{m}\sum_{i=1}^m \sigma_i T_f(x'_i) \right] +
    \sup_{f\in \FF}\left[
    \frac{1}{m}\sum_{i=1}^m -\sigma_iT_f(x_i)\right]\right)\\
    & = 2 \expt_{x_{1:m}\sim q}
    \expt_{\sigma_{1:m}\sim \{-1,1\}}
    \sup_{f\in \FF}\left[
    \frac{1}{m}\sum_{i=1}^m \sigma_iT_f(x_i)\right]\\
    & = 2\RR^q_m(\mathcal{T}_\FF).
    \end{align*}
The equality marked by $(*)$ is true since $\sigma_i$ indeed exchanges $x_i$ and $x_i'$ together, and since they are iid samples from $q$ it cannot change the expectation. 
With the same approach, we would have 
    $$\expt_{x_{1:m}\sim q}\sup_{f\in \FF} 
    \left[\frac{1}{m}\sum_{i=1}^m T_f(x_i)-1\right]\leq 2\RR^q_m(\mathcal{T}_\FF),$$
    which concludes 
    $\max\{A,B\}\leq 2\RR^q_m(\mathcal{T}_\FF)$
    completing the proof.
\end{proof}
\subsection{Contraction Lemma}\label{apdx:contraction_lemma}
When dealing with Rademacher complexity, we can leverage the helpful lemma known as Talagrand’s Contraction Lemma (refer to Theorem 4.12 in \cite{alma991043798149103276}). This lemma, in particular, establishes that $\RR_m(\phi\circ\FF)\leq L\RR_m(\FF)$, where $\phi\colon \R\longrightarrow \R$ is an $L$-Lipschitz function, and $\phi\circ \FF = \{\phi\circ f\colon f\in\FF\}$. Our requirement extends beyond this lemma to a generalization.

\begin{lemma}
\label{Talagrand_lemma_original}
    Let $\Psi: \R^{\geq 0}\longrightarrow \R^{\geq 0}$ be convex and increasing,  $\boldsymbol{T}\subset \R^m$ a bounded set, and $\eps_1,\ldots,\eps_m >0$. For any $1$-Lipschitz functions $\phi_i\colon \R\longrightarrow \R$ with $\phi_i(0) = 0$ for each $i\in [m]$, we have 
    $$\expt_{\varrho_{_{1:m}}\sim \prod_{i=1}^m\{\pm\eps_i\}}
    \Psi\left(\frac{1}{2}\sup_{\t\in \T}\left|
    \sum_{i=1}^m \varrho_i\phi_i(t_i)
    \right|\right)\leq 
    \expt_{\varrho_{_{1:m}}\sim \prod_{i=1}^m\{\pm\eps_i\}}
    \Psi\left(\sup_{\t\in \T}\left|
    \sum_{i=1}^m \varrho_{_i} t_i
    \right|\right),
    $$ 
    where $\varrho_{_{1:m}}\sim \prod_{i=1}^m \{\pm\eps_i\}$  means $m$ independent samples $\varrho_{_1},\ldots, \varrho_{_m}$ such that each $\varrho_{_i}$ is uniformly drawn from  $\{-\eps_i,\eps_i\}.$
\end{lemma}

Before turning to the proof of this lemma, note that the only difference between this lemma and Theorem~4.12 in~\cite{alma991043798149103276} is that 
$\varrho_{_i}$ here takes its value randomly and uniformly in $\{+\eps_i, -\eps_i\}$ rather than $\{+1, -1\}$. However, the proof is almost identical to that of Theorem~4.12 in~\cite{alma991043798149103276}.

\begin{proof}[Proof of Lemma~\ref{Talagrand_lemma_original}]
    We first prove      
    \begin{equation}\label{with_no_abs}
        \expt_{\varrho_{_{1:m}}\sim \prod_{i=1}^m\{\pm\eps_i\}}
    \Psi\left(\sup_{\t\in \T}
    \sum_{i=1}^m \varrho_{_i}\phi_i(t_i)
    \right)\leq 
    \expt_{\varrho_{_{1:m}}\sim \prod_{i=1}^m\{\pm\eps_i\}}
    \Psi\left(\sup_{\t\in \T}
    \sum_{i=1}^m \varrho_{_i} t_i
    \right).
    \end{equation}
To this end, it siffices to show  
\begin{equation}\label{with_no_abs_mid}
    \expt_{\varrho\in \{\pm\eps\}}
\Psi\left(\sup_{\t\in \T}\left( t_1+ \varrho \phi(t_2)\right)\right)
\leq \expt_{\varrho\in \{\pm\eps\}}
\Psi\left(\sup_{\t\in \T}\left( t_1+ \varrho t_2\right)\right),
\end{equation}
where $\t =(t_1,t_2)\subset \R^2$ is a bounded set and $\phi$ is a $1$-Lipschitz function.
It suffices to prove that for all $\t=(t_1, t_2),\s=(s_1, s_2)\in\T$, we have 
$$\expt_{\varrho\in \{\pm\eps\}}
\Psi\left(\sup_{\t\in \T}\left( t_1+ \varrho t_2\right)\right)
\geq \frac{1}{2} \Psi(t_1+ \eps_2 \phi(t_2)) + \frac{1}{2}\Psi(s_1 - \eps_2 \phi(s_2)) = I.$$
Since $\Psi$ is increasing, without loss of generality, we may assume that 
\begin{equation}\label{eq:s_t}
    t_1+ \eps_2 \phi(t_2)\geq s_1 + \eps_2 \phi(s_2)
\quad\quad\text{and}\quad\quad 
s_1 - \eps_2 \phi(s_2)\geq t_1 - \eps_2 \phi(t_2).
\end{equation}
We distinguish between the following cases.
\begin{enumerate}
    \item {\bf Case $\boldsymbol{t_2,s_2\geq 0}$.} 
    \begin{itemize}
    \item If $s_2\leq t_2$, set 
        $a = s_1-\eps_2\phi(s_2)$, $b = s_1 -\eps_2 s_2$, $a' = t_1 + \eps_2 t_2$, $b' =t_1 + \eps_2 \phi(t_2)$.
        Since $\phi$ is 1-Lipschitz, $\phi(0) = 0$, and $s_2\geq 0$, $|\phi(s_2)|\leq s_2$ which concludes 
        $a\geq b$. Also, by~(\ref{eq:s_t}), it implies
        $b'\geq s_1 +\eps_2\phi(s_2)\geq s_1 - \eps_2 s_2 = b$.
        Furthermore, as $s_2\leq t_2$, 
        we have $\phi(t_2)-\phi(s_2)\leq (t_2-s_2)$ and hence 
        $$a-b = \eps_2(s_2 - \phi(s_2)) \leq \eps_2(t_2 - \phi(t_2)) = a'-b'.$$
        Since $\Psi$ is convex and increasing, for any arbitrary fixed  $x>0$, the function $\Psi(\cdot + x) - \Psi(\cdot)$ is increasing. 
        Setting $x = a-b\geq 0$, since $b\leq b'$, we have 
        $$\Psi(a)-\Psi(b)\leq \Psi(b'+(a-b)) - \Psi(b')$$
        (if $a = b$, it is true always). 
        As 
        $b'+ a - b \leq a'$, we obtain 
        $\Psi(b'+(a-b)) - \Psi(b')\leq \Psi(a') - \Psi(b')$ and consequently, 
        $\Psi(a)-\Psi(b)\leq \Psi(a') - \Psi(b')$ which is equivalent to $2I\leq \Psi(t_1+\eps_2 t_2) + \Psi(s_1 - \eps_2 s_2).$\\

    \item If $t_2\leq s_2$, set $a=t_1+\eps_2\phi(t_2)$, $b=t_1-\eps_2t_2$,
        $a'=s_1+\eps_2s_2$, $b'=s_1-\eps_2\phi(s_2)$. Since $\phi$ is $1$-Lipschitz, $|\phi(t_2)|\leq t_2$ which implies $a\geq b$. Again, by~(\ref{eq:s_t}), 
        $$b' = s_1-\eps_2\phi(s_2)\geq t_1-\eps_2\phi(t_2) \geq t_1-\eps_2 t_2 = b.$$ As $\phi(t_2)-\phi(s_2)\leq s_2-t_2$, 
        $$a-b = \eps_2(t_2+\phi(t_2))\leq \eps_2(s_2+\phi(s_2)) = a'-b'.$$
        Again, as  $\Psi(\cdot + x) - \Psi(\cdot)$ is increasing for each fixed $x>0$ and $b\leq b'$,  by setting $x = a-b\geq 0$, we have 
        $$\Psi(a)-\Psi(b)\leq \Psi(b'+(a-b)) - \Psi(b')$$
        (if $a = b$, it is true always). 
        As $b'+ a - b \leq a'$, we obtain 
        $\Psi(b'+(a-b)) - \Psi(b')\leq \Psi(a') - \Psi(b')$ and consequently, 
        $\Psi(a)-\Psi(b)\leq \Psi(a') - \Psi(b')$ which is equivalent to $2I\leq \Psi(s_1+\eps_2 s_2) + \Psi(t_1 - \eps_2 t_2).$
    \end{itemize} 
    \item {\bf Case $\boldsymbol{t_2,s_2 \leq 0}$.} This case is similar to the previous case. 
    \item {\bf Case $\boldsymbol{t_2\geq 0, s_2 \leq 0}$.} Since $\Psi$ is increasing and  $\phi(t_2)\leq t_2, -\phi(s_2)\leq -s_2$, we have 
    $$2I\leq \Psi(t_1+\eps_2 t_2)+ \Psi(s_1-\eps_2 s_2).$$
    \item {\bf Case $\boldsymbol{t_2\leq 0, s_2 \geq 0}$.} This case follows with a similar argument to Case (3). This completes the proof of~(\ref{with_no_abs_mid}) and consequently the proof of~(\ref{with_no_abs}).    
\end{enumerate}
To conclude the theorem, by convexity of $\Psi$, 
\begin{align*}
    \expt_{\varrho_{_{1:m}}\sim \prod_{i=1}^m\{\pm\eps_i\}}
    \Psi\left(\frac{1}{2}\sup_{\t\in \T}\left|
    \sum_{i=1}^m \varrho_{_i}\phi_i(t_i)
    \right|\right) \leq\ &~\frac{1}{2}
    \expt_{\varrho_{_{1:m}}\sim \prod_{i=1}^m\{\pm\eps_i\}}
    \Psi\left(\sup_{\t\in \T}\left[
    \sum_{i=1}^m \varrho_{_i}\phi_i(t_i)
    \right]^+\right)\\
    & + \frac{1}{2}
    \expt_{\varrho_{_{1:m}}\sim \prod_{i=1}^m\{\pm\eps_i\}}
    \Psi\left(\sup_{\t\in \T}\left[
    \sum_{i=1}^m \varrho_{_i}\phi_i(t_i)
    \right]^-\right)\\
    \leq\ & \expt_{\varrho_{_{1:m}}\sim \prod_{i=1}^m\{\pm\eps_i\}}
    \Psi\left(\sup_{\t\in \T}\left[
    \sum_{i=1}^m \varrho_{_i}\phi_i(t_i)
    \right]^+\right)\\
    =\ & \expt_{\varrho_{_{1:m}}\sim \prod_{i=1}^m\{\pm\eps_i\}}
    \Psi\left(\left[\sup_{\t\in \T}
    \sum_{i=1}^m \varrho_{_i}\phi_i(t_i)
    \right]^+\right)\\
    \text{($*$)}\quad 
    \leq\ & \expt_{\varrho_{_{1:m}}\sim \prod_{i=1}^m\{\pm\eps_i\}}
    \Psi\left(\left[\sup_{\t\in \T}
    \sum_{i=1}^m \varrho_{_i} t_i 
    \right]^+\right)\\ 
    \leq\ & \expt_{\varrho_{_{1:m}}\sim \prod_{i=1}^m\{\pm\eps_i\}}
    \Psi\left(\sup_{\t\in \T}
    \left|\sum_{i=1}^m \varrho_{_i} t_i 
    \right|\right), 
\end{align*}
where ($*$) is true by applying (\ref{with_no_abs}) on the convex and increasing function $\Psi(\max(0, \cdot))$.
This completes the proof. 
\end{proof} 
\begin{corollary}\label{cor:Talagrand}
    Let $\boldsymbol{T}\subset \R^m$ be a bounded set, and $\eps_1,\ldots,\eps_m >0$. For any $L$-Lipschitz functions $\phi_i\colon \R\longrightarrow \R$ with $\phi_i(0) = 0$ for each $i\in[m]$, we have 
    $$\expt_{\varrho_{_{1:m}}\sim \prod_{i=1}^m\{\pm\eps_i\}}
    \sup_{\t\in \T}\left|
    \sum_{i=1}^m \varrho_{_i}\phi_i(t_i)
    \right|\leq 
    2L\expt_{\varrho_{_{1:m}}\sim \prod_{i=1}^m\{\pm\eps_i\}}
    \sup_{\t\in \T}\left|
    \sum_{i=1}^m \varrho_{_i} t_i
    \right|.$$
\end{corollary}
\begin{proof}
    Setting $\Psi$ as identity and replacing $\phi_i$ with $\frac{\phi_i}{L}$, we conclude the result by Lemma~\ref{Talagrand_lemma_original}. 
\end{proof}
In our setting, we only deal with the case that each $\eps_i\in(0,1]$. This particular case can be directly inferred from Theorem~4.12 in~\cite{alma991043798149103276}. Nevertheless, Lemma~\ref{Talagrand_lemma_original} and 
Corollary~\ref{cor:Talagrand}, in the presented general form, can be of independent interest.

\subsection{Proof of Theorem~\ref{thm:monotonic_coreset1}}\label{apx:proof_radamacher_thm}
This subsection primarily focuses on proving Theorem~\ref{thm:monotonic_coreset1} (and Theorem~\ref{thm:main_monotonic_coreset_R}), a crucial tool enabling us to derive our main results. We start with bounding the Radamacher complexity of $\TT_\FF$. 
\begin{lemma}\label{lem:rademacherupper}
For a well-behaved measure $P$ (see 
Definition~\ref{def:well_beaived_P1}), 
if $s(\cdot):\XX\longrightarrow(1,\infty)$ is an upper sensitivity function for $\L^H_{\phi, k}$ with total sensitivity $S=\int_\XX s(x)\d P$, then  $$\RR^q_m(\mathcal{T}_{\L^H_{\phi, k}})\leq C\sqrt{\frac{S}{m}},$$
where $C= \left(2LE_1+\phi(0)\right)\max\left(4E_1k, \frac{1}{E_2}\right) 
        + 8LkE_1 + 1$.
The same statement holds if we replace $\L^H_{\phi, k}$ by $\bar{\L}_{\phi,k}^H$. 
\end{lemma}

\begin{proof}
    Due to similarity, we only work with $\L_{\phi, k}$. 
    We first show that 
        \begin{equation}\label{ineq:upper}
            \expt_{x_{1:m}\sim q}\expt_{\sigma_{1:m}\sim \{-1,1\}}
        \left\|\sum_{i=1}^m\sigma_i\frac{\K(x_i,\cdot)}{s(x_i)}\right\|_H\leq \sqrt{\frac{m}{S}}E_1
        \quad\text{and}\quad
            \expt_{x_{1:m}\sim q}\expt_{\sigma_{1:m}\sim \{-1,1\}}
        \left|\sum_{i=1}^m\sigma_i\frac{1}{s(x_i)}\right|\leq \sqrt{\frac{m}{S}}.
        \end{equation}
    Because of the similarity, we only prove the first one. 
    To this end, using Jensen's inequality for the concave function $t\mapsto \sqrt{t}$, we obtain  
    \begin{align*}
        \expt_{x_{1:m}\sim q}
        \expt_{\sigma_{1:m}\sim \{-1,1\}} & 
        \left\|\sum_{i=1}^m\sigma_i\frac{\K(x_i,\cdot)}{s(x_i)}\right\|_H \\
        & \leq \sqrt{\expt_{x_{1:m}\sim q}
        \expt_{\sigma_{1:m}\sim \{-1,1\}}
          \left\|\sum_{i=1}^m\sigma_i\frac{\K(x_i,\cdot)}{s(x_i)}\right\|^2_H }\\
          & = \sqrt{\expt_{x_{1:m}\sim q}
        \expt_{\sigma_{1:m}\sim \{-1,1\}}
          \left\langle\sum_{i=1}^m\sigma_i\frac{\K(x_i,\cdot)}{s(x_i)}, \sum_{i=1}^m\sigma_i\frac{\K(x_i,\cdot)}{s(x_i)}\right\rangle_H }\\
        & = \sqrt{\expt_{x_{1:m}\sim q}
        \expt_{\sigma_{1:m}\sim \{-1,1\}}\left[\sum_{i=1}^m\sigma_i^2\frac{\K(x_i,x_i)}{s(x_i)^2}
          +\sum_{i\neq j}\sigma_i\sigma_j\frac{\K(x_i,x_j)}{s(x_i)s(x_j)}
          \right] }\\
        & = \sqrt{\expt_{x_{1:m}\sim q}
          \left[\sum_{i=1}^m\frac{\K(x_i,x_i)}{s^2(x_i)}
          \right] }\\
        \text{\footnotesize{(since $q(x) = \frac{s(x)}{S}p(x)$)}}\quad
        & =  \frac{1}{\sqrt{S}}\sqrt{\expt_{x_{1:m}\sim p}
          \left[\sum_{i=1}^m\frac{\K(x_i,x_i)}{s(x_i)}
          \right] }\\
        \text{\footnotesize{(since $s(x)\geq 1$)}}\quad 
        & \leq \frac{1}{\sqrt{S}}\sqrt{\expt_{x_{1:m}\sim p}
          \left[\sum_{i=1}^m\K(x_i,x_i)
          \right] }\\
        & = \frac{1}{\sqrt{S}}\sqrt{m\expt_{x\sim p}\K(x,x)}
        \leq \sqrt{\frac{m}{S}}E_1.
    \end{align*}
    For simplicity set, $f_w(x) = \ell_{\phi,k}^H(w,x)$ and $\alpha(w)=\frac{S}{\int_\XX f_w(x)\d P(x)}$. 
    During the proof, we need some good upper bounds for $\alpha(w)\|w\|_H^2$ and $\alpha(w)$. 
    Note that $\int_\XX f_w(x)\d P(x) = \frac{1}{k}\|w\|_H^2 + \int_\XX \phi(\K(w,x))\d P(x)\geq \frac{1}{k}\|w\|_H^2$ which implies $\alpha(w)\|w\|_H^2\leq Sk$. 
    Also,
     \begin{align*}
         \int_\XX f_w(x)\d P(x) & = \frac{1}{k}\|w\|_H^2 
         + \int_\XX \phi(\K(w, x))\d P(x)\\ 
         &\geq \frac{1}{k}\|w\|_H^2 + \int_\XX \phi(\|w\|_H \|x\|_H)\d P(x) & \text{($\phi$ is non-increasing \&  Cauchy–Schwarz ineq.)}\\
         &\geq \left\{
         \begin{array}{lc}
            \frac{1}{4E_1^2k}  &  \|w\|_H\geq \frac{1}{2E_1}\\\\
            \int_\XX \phi(\frac{\|x\|_H}{2E_1})\d P(x)  & \|w\|_H\leq \frac{1}{2E_1}
         \end{array}\right.\\
         &\geq \left\{
         \begin{array}{lc}
            \frac{1}{4E_1k}  &  \|w\|_H\geq \frac{1}{2E_1}\\\\
            E_2  &  \|w\|_H\leq \frac{1}{2E_1}
         \end{array}\right.\\
         & = \min\left(\frac{1}{4E_1k}, E_2\right).
     \end{align*}
     This concludes $\alpha(w)\leq \max\left(4SE_1k, \frac{S}{E_2}\right)$ for each $w$.   
     Note that if 
     $2^n\leq \|w\|_H\leq 2^{n+1}$, then 
     \begin{align*}
         \int_\XX f_w(x)\d P(x) & = \frac{1}{k}\|w\|_H^2 
         + \int_\XX \phi(\K(w, x))\d P(x) \geq \frac{2^{2n}}{k}
     \end{align*}
     which implies 
     $\alpha(w) \leq \frac{kS}{2^{2n}}.$   So, we obtained 
     \begin{equation}\label{eq:alpha_w}
         \alpha(w)\leq \left\{\begin{array}{cc}
     \max\left(4SE_1k, \frac{S}{E_2}\right)&  \|w\|_2\leq 1\\
     \frac{kS}{2^{2n}} & 2^n\leq \|w\|_H\leq 2^{n+1}.
\end{array}\right.
     \end{equation}
     To bound the supremum within the Radamacher complexity, we utilize these inequalities to partition the parameter space. It is necessary as we want to use Lemma~\ref{cor:Talagrand} in which we need a bounded property.   
     Consequently, 
    \begin{align*}
        \RR^q_m(\mathcal{T}_{\L_{\phi,k}^H}) = \ &  
        \expt_{x_{1:m}\sim q}
        \expt_{\sigma_{1:m}\sim \{-1,1\}}
        \sup_{f\in \L_{\phi,k}^H}\left[
        \frac{1}{m}\sum_{i=1}^m \sigma_iT_f(x_i)\right]\\ 
        = \ & \expt_{x_{1:m}\sim q}
        \expt_{\sigma_{1:m}\sim \{-1,1\}}
        \sup_{w\in \XX} \underbrace{\frac{S}{\int_\XX f_w(x)\d \mu}}_{=\alpha(w)} \left[
        \frac{1}{m}\sum_{i=1}^m \sigma_i \frac{f_w(x_i)}{s(x_i)}\right]\\
        = \ & \expt_{x_{1:m}\sim q}
        \expt_{\sigma_{1:m}\sim \{-1,1\}}
        \sup_{w\in \XX} \alpha(w) \left[
        \frac{1}{m}\sum_{i=1}^m \sigma_i \frac{\phi(\K(w,x_i)) +\frac{1}{k}\|w\|_H^2}{s(x_i)}\right]\\
        \leq \ & \underbrace{\expt_{x_{1:m}\sim q}
        \expt_{\sigma_{1:m}\sim \{-1,1\}}
        \sup_{w\in \XX} \alpha(w) \left|
        \frac{1}{m}\sum_{i=1}^m \sigma_i \frac{\phi(\K(w,x_i))}{s(x_i)}\right|}_{= M}\\
        & \quad \quad +\frac{1}{k}\underbrace{\expt_{x_{1:m}\sim q}
        \expt_{\sigma_{1:m}\sim \{-1,1\}}
        \sup_{w\in \XX} \alpha(w) \left|
        \frac{1}{m}\sum_{i=1}^m \sigma_i \frac{\|w\|_H^2}{s(x_i)}\right|}_{= N}\\
        \textbf{($*$)}\quad 
         \leq\  & \left[2LE_1 \max\left(4E_1k, \frac{1}{E_2}\right) 
        + 8LkE_1 + \max\left(4E_1k, \frac{1}{E_2}\right)\phi(0) + 1\right]\sqrt{\frac{S}{m}}.
    \end{align*}
    To prove ($*$), we deal with $M$ and $N$ separately.
    Note that $\Bar{\phi}(\cdot) = \phi(\cdot)-\phi(0)$ is an $L$-Lipschitz functions with $\Bar{\phi}(0)=0$ and thus satisfies the condition of Corollary~\ref{cor:Talagrand}. 
    Note that 
    \begin{align*}
        M \leq 
        &  \underbrace{\expt_{x_{1:m}\sim q}
        \expt_{\sigma_{1:m}\sim \{-1,1\}}
        \sup_{w\in \XX} \alpha(w) \left|
        \frac{1}{m}\sum_{i=1}^m \sigma_i \frac{\bar{\phi}(\K(w,x_i))}{s(x_i)}\right|}_{=M_1}\\ 
        & + \underbrace{\expt_{x_{1:m}\sim q}
        \expt_{\sigma_{1:m}\sim \{-1,1\}}
        \sup_{w\in \XX} \alpha(w) \left|
        \frac{1}{m}\sum_{i=1}^m \sigma_i \frac{\phi(0)}{s(x_i)}\right|}_{=M_2}\\
        \leq & M_1 + \phi(0)\max\left(4SE_1k, \frac{S}{E_2}\right)\expt_{x_{1:m}\sim q}
        \expt_{\sigma_{1:m}\sim \{-1,1\}}
        \left|
        \frac{1}{m}\sum_{i=1}^m  \frac{\sigma_i}{s(x_i)}\right|\\
        \leq & M_1 + \sqrt{\frac{S}{m}}\max\left(4E_1k, \frac{1}{E_2}\right)\phi(0) & \text{by~(\ref{ineq:upper})}.
    \end{align*}
    To upper bound $M_1$, define 
    $$M_1^0 = \expt_{x_{1:m}\sim q}
        \expt_{\sigma_{1:m}\sim \{-1,1\}}
        \sup_{\|w\|_H\leq 1} \alpha(w)\left|
        \frac{1}{m}\sum_{i=1}^m \sigma_i \frac{\bar{\phi}(\K(w,x_i))}{s(x_i)}\right|$$
    and, for $n\in \mathbb{N}$, 
    $$M_1^n = \expt_{x_{1:m}\sim q}
        \expt_{\sigma_{1:m}\sim \{-1,1\}}
        \sup_{2^{n-1}\leq \|w\|\leq 2^n} \alpha(w)\left|
        \frac{1}{m}\sum_{i=1}^m \sigma_i \frac{\bar{\phi}(\K(w,x_i))}{s(x_i)}\right|$$
    and note that $M_1\leq \sum_{n=0}^\infty M_1^n$.
    We can write 
    \begin{align*}
        M_1^0  
        & = \expt_{x_{1:m}\sim q}
        \expt_{\sigma_{1:m}\sim \{-1,1\}}
        \sup_{\|w\|_H\leq 1} \alpha(w)\left|
        \frac{1}{m}\sum_{i=1}^m \sigma_i \frac{\bar{\phi}(\K(w, x_i))}{s(x_i)}\right|\\
        \text{\footnotesize{(since $\alpha(w)\leq \max(4SE_1k, \frac{S}{E_2})$})}\quad
        & \leq \max\left(4SE_1k, \frac{S}{E_2}\right)\expt_{x_{1:m}\sim q}
        \expt_{\sigma_{1:m}\sim \{-1,1\}}
        \sup_{\|w\|_H\leq 1} \frac{1}{m}\left|
        \sum_{i=1}^m \sigma_i \frac{\bar{\phi}(\K(w, x_i))}{s(x_i)}\right|\\
        \text{\footnotesize{(Cor.~\ref{cor:Talagrand} with $\eps_i = \frac{1}{s(x_i)}$)}}\quad
        & \leq 2L \max\left(4SE_1k, \frac{S}{E_2}\right)
        \expt_{x_{1:m}\sim q}
        \expt_{\sigma_{1:m}\sim \{-1,1\}}
        \sup_{\|w\|_H\leq 1} \left|
        \frac{1}{m}\sum_{i=1}^m \sigma_i \frac{\K(w, x_i)}{s(x_i)}\right|\\
        & = \frac{2L}{m}\max\left(4SE_1k, \frac{S}{E_2}\right) 
        \expt_{x_{1:m}\sim q}
        \expt_{\sigma_{1:m}\sim \{-1,1\}}
        \sup_{\|w\|_H\leq 1} \left|
          \left\langle \K(w,\cdot), \sum_{i=1}^m\sigma_i\frac{\K(x_i, \cdot)}{s(x_i)}\right\rangle_H\right|\\
          \text{\footnotesize{(Cauchy–Schwarz inequality)}}\quad
        & \leq \frac{2L}{m} \max\left(4SE_1k, \frac{S}{E_2}\right) \expt_{x_{1:m}\sim q}
        \expt_{\sigma_{1:m}\sim \{-1,1\}}
        \sup_{\|w\|_H\leq 1} \|w\|_H
        \left\|\sum_{i=1}^m\sigma_i\frac{\K(x_i, \cdot)}{s(x_i)}\right\|_H\\
        & \leq \frac{2L}{m} \max\left(4SE_1k, \frac{S}{E_2}\right) \expt_{x_{1:m}\sim q}
        \expt_{\sigma_{1:m}\sim \{-1,1\}}
        \left\|\sum_{i=1}^m\sigma_i\frac{\K(x_i, \cdot)}{s(x_i)}\right\|_H\\
        \text{\footnotesize{(by (\ref{ineq:upper}))}}\quad
        & \leq  2LE_1\max\left(4E_1k, \frac{1}{E_2}\right)\sqrt{\frac{S}{m}}.
    \end{align*}
and similarly, 
    \begin{align*}
        M_1^n  
        & = \expt_{x_{1:m}\sim q}
        \expt_{\sigma_{1:m}\sim \{-1,1\}}
        \sup_{2^{n-1}\leq \|w\|_H\leq 2^{n}} \alpha(w)\left|
        \frac{1}{m}\sum_{i=1}^m \sigma_i \frac{\bar{\phi}(\K(w, x_i))}{s(x_i)}\right|\\
        \text{\footnotesize{(Since $\alpha(w) \leq \frac{kS}{2^{2(n-1)}}$)}}\quad
        & \leq \frac{kS}{m4^{n-1}}\expt_{x_{1:m}\sim q}
        \expt_{\sigma_{1:m}\sim \{-1,1\}}
        \sup_{2^{n-1}\leq \|w\|_H\leq 2^{n}}\left|
        \sum_{i=1}^m \sigma_i \frac{\bar{\phi}(\K(w, x_i))}{s(x_i)}\right|\\
        \text{\footnotesize{(Cor.~\ref{cor:Talagrand} with $\eps_i = \frac{1}{s(x_i)}$)}}\quad
        & \leq \frac{2LkS}{m4^{n-1}}
        \expt_{x_{1:m}\sim q}
        \expt_{\sigma_{1:m}\sim \{-1,1\}}
        \sup_{2^{n-1}\leq \|w\|_H\leq 2^{n}} \left|
        \sum_{i=1}^m \sigma_i \frac{\K(w, x_i)}{s(x_i)}\right|\\
        & = \frac{2LkS}{m4^{n-1}}
        \expt_{x_{1:m}\sim q}
        \expt_{\sigma_{1:m}\sim \{-1,1\}}
        \sup_{2^{n-1}\leq \|w\|_H\leq 2^{n}} \left|
          \left\langle \K(w,\cdot), \sum_{i=1}^m\sigma_i\frac{\K(x_i,\cdot)}{s(x_i)}\right\rangle_H\right|\\
        & \leq \frac{2LkS}{m4^{n-1}} \expt_{x_{1:m}\sim q}
        \expt_{\sigma_{1:m}\sim \{-1,1\}}
        \sup_{2^{n-1}\leq \|w\|_H\leq 2^{n}} \|w\|_H
        \left\|\sum_{i=1}^m\sigma_i\frac{\K(x_i,\cdot)}{s(x_i)}\right\|_H\\
        & \leq \frac{LkS}{m2^{n-3}} \expt_{x_{1:m}\sim q}
        \expt_{\sigma_{1:m}\sim \{-1,1\}}
        \left\|\sum_{i=1}^m\sigma_i\frac{\K(x_i,\cdot)}{s(x_i)}\right\|_H\\
        \text{\footnotesize{(by (\ref{ineq:upper}))}}\quad
        & \leq \frac{LkE_1}{2^{n-3}}\sqrt{\frac{S}{m}}.
    \end{align*}
Therefore, 
\begin{align*}
    M_1 & \leq M_1^0 + \sum_{n=1}^\infty M_1^n\\
        & \leq 2LE_1\max\left(2E_1k, \frac{1}{E_2}\right)\sqrt{\frac{S}{m}}
        + \sum_{n=1}^\infty\frac{LkE_1}{2^{n-3}}\sqrt{\frac{S}{m}}\\
        & = 2LE_1 \max\left(2E_1k, \frac{1}{E_2}\right) \sqrt{\frac{S}{m}}
        + 8LkE_1\sqrt{\frac{S}{m}}
\end{align*}
and thus
$$M\leq \left[2LE_1 \max\left(4E_1k, \frac{1}{E_2}\right) 
        + 8LkE_1 + \max\left(4E_1k, \frac{1}{E_2}\right)\phi(0)\right]\sqrt{\frac{S}{m}}.$$

To complete the proof, we need to upper bound $N$ which is done as follows
\begin{align*}
    N & = \expt_{x_{1:m}\sim q}
        \expt_{\sigma_{1:m}\sim \{-1,1\}}
        \sup_{ w\in \XX} \alpha(w) \left|
        \frac{1}{m}\sum_{i=1}^m \sigma_i \frac{\|w\|_H^2}{s(x_i)}\right|\\
        & = \frac{1}{m}\expt_{x_{1:m}\sim q}
        \expt_{\sigma_{1:m}\sim \{-1,1\}}
        \sup_{ w\in \R^d} \alpha(w)\|w\|_H^2 \left[
        \sum_{i=1}^m \sigma_i \frac{1}{s(x_i)}\right]\\
        \text{\footnotesize{(since $\alpha(w)\| w\|^2\leq Sk$)}}\quad\quad
        &\leq \frac{Sk}{m} \expt_{x_{1:m}\sim q}
        \expt_{\sigma_{1:k}} \left|\sum_{i=1}^m \frac{\sigma_i}{s(x_i)}\right|\\
        \text{\footnotesize{(by~(\ref{ineq:upper}))}}\quad\quad 
        &\leq \sqrt{\frac{S}{m}}k,
\end{align*}
which completes the proof. 
\end{proof}

We are now in a position to state the following conclusion. 
\begin{theorem}\label{thm:monotonic_coreset}{\rm [Theorem~\ref{thm:monotonic_coreset1}, Restated]}
For a well behaved probability measure $P$ (see Definition~\ref{def:well_beaived_P1}), if
$s(\cdot):\XX\longrightarrow(1,\infty)$ is an upper sensitivity function for $\L^H_{\phi, k}$ with total sensitivity $S$ and $m\geq \frac{2S}{\eps^2}\left(8C^2 + S\log\frac{2}{\delta})\right)$,  then, with probability at least $1-\delta$, any $s$-sensitivity sample $x_1,\ldots,x_m$ from $\XX$ with 
weights $u_i = \frac{S}{m s(x_i)}$
provides an $\eps$-coreset for $(\XX,\ P,\ \XX,\ \ell^H_{\phi, k})$,
where 
$$C= (2LE_1+\phi(0))\max(4E_1k, \frac{1}{E_2}) + 8LkE_1 + 1.$$
The statement remains true if we replace $\L^H_{\phi, k}$ and  $(\XX,\ P,\ \XX,\ \ell^H_{\phi, k})$ by  $\bar{\L}^H_{\phi, k}$ and $(\XX,\ P,\ \XX,\ \bar{\ell}^H_{\phi, k})$. 
\end{theorem}
\begin{proof}
    The proof immediately follows from Theorem~\ref{thm:rademacher} and Lemma~\ref{lem:rademacherupper}. 
\end{proof}
When $E_1$ is dimension independent, this theorem provides a no-dimensional $\eps$-coreset for $(\R^d,\ P,\ \R^d,\ \ell_{\phi, k})$, that is, whose size is independent of the dimension of the space. In the following subsection, we present some applications of this theorem. 
This theorem holds significant importance in deriving the dimension-free results outlined in Table~\ref{tab:mainresults}. 

\section{Coresets for Monotonic Functions}\label{sec:coresets_for_monotonic_functions}
Let $H$ be a reproducing kernel Hilbert space of real-valued functions on $\XX$ with kernel $\K\colon \XX\times \XX\longrightarrow \R$ and $P$ a probability measure over $\XX$. 
For a non-increasing $L$-Lipschitz function $\phi\colon \R\longrightarrow(0,\infty)$, we remind that $\ell^H_{\phi,k}:\XX\times \XX\longrightarrow\R$ where 
     $\ell^H_{\phi,k}(x,w) = \phi(\K(x, w)) + \frac{1}{k}\|w\|_H^2.$
An $\eps$-coreset for $(\XX, P, \L^H_{\phi, k})$, called \emph{$\eps$-coreset for a monotonic function $\phi$ with respect to $H$}, is a set $X=\{x_1,\ldots, x_m\} \subseteq \XX$ accompanied by a measure (or weight function) $\nu$ such that 
\begin{equation}\label{def:coreset_for_c1_H}
   \left| \int_{x\in\XX} \ell^H_{\phi, k}(x,w)\d P(x) - \sum_{i=1}^m \nu(x_i) \ell^H(x_i,w) \right| \leq \epsilon \int_{x\in\XX} \ell^H_{\phi, k}(x,w)\d P(x) \quad \forall w \in \XX. 
\end{equation}
We start with a straightforward observation. 
\begin{lemma}\label{lem:f_increase_K}
    If $\phi:\R\longrightarrow(0,\infty)$ is a non-increasing function and $0< K_1 \leq K_2$, then 
    $$\frac{\phi(-t)+\frac{t^2}{K_1}}{\phi(t)+\frac{t^2}{K_1}}
    \leq \frac{\phi(-t)+\frac{t^2}{K_2}}{\phi(t)+\frac{t^2}{K_2}} 
    \quad\quad \text{for each } t\geq 0.$$
\end{lemma}
\begin{proof}
    For $t\geq 0$, we should check the inequality 
$$
\left(\phi(-t)+\frac{t^2}{K_1}\right)\left(\phi(t)+\frac{t^2}{K_2}\right)\leq 
\left(\phi(-t)+\frac{t^2}{K_2}\right)\left(\phi(t)+\frac{t^2}{K_1}\right)$$
which is equivalent to 
$$
\phi(-t)\frac{t^2}{K_2} + \phi(t)\frac{t^2}{K_1} \leq
\phi(-t)\frac{t^2}{K_1} + \phi(t)\frac{t^2}{K_2}
$$
which is valid since 
\begin{align*}
\phi(-t)t^2\left(\frac{1}{K_1}-\frac{1}{K_2}\right) + 
\phi(t)t^2\left(\frac{1}{K_2}-\frac{1}{K_1}\right) 
& =     t^2\left(\frac{1}{K_1}-\frac{1}{K_2}\right)
\left(\phi(-t)-\phi(t)\right)\geq 0,
\end{align*}
completing the proof.
\end{proof}

A version of the lemma described below appears in \citep{pmlr-v162-tolochinksy22a}. We carefully revisit the ideas in their proof in order to extend and enhance the lemma.
 
\begin{lemma}\label{lem:beta_property}
    Let $\phi:\R\longrightarrow(0,\infty)$ be a non-increasing function such that  
    \begin{equation}\label{beta_property}
        \frac{\phi(-\alpha z)+\frac{z^2}{k}}{\phi(\alpha z)+\frac{z^2}{k}}\leq \beta(\alpha)\quad\quad 0\leq \alpha  \leq B_1, 0\leq z\leq B_2.
    \end{equation}
    If we set $M=\phi(-B_1B_2)$, then, for each $x,y, w\in \XX$ with $\|x\|_H, \|y\|_H\leq B_1, \|w\|_H\leq B_2$, we have   
    $$\frac{\phi(0)}{M \beta(\|x\|_H)}\leq \frac{\ell^H_{\phi,k}(x,w)}{\ell^H_{\phi,k}(y,w)} \qquad \text{and} \qquad \frac{\phi(0)}{M \beta(\|x\|_H)}\leq \frac{\bar{\ell}^H_{\phi,k}(x,w)}{\bar{\ell}^H_{\phi,k}(y,w)}.$$
    If $\phi$ is universally bounded by $M$, then 
    we do not need upper bounds $B_1, B_2$ for $\alpha, z$ and consequentially do not need upper bounds for $\|x\|_H,\|y\|_H$, and $\|w\|_H$, and the same statement holds.  
\end{lemma}
\begin{proof}
    Due to similarity, we only prove the lemma for $\ell^H_{\phi, k}$. 
    First, note that as $\phi$ is non-increasing, $\beta(\alpha)\geq 1$ for each $\alpha$. 
    Consider arbitrary $x,y, w\in \XX$ with $\|x\|_H, \|y\|_H\leq B_1, \|w\|_H\leq B_2$. In the following, we deal with the two different cases $\K(x,w)\leq 0$ and $\K(x,w)>0$ separately.  
    If $\K(x, w) \leq 0$, then
\begin{align*}
    \phi(\K(y, w)) 
    & = \phi(\langle y, w\rangle_H)\\
    & \leq \phi(-\|y\|_H\|w\|_H)
    & \text{($\phi$ is non-increasing along with Cauchy–Schwarz inequality)}\\ 
    & \leq \phi(-B_1B_2)\\
    & = M =  \frac{M}{\phi(0)}\phi(0)\\
    & \leq \frac{M}{\phi(0)}\phi(\K(x, w))
    & \text{(since $\K(x,w) \leq 0$ and $\phi$ is non-increasing)}
\end{align*}
and hence, as $\frac{M}{\phi(0)}, \beta(\|x\|_H)\geq 1$, by adding $\frac{\|w\|_H^2}{k}$ to both sides, we obtain  
\begin{equation}\label{eq:1}
    \ell_{\phi,k}(y, w) \leq \frac{M}{\phi(0)}\ell^H_{\phi,k}(x, w)\leq \frac{M}{\phi(0)}\beta(\|x\|_H)\ell^H_{\phi,k}(x, w).
\end{equation}
Now, assume that $\K(x, w) > 0$.
Similarly,
\[
    \phi(\K(y, w)) \leq M 
    \leq \frac{M}{\phi(0)}\phi(-\K(x, w))
    \leq \frac{M}{\phi(0)}\phi(-\|x\|_H\|w\|_H).
\]
Again, by adding $\frac{\|w\|_H^2}{k}$ to the both sides, we have  
\begin{align*}
    \ell^H_{\phi,k}(y, w) = \phi(\K(y, w)) + \frac{\|w\|_H^2}{k} 
    &\leq \frac{M}{\phi(0)}\left(\phi(-\|x\|_H\|w\|_H)+ \frac{\|w\|_H^2}{k}\right)\\
    \text{(Property \ref{beta_property})}\quad
    &\leq \frac{M}{\phi(0)}\beta(\|x\|_H)\left(\phi(\|x\|_H\|w\|_H)+ \frac{\|w\|_H^2}{k}\right)\\
    \text{(Cauchy-Schwarts Inequality and $\phi$ is non-increasing)}\quad
    &\leq \frac{M}{\phi(0)}\beta(\|x\|_H)\underbrace{\left(\phi(\K(x, w))+ \frac{\|w\|_H^2}{k}\right)}_{= \ell^H_{\phi,k}(x, w)},
\end{align*}
which implies
\begin{equation}\label{eq:2}
    \ell^H_{\phi,k}(y, w) \leq \frac{M}{\phi(0)}\beta(\|x\|_H)\ell^H_{\phi,k}(x, w).
\end{equation}
Combining Equations~(\ref{eq:1}) and~(\ref{eq:2}), we obtain the desired inequality. 
\end{proof}
This lemma provides us with a function $\gamma(\alpha) = \frac{\phi(0)}{M \beta(\alpha)}\leq 1$ ensuring
$\gamma(\|x\|_H)\leq \frac{\ell^H_{\phi,k}(x,w)}{\ell^H_{\phi,k}(y,w)}$. 
The subsequent lemma will highlight the usefulness of this function.
\begin{lemma}\label{lem:sensitivity_beta}
Let $P$ be a probability measure over $\XX$.
Assume that $\W\subseteq \XX$, $\ell: \XX\times \XX \longrightarrow(0,\infty)$, and 
    $\gamma:[0,\infty)\longrightarrow [0,\infty)$ such that  $0<\int_\XX \gamma(\|x\|_H)\d P(x)<\infty$ 
     and   
    \begin{equation}\label{def:beta_property}
        \gamma(\|x\|_H)\leq \frac{\ell(x,w)}{\ell(y,w)} \quad\quad \forall x,y\in \XX, w\in \W.
    \end{equation}
Then 
$$s(y)= \frac{1}{\int_\XX \gamma(\|x\|_H)\d P(x)}$$
is an upper sensitivity function for $(\XX, P, \L_W)$, where 
$\L_\W = \{l(\cdot,w)\colon w\in \W\}.$
\end{lemma}
\begin{proof}
  For a $y\in \XX$, we have 
\begin{align*}
    \sup_{w\in\W}\frac{\ell(y,w)}{\int_\XX \ell(x,w)\d P(x)} 
    & = \sup_{w\in\W}\frac{1}{\int_X \frac{\ell(x,w)}{\ell(y,w)}\d P(x)}\\
    & \leq \sup_{w\in\W}\frac{1}{\int_X \gamma(\|x\|_H)\d P(x)}\\
    & = \frac{1}{\int_X \gamma(\|x\|_H)\d P(x)}
\end{align*}
completing the proof. 
\end{proof}

It is worth noting that the provided upper sensitivity function $s(y) = \frac{1}{\int_\XX \gamma(\|x\|)\d P}$  remains constant, irrespective of $y\in \XX$. 
In a specific scenario where $\gamma$ is lower bounded by $L$, the function $s(y) = \frac{1}{L}$ serves as an upper sensitivity function, possessing a total value of $S=\frac{1}{L}$. It is important to note that when $\phi$ and $k$ grantee Property~(\ref{def:beta_property}) for $\ell=\ell^H_{\phi,k}$ (like what is established in Lemma~\ref{lem:beta_property}), then Lemma~\ref{lem:sensitivity_beta} provides a methodology to calculate the upper sensitivity function and its total value for $\L^H_{\phi,k}$. This enables us to obtain $s$-sensitivity samples from $\XX$, making Theorem~\ref{thm:monotonic_coreset} (and Theorem~\ref{thm:main_braverman2022new}) applicable in this context.

\subsection{Coreset for Sigmoid function}
\citet{pmlr-v162-tolochinksy22a} examined Property~(\ref{def:beta_property}) for specific functions $\phi$. 
They demonstrated for the sigmoid function $\sigma(x) = \frac{1}{1+e^x}$ that, given $\alpha>0$, there exists a threshold $k_\alpha$ such that for any $k\geq k_\alpha$
\begin{equation}\label{eq:k_alpha1}
    \frac{\sigma(-\alpha z)+\frac{z^2}{k}}{\sigma(\alpha z)+\frac{z^2}{k}}\leq 66k\alpha
    \quad\quad\forall z \geq 0. 
\end{equation}
Using this result, they established the first item of Theorem~\ref{thm:main_TJFlogisticsigmoin1} in which, as a drawback, $k$ should be sufficiently large contingent on the given data points $\XX$. This dependency comes from the role of $k_\alpha$ in Equation~\eqref{eq:k_alpha1}.
In the subsequent lemma, we eliminate this dependency.
\begin{lemma}\label{lem:beta_for_sigmoid}
    For $\sigma(x) = \frac{1}{1+e^x}$ and fixed $k>0$, we have 
    $$\frac{\sigma(-\alpha z)+\frac{z^2}{k}}{\sigma(\alpha z)+\frac{z^2}{k}}\leq 
           4\left(1 + \max(e, \alpha^2 k)\right)\quad\quad \forall\alpha, z \geq 0.$$
\end{lemma}

\begin{proof}
    One can verify that $\sigma:\R\longrightarrow(0,1)$ is a positive decreasing function and $\sigma'(t) = -\sigma(t)\sigma(-t)$.
    For a fixed $\alpha>0$, if we set $t = \alpha z$ and $K = \alpha^2 k$, 
    $$\frac{\sigma(-\alpha z)+\frac{z^2}{k}}{\sigma(\alpha z)+\frac{z^2}{k}} = \frac{\sigma(-t)+\frac{t^2}{K}}{\sigma(t)+\frac{t^2}{K}}\leq \frac{1+\frac{t^2}{K}}{\sigma(t)+\frac{t^2}{K}} = f(t).$$
    To find the maximum of $f(t)$ on $(0,\infty)$ in terms of $K$, we begin by calculating $f'(t)$, the derivative of $f$, after factoring out $1/(\sigma(t)+\frac{t^2}{K})^2$, as follows:
\begin{align*}
    \left(\sigma(t)+\frac{t^2}{K}\right)^2f'(t) 
    = & \frac{2t}{K}\left(\sigma(t)+\frac{t^2}{K}\right) 
    - \left(\sigma'(t) + \frac{2t}{K}\right)\left(1+\frac{t^2}{K}\right)\\
    = & \sigma(t)\frac{2t}{K} - \sigma'(t) - \frac{2t}{K} 
    - \sigma'(t) \frac{t^2}{K}\\
    = & \sigma(t)\frac{2t}{K} + \sigma(t)\sigma(-t) - \frac{2t}{K} 
    + \sigma(t)\sigma(-t) \frac{t^2}{K}\\
    = & \frac{2t}{K}\left(\sigma(t) - 1\right) + \sigma(t)\sigma(-t)\left(1+\frac{t^2}{K}\right)\\
    \text{\footnotesize (since $\sigma(t) +\sigma(-t) = 1$)}\quad
    = & -\sigma(-t)\frac{2t}{K} + \sigma(t)\sigma(-t)\left(1+\frac{t^2}{K}\right)\\
    = &  \sigma(-t)\left[-\frac{2t}{K} + \sigma(t)\left(1+\frac{t^2}{K}\right)\right].
\end{align*}
So, as $(\sigma(t) + t^2/K)^2>0$, whenever 
$\sigma(t)\left(1+\frac{t^2}{K}\right) < \frac{2t}{K}$ or equivalently $\frac{K}{2t} + \frac{t}{2}< 1 + e^t,$
$f'$ is negative. 
For 
$$h(t) = 1+e^t-\frac{t}{2}-\frac{K}{2t},$$
we have 
$$h'(t) = e^t-\frac{1}{2} + \frac{K}{2t^2} > 0\qquad\qquad \text{for each } t\geq 0.$$
This indicates that $h$ is increasing, implying it can have at most one zero.
Since $h$ is negative for small $t$ and positive for large $t$, it follows that $h$ possesses a unique zero, denoted as $t_0$. 
This implies that $f$ is decreasing for $t\geq t_0$ and increasing for $t\leq t_0$, thus attaining its maximum at $t_0$.  In the following discussion, we distinguish between two cases: $K\geq e$ and $K\leq e$. Initially, we focus on bounding $\sup_{t\in(0,\infty)}\frac{1+\frac{t^2}{K}}{\sigma(t)+\frac{t^2}{K}}$ for the scenario where $K\geq e$. Subsequently, we employ this bound, along with a straightforward observation, to derive a bound for $\sup_{t\in(0,\infty)}\frac{1+\frac{t^2}{K}}{\sigma(t)+\frac{t^2}{K}}$ when $K\leq e$.
\begin{enumerate}
    \item {\bf Case 1: $\boldsymbol{K\geq e}.$ }
        We start with $t\geq \log K\geq 1$. 
        As $h$ is increasing 
        $$h(t)\geq h(\log K) = 1+K - \frac{\log K}{2} - \frac{K}{2\log K}\geq 1.$$
        This means that $h$ is positive for $t\geq \log K$.
        Now, assume that $t\leq 0.5\log K$.
        Again, because $h$ is increasing, we have
        $$h(t)\leq h(0.5\log K)
        \leq  1 + \sqrt{K}-\frac{\log K}{4} - \frac{K}{\log K} < 0.  
        $$
        These observations together yield  
        $t_0\in \left[0.5\log K,\ \log K\right]$. 
        Consequently,
        \begin{align*}
            \sup_{t\in(0,\infty)}\frac{1+\frac{t^2}{K}}{\sigma(t)+\frac{t^2}{K}}
            & = \frac{1+\frac{t_0^2}{K}}{\sigma(t_0)+\frac{t_0^2}{K}}\\
            & \leq \frac{1+\frac{\log^2 K}{K}}{\sigma(\log K)+\frac{\log^2 K}{4K}}\\
            & \leq \frac{1+\frac{\log^2 K}{K}}{\frac{\log^2 K}{4K}}\\
            & \leq 4\left(1 + \frac{K}{\log^2K}\right)\leq 4(1 + K).
        \end{align*}
    \item {\bf Case 2: $\boldsymbol{K\leq e}.$} 
        Calculation shows that if $0 < K_1 \leq K_2$,
        then 
        $$\frac{1+\frac{t^2}{K_1}}{\sigma(t)+\frac{t^2}{K_1}}\leq \frac{1+\frac{t^2}{K_2}}{\sigma(t)+\frac{t^2}{K_2}}\quad\quad \forall t\geq 0$$
        which concludes 
        $$\sup_{t\in(0,\infty)}\frac{1+\frac{t^2}{K_1}}{\sigma(t)+\frac{t^2}{K_1}}\leq \sup_{t\in(0,\infty)}\frac{1+\frac{t^2}{K_2}}{\sigma(t)+\frac{t^2}{K_2}}.$$
        Consequently, for $K\leq e$,  
        $$\sup_{t\in(0,\infty)}\frac{1+\frac{t^2}{K}}{\sigma(t)+\frac{t^2}{K}}\leq \sup_{t\in(0,\infty)}\frac{1+\frac{t^2}{e}}{\sigma(t)+\frac{t^2}{e}}\leq 4(1+e),$$
        where the right most inequality comes from Case (1).
\end{enumerate}
Putting Cases (1) and (2) together, we obtain 
\begin{align*}
    \max_{t\in [0,t_0]} f(t) 
    &\leq \left\{
    \begin{array}{lc}
        4(1 + K)  &  K\geq e\\ \\
        4(1+e) &  K < e.
    \end{array}
    \right.
\end{align*}
So, replacing $K$ with $\alpha^2 k$, we obtain
$$\frac{\sigma(-\alpha z)+\frac{z^2}{k}}{\sigma(\alpha z)+\frac{z^2}{k}}\leq \beta(\alpha) = \left\{
        \begin{array}{cc}
           4(1 + \alpha^2 k)  & \alpha^2 k\geq e\\ \\
           4(1+e)  & \alpha^2 k \leq e
        \end{array}
    \right.
\quad\quad \forall\alpha, z > 0.$$
Note that for $\alpha =0$ or $z=0$, the inequality is valid as well. 
\end{proof}

\begin{lemma}\label{lem:sensitivity_sigmoid}
    Assume that $H$ is a reproducing kernel Hilbert space of real-valued functions on $\XX$ with kernel $\K\colon \XX\times \XX\longrightarrow \R$ and $P$ is a probability measure over $\XX$ such that 
    $\expt\limits_{x\sim P}\|x\|_H^2 \leq E^2_1.$ 
    Then $s(x) = 60 + 32kE_1^2$ is an upper sensitivity function for $(\XX,\ P,\ \ell^H_{\sigma,k})$ and $(\XX,\ P,\ \bar{\ell}^H_{\sigma,k})$.
\end{lemma}
\begin{proof}
Given $\expt_{x\sim P}(\|x\|_H^2)\leq E^2_1$,
Markov's inequality  implies that  
$\pr_{x\sim P}(\|x\|_H\geq 2E^2_1)\leq \frac{1}{2}$  and  consequently $\pr_{x\sim P}(\|x\|_H\leq 2E^2_1)\geq \frac{1}{2}$ which will be used later to derive inequality marked by in $(*)$.  Since $\sigma$ is universally bounded by $1$.  
Lemma~\ref{lem:beta_for_sigmoid} indicates that Lemma~\ref{lem:beta_property} 
is applicable with  
$$\beta_\sigma(\alpha)= \left\{
        \begin{array}{cc}
           4(1 + \alpha^2 k)  & \alpha^2 k\geq e\\ \\
           4(1+e)  & \alpha^2 k \leq e
        \end{array}
    \right.$$
and for all $x,y,w\in\XX$ and $M=1$. Using Lemma~\ref{lem:sensitivity_beta} with 
$\gamma(\|x\|_H) = \frac{\sigma(0)}{M \beta_\sigma(\|x\|_H)} = \frac{1}{2 \beta_\sigma(\|x\|_H)}$,
we conclude that 
$s(y) = \frac{2}{\int_{\R^d}\frac{1}{\beta_\sigma(\|x\|_H)}\d P}$ 
is an upper sensitivity function for 
$(\XX,\ P,\ \XX,\ \ell^H_{\sigma,k})$ and 
$(\XX,\ P,\ \XX,\ \bar{\ell}^H_{\sigma,k})$.
Note that  
\begin{align*}
    \int_{\XX}\frac{1}{\beta_\sigma(\|x\|_H)}\d P(x)  
    &\geq  \int_{\{x: e \leq k\|x\|_H^2\leq 2kE^2_1\}}\frac{1}{4(1+\|x\|_H^2 k)}\d P(x)
    + \int_{\{x: k\|x\|_H^2< e\}}\frac{1}{4(1+e)}\d P(x)\\
    & \geq \frac{1}{4(1+2kE_1^2)}P\left(\left\{x: \frac{e}{k} \leq \|x\|_H^2\leq 2E_1^2\right\}\right)
    + \frac{1}{4(1+e)}P\left(\left\{x: \|x\|_H^2< \frac{e}{k}\right\}\right)\\
    & \geq \pr\left(\left\{x: \|x\|_H^2\leq 2E^2_1\right\}\right) \min\left(\frac{1}{4(1+2kE^2_1)}, \frac{1}{4(1+ e)}\right)\\
    (*)\qquad\qquad
    & \geq \frac{1}{2}\min\left(\frac{1}{4(1+2kE^2_1)}, \frac{1}{4(1+ e)} 
    \right) \\
    & \geq \frac{1}{8(1+e+2kE^2_1)}\geq \frac{1}{30 + 16kE_1^2}.
\end{align*}
This concludes $s(x) = 60+32kE_1^2$ is an upper sensitivity function for $(\XX, P,  \ell^H_{\sigma,k})$ and $(\XX, P,  \bar{\ell}^H_{\sigma,k})$.
\end{proof}
Now, we are at a point to put Theorem~\ref{thm:monotonic_coreset1} into action.

\begin{theorem}[No Dimensional Sigmoid Coreset]\label{thm:coreset_for_sigmoid}
    Assume that $H$ is a reproducing kernel Hilbert space of real-valued functions on $\XX$ with kernel $\K\colon \XX\times \XX\longrightarrow \R$,  $P$ is a probability measure over $\XX$ such that 
    $\expt\limits_{x\sim p}\|x\|_H^2 \leq E^2_1$,  $S=60 + 32kE_1^2$, and 
    $C=  \left(2E_1+\frac{1}{2}\right)\max\left(4E_1k, \frac{2}{5}\right) + 8kE_1 + 1.$ 
    For $m \geq \frac{2S}{\eps^2} (8C^2 + S\log\frac{2}{\delta})$,  
    any iid sample $x_1,\ldots,x_m$ from $P$ with 
weights $u_i = \frac{1}{m}$
provides an $\eps$-coreset for $(\XX, P, \XX, \ell^H_{\sigma, k})$  {\rm (}respectively for $(\XX, P, \XX,\bar{\ell}^H_{\sigma, k})$ {\rm )} with probability at least $1-\delta$. 
\end{theorem}
\begin{proof}
    Lemma~\ref{lem:sensitivity_sigmoid} indicates that  $s(x) = 60+32kE^2_1$ is an upper sensitivity function. Given that it is a constant function, $s$-sensitivity sampling is equivalent to sampling according to $P$.
    Notice that $\sigma$ is a $1$-Lipschitz function. 
    As $\log (\cdot)$ is a concave function, using Jensen's inequality, we have 
    $\expt\left(\log\sigma\left(\frac{\|x_i\|_H}{2E_1}\right)\right)\leq 
    \log\expt\left(\sigma\left(\frac{\|x_i\|_H}{2E_1}\right)\right).$
    On the other hand, since 
    $-\log \sigma(t) = \log(1+e^t)\leq 1+t,$
     we have 
    $$\expt\left(\log\sigma\left(\frac{\|x_i\|_H}{2E_1}\right)\right)
    \geq -\expt \left(1+\frac{\|x_i\|_H}{2E_1}\right) 
    = -1 -\frac{\expt\|x_i\|_H}{2E_1}
    \geq -\frac{3}{2},$$ 
    which concludes 
    $$\expt\left(\sigma\left(\frac{\|x_i\|_H}{2E_1}\right)\right)
    \geq e^{\expt\left(\log\sigma\left(\frac{\|x_i\|_H}{2E_1}\right)\right)}\geq e^{-\frac{3}{2}} \geq 0.4 =  E_2.$$
    Hence,  $P$ is well behaved (see Definition~\ref{def:well_beaived_P1}), thereby allowing Theorem~\ref{thm:monotonic_coreset1} to establish the statement.
\end{proof}
While Theorem~\ref{thm:coreset_for_sigmoid} holds for any reproducing kernel Hilbert space meeting the specified criteria, the most intriguing instance arises in Hilbert space $\R^d$ equipped with the standard inner product (also known as the dot product) as its kernel. In this scenario, we can also employ Theorem~\ref{thm:main_braverman2022new}, which leads to the subsequent theorem.
\begin{theorem}\label{thm:coreset_for_sigmoid_R}
    Let $P$ be a probability measure over $\R^d$ such that 
    $\expt\limits_{x\sim p}\|x\|_2^2 \leq E_1$ and 
    $S=60 + 32kE_1^2$.
    There is an $m = O\left(\frac{S}{\eps^2}\left(d\log S + \log\frac{1}{\delta})\right)\right)$ such that 
    any iid sample $x_1,\ldots,x_m$ from $P$ with weights $u_i = \frac{1}{m}$ provides an $\eps$-coreset for $(\R^d, P, \R^d, \ell_{\sigma, k})$ {\rm (} respectively for  $(\R^d, P, \R^d, \bar{\ell}_{\sigma, k})${\rm )} with probability at least $1-\delta$.
\end{theorem}
\begin{proof}
    Because of similarity, we only prove the statement for $(\R^d, P, \R^d, \ell_{\sigma, k})$.
    Given that Lemma~\ref{lem:sensitivity_sigmoid} asserts the constant function $s(x) = 60+32kE^2_1$ as an upper sensitivity function for $\L_{\sigma, k}$, we deduce $\ranges(\mathcal{T}_{\L_{\sigma, k}}, \succ) = \ranges(\L_{\sigma, k}, \succ)$. To clarify, every function in $\mathcal{T}_{\L_{\sigma, k}}$ is a function in $\L_{\sigma, k}$ scaled by a positive constant. This implies   
    $\vc(\ranges(\mathcal{T}_{\L_{\sigma, k}}, \succ)) = \vc(\ranges(\L_{\sigma, k}, \succ)).$ 
    We would like to remind that  
    $$\mathcal{T}_{_{\L_{\sigma, k}}} = 
\left\{f_w(\cdot) = \frac{\sigma(\langle  w, \cdot\rangle) + \frac{1}{k}\|w\|_2^2}{s(\cdot)} \colon  w \in\R^d\right\}.
$$ 
For an $f_w\in \L_{\sigma, k}$ and $r\geq 0$,
\begin{align*}
    \range(f_w, \succ, r) 
    & = \left\{x \in \R^d \colon f_w(x)> r\right\}\\
    & = \left\{x \in \R^d \colon \sigma(\langle x,w\rangle) + \frac{\|w\|^2}{k} > r\right\}\\
    & = \left\{x \in \R^d \colon \frac{1}{1+e^{\langle x,w\rangle}} > \underbrace{r - \frac{\|w\|^2}{k}}_{=t}\right\}\\
    & = \left\{\begin{array}{ll}
      \R^d   &  t\leq 0\\ 
      \varnothing   &  t\geq 1\\ 
       \left\{x\in\R^d\colon \langle x,w\rangle< \log (\frac{1}{t}-1)\right\}  & 0 < t < 1,
    \end{array}\right.
\end{align*}
which concludes that $\ranges(\L_{\sigma, k}, \succ)$ only includes half-spaces, empty set, and the whole space $\R^d$.  Therefore, by Radon's theorem, $\vc(\ranges(\L_{\sigma, k}, \succ))\leq d+1$ (for a proof, see Lemma 10.3.1 in ~\cite{Matousek2002}). Leveraging Theorem~\ref{thm:main_braverman2022new}, we derive the statement for $m = O(\frac{S}{\eps^2}\left(d\log S + \log\frac{1}{\delta}\right))$, thereby completing the proof. 
\end{proof}

Theorems~\ref{thm:coreset_for_sigmoid} and \ref{thm:coreset_for_sigmoid_R} can be considered as an advancement of the first item of Theorem~\ref{thm:main_TJFlogisticsigmoin1}. Theorem~\ref{thm:main_TJFlogisticsigmoin1} is applicable when $\XX$ is a finite set inside the unit ball, and $m$ is function that is linear in terms of $\log(|\XX|)$ and quadratic in terms $d$, whereas in Theorem~\ref{thm:coreset_for_sigmoid}, $\XX$ can be infinite, and moreover, the presented $m$ remains constant concerning those parameters (assuming $E_1$ is independent of $d$). Moreover, in Theorem~\ref{thm:coreset_for_sigmoid_R}, $m$ linearly depends on $d$, whereas in Theorem~\ref{thm:main_TJFlogisticsigmoin1}, it has a quadratic dependence on $d$.

\subsection{Coreset for Logistic Function}
In binary logistic regression, minimizing the negative log-likelihood involves working with the logistic function $\lgst(t) = \log(1+e^{-t})$. Therefore, having small coresets for logistic functions holds significant value. In this section, we establish an upper bound on the size of a coreset for the logistic function. 
\begin{lemma}\label{lem:beta_for_logistic}
    For $\lgst(x) = \log (1+e^{-t})$ and fixed $k>0$, we have 
    $$\frac{\lgst(-\alpha z)+\frac{z^2}{k}}{\lgst(\alpha z)+\frac{z^2}{k}}\leq \left\{
    \begin{array}{cc}
        \frac{85\alpha^2k}{\log(\alpha^2k)}  &  \alpha^2 k\geq e\\ \\ 
        85               &  \alpha^2 k\leq e.
    \end{array}\right. \quad\quad \forall\alpha, z \geq 0.$$
\end{lemma}
\begin{proof}
Simplifying notation, let $\phi(t) = \lgst(t)$.
For a given $\alpha>0$, if we set $t = \alpha z$ and $K = \alpha^2 k$, 
$$\frac{\phi(-\alpha z)+\frac{z^2}{k}}{\phi(\alpha z)+\frac{z^2}{k}} = \frac{\phi(-t)+\frac{t^2}{K}}{\phi(t)+\frac{t^2}{K}} = f_K(t)\quad \quad t\geq 0.$$
Note that $\phi'(t) = -\sigma(t) = \frac{-1}{1+e^{t}}$.
Upon computing the derivative of $f_K$, we obtain 
\begin{align*}
    \left(\phi(t)+\frac{t^2}{K}\right)^2f_K'(t) 
    & =\left(-\phi'(-t)+\frac{2t}{K}\right)\left(\phi(t)+\frac{t^2}{K}\right)
    - \left(\phi'(t)+\frac{2t}{K}\right)\left(\phi(-t)+\frac{t^2}{K}\right)\\
    & = -\phi'(-t)\phi(t) - \phi'(-t)\frac{t^2}{K} 
    + \phi(t)\frac{2t}{K} -\phi'(t)\phi(-t) 
    - \phi'(t)\frac{t^2}{K} - \phi(-t)\frac{2t}{K}\\
    & = -\left(\phi'(-t)\phi(t)+ \phi'(t)\phi(-t)\right) + \left(\phi(t)- \phi(-t)\right)\frac{2t}{K} - (\phi'(t)+\phi'(-t))\frac{t^2}{K}\\
    & = \left(\sigma(-t)\phi(t) + \sigma(t)\phi(-t)\right) + \left(\phi(t)- \phi(-t)\right)\frac{2t}{K} + (\underbrace{\sigma(t)+\sigma(-t)}_{=1})\frac{t^2}{K}\\
    & = \left((1-\sigma(t))\phi(t) + \sigma(t)\phi(-t)\right) + \left(\phi(t)- \phi(-t)\right)\frac{2t}{K} + (\underbrace{\sigma(t)+\sigma(-t)}_{=1})\frac{t^2}{K}\\
    & = \phi(t) + \underbrace{\left(\phi(t)- \phi(-t)\right)}_{=-t\leq 0}\left(\frac{2t}{K} - \sigma(t)\right) 
    + \frac{t^2}{K} = h(t).
\end{align*}
    We investigate the two cases $K\geq e$ and $K\leq e$ separately. 
    First, assume that $K\geq e$. In the subsequent analysis, our goal is to determine $0<t_1<t_2$ such that $h(t) > 0$ (equivalently $f_K$ is increasing) for $t\in[0,t_1]$ and $h(t) < 0$ (equivalently $f_K$ is decreasing) for $t\in[t_1,\infty)$. This concludes that $f_K$ takes its maximum at some $t_0\in[t_1,t_2]$. 
    Considering the formula of $h(t)$, 
    if $\frac{2t}{K} \leq \sigma(t)$ or equivalently $g_1(t) = 1+e^t - \frac{K}{2t}\leq 0$, then $h(t)$ is positive. Since the function $g_1$ is increasing, for $t\in(0,\ 0.4\log K]$,
    $$g_1(t)\leq g_1(0.4\log k) = 1 +\sqrt[5]{K^2} - \frac{5K}{4\log K}<0\quad\Longrightarrow \quad h(t) > 0.$$ 
    Thus, $f_K$ is increasing for $t\in[0,\ 0.4\log K]$. 
    Note \begin{equation}\label{eq:logisistic_neg}
        h(t) < 0\quad\quad \Longleftrightarrow  \quad\quad 
    \phi(t) + \frac{t^2}{K} < t\left(\frac{2t}{K} - \sigma(t)\right)\quad\quad \Longleftrightarrow \quad\quad t\sigma(t)+\phi(t) < \frac{t^2}{K}.
    \end{equation}    
    As $t\sigma(t)+\phi(t) < te^{-t} + e^{-t} = e^{-t}(1+t)$, having $g_2(t) = e^{-t}(1+t) - \frac{t^2}{K}<0$ implies $h(t)<0.$ 
    One can see that $g_2$ is decreasing and 
    $g_2(2\log K) <0$. Therefore, for $t\geq t_1 = 2\log K$, $g_2(t) <0$ and hence $h(t)<0$ which concludes 
    $f$ is decreasing for $t\geq 2\log K$. 
    So far, we have seen that $f$ is increasing for $t\in[0,\ 0.4\log K]$ and decreasing for $t\geq 2\log K$, which yields
\begin{align*}
    \sup_{t\in\left[0,\ \infty\right)} f_K(t) 
    & = \sup_{t\in\left[0.4\log K,\ 2\log K\right]} f_K(t)\\
    & = \sup_{t\in\left[0.4\log K,\ 2\log K\right]} \frac{\phi(-t)+\frac{t^2}{K}}{\phi(t)+\frac{t^2}{K}} \\
    & \leq \frac{\phi(-2\log K)+\frac{(2\log K)^2}{K}}{\frac{4\log^2 K}{25K}}\\
    \text{\footnotesize{(since $\phi(-t)\leq 1+t$)}}\quad
    & \leq \frac{1+2\log K+\frac{4\log^2 K}{K}}{\frac{4\log^2 K}{25K}}\\
    & = \frac{25K}{4\log^2 K} + \frac{25K}{2\log K}
        + 25\\
    & \leq 25 + \frac{75K}{4\log K}.
\end{align*} 
    Up to this point, we have established $\sup_{t\in[0,\ \infty)} f_K(t)\leq 25 + \frac{75K}{4\log K}$ under the condition that $K\geq e$.
    To conclude the argument, we need to determine an upper bound for $\sup_{t\in[0,\ \infty)} f_K(t)$ when $0< K\leq e$. Utilizing Lemma~\ref{lem:f_increase_K}, for $0<K\leq e$, we deduce  
    $$\sup_{t\in[0,\ \infty)} f_K(t) \leq \sup_{t\in[0,\ \infty)} f_e(t) \leq  25 + \frac{75e}{4}\leq 85
    $$ which concludes  
    \begin{align*}
        \sup_{t\in[0,\ \infty)} f_K(t)& \leq \left\{
    \begin{array}{cc}
        25 + \frac{75K}{4\log K}  &  K\geq e\\ \\ 
        85  &  K\leq e 
    \end{array}
    \right.\\
    & \leq \left\{
    \begin{array}{cc}
        \frac{85\alpha^2k}{\log(\alpha^2k)}  &  \alpha^2 k\geq e\\ \\ 
        85               &  \alpha^2 k\leq e.
    \end{array}\right.
    \end{align*} 

\end{proof}

\begin{lemma}\label{lem:sensitivity_logistic}
     Assume that $H$ is a reproducing kernel Hilbert space of real-valued functions on $\XX$ with kernel $\K\colon \XX\times \XX\longrightarrow \R$ and $P$ is a probability measure over $\XX$ such that  
    $P\left(\left\{x\in\XX\colon \|x\|_H\geq A\right\}\right)= 0$.
    Then $s(y) = 1 + \frac{340(1+kA^2)}{\sqrt{\max(1,\log(kA^2))}}$ is an upper sensitivity function for $(\XX, P, \L_{\lgst})$ and $(\XX, P, \bar{\L}_{\lgst})$.
\end{lemma}
\begin{proof}
As $P\left(\left\{x\in\XX\colon \|x\|_H\geq A\right\}\right)= 0$, we only need to determine $s(y)$ for $
y$ with $\|y\|_H\leq A$. Henceforth, we assume that $\|y\|_H\leq A$. For $w\in \W = \{w\in \XX\colon \|w\|_H\leq \frac{1}{A}\sqrt{\max(1,\log(A^2k))}\}$, 
Lemma~\ref{lem:beta_for_logistic} indicates that Lemma~\ref{lem:beta_property} 
is applicable with  $B_1 = A, B_2 = \frac{1}{A}\sqrt{\max(1,\log(A^2k))}, \phi(-B_1B_2)\leq 1 + \sqrt{\max(1,\log(A^2k))} = M$, and 
$$\beta_{\lgst}(\alpha)= \left\{
    \begin{array}{cc}
        \frac{85\alpha^2k}{\log(\alpha^2k)}  &  \alpha^2 k\geq e\\ \\ 
        85               &  \alpha^2 k\geq e.
    \end{array}\right.$$
Using Lemma~\ref{lem:sensitivity_beta} with 
$\gamma(\|x\|_H) = \frac{\lgst(0)}{M \beta_{\lgst}(\|x\|_H)} \geq \frac{1}{2M \beta_{\lgst}(\|x\|_H)}$,
we conclude that 
$$\sup_{\|w\|_H\leq B_2} \frac{\ell^H_{_{\lgst, k}}(y, w)}{\int_{x\in \XX} \ell^H_{_{\lgst, k}}(x, w)\d P(x)}\leq \frac{2M}{\int_{x\in\XX}\frac{1}{\beta_{\lgst}(\|x\|_H)}\d P(x)}.$$ 
Note that  
\begin{align*}
    \int_{x\in\XX}\frac{1}{\beta_{\lgst}(\|x\|_H)}\d P(x)  
    &\geq  \int_{\{x: e \leq k\|x\|_H^2\}}\frac{\log (k\|x\|_H^2)}{85k\|x\|_H^2}\d P(x)
    + \int_{\{x: k\|x\|_H^2< e\}}\frac{1}{85}\d P(x)\\
    & \geq \frac{\log (kA^2)}{85kA^2}P\left(\left\{x: \frac{e}{k} \leq \|x\|_H^2\right\}\right)
    + \frac{1}{85}P\left(\left\{x: \|x\|_H^2< \frac{e}{k}\right\}\right)\\
    & \geq \left\{
    \begin{array}{cc}
        \frac{\log (kA^2)}{85kA^2}  &  A^2 k\geq e\\ \\ 
        \frac{1}{85}               &  A^2 k < e
    \end{array}\right.\\
    & \geq \frac{\max(1,\log(kA^2))}{85(1+kA^2)}.
\end{align*}
This concludes 
\begin{align*}
    \sup_{\|w\|_H\leq B_2} \frac{\ell^H_{_{\lgst, k}}(y, w)}{\int_{x\in \XX} \ell^H_{_{\lgst, k}}(x, w)\d P(x)} & \leq \frac{170(1+kA^2)}{\max(1,\log(kA^2))}\left(1 + \sqrt{\max(1,\log(kA^2))}\right)\\
    & \leq \frac{340(1+kA^2)}{\sqrt{\max(1,\log(kA^2))}}.
\end{align*}
As 
$\int_{x\in \XX} \ell_{_{\lgst, k}}(x, w)\d P\geq \frac{1}{k}\|w\|_H^2$,  we have  
\begin{align*}
    \sup_{\|w\|_H\geq B_2} \frac{\ell_{_{\lgst, k}}(y, w)}{\int_{x\in \XX} \ell_{_{\lgst, k}}(x, w)\d P} & \leq 
    \sup_{\|w\|\geq B_2}\frac{\frac{1}{k}\|w\|^2 + \log(1+e^{-\langle y, w\rangle})}{\frac{1}{k}\|w\|^2}\\
    & \leq \sup_{\|w\|_H\geq B_2}\left[1 + \frac{k}{\|w\|_H^2}\log(1+e^{-\langle y, w\rangle})\right]\\
    & \leq \sup_{\|w\|_H\geq B_2}\left[1 + \frac{k}{\|w\|_H^2}\log(1+e^{\|y\|_H\|w\|_H})\right]\\
    \text{(since $\log(1+e^t)\leq 1+t$ for $t\geq 0$)}\quad\quad
    & \leq \sup_{\|w\|_H\geq B_2}\left[1 + \frac{k}{\|w\|_H^2}\left(1 + \|y\|_H\|w\|_H\right)\right]\\
    \text{(since $\|w\|_H\geq B_2$ and $\|y\|_H\leq A$)}\quad\quad
    &\leq 1 + \frac{kA^2}{\max(1,\log (A^2k))} + \frac{kA^2}{\sqrt{\max(1,\log (A^2k))}}\\
    & \leq 1 + \frac{2kA^2}{\sqrt{\max(1,\log (A^2k))}}.
\end{align*}
So, $s(y) = 1 + \frac{340(1+kA^2)}{\sqrt{\max(1,\log(kA^2))}}$ is an upper sensitivity 
function for $(\XX, P, \L^H_{\lgst,k})$.
The proof for $(\XX, P, \bar{\L}^H_{\lgst,k})$ proceeds in a similar manner.
\end{proof}

\begin{theorem}[No Dimensional Logistic Coreset]\label{thm:coreset_for_logistic}
    Assume that $H$ is a reproducing kernel Hilbert space of real-valued functions on $\XX$ with kernel $\K\colon \XX\times \XX\longrightarrow \R$,  $P$ is probability measure over $\XX$ such that 
    $P\left(\left\{x\in \XX\colon \|x\|_H\geq A\right\}\right)= 0$,  $S=1 + \frac{340(1+kA^2)}{\sqrt{\max(1,\log(kA^2))}}$, and $C= (2A+1)\max(4Ak, 2.5)+8Ak + 1$. 
    For $m \geq \frac{2S}{\eps^2} (8C^2 + S\log\frac{2}{\delta})$, any iid sample $x_1,\ldots,x_m\in \XX$ according to $P$ with weights $u_i = \frac{1}{m}$ provides an $\eps$-coreset for $(\XX, P, \XX, \ell^H_{\lgst, k})$ 
    {\rm (}respectively for $(\XX, P, \XX, \bar{\ell}^H_{\lgst, k})$ {\rm )}
    with probability at least $1-\delta$.
\end{theorem}
\begin{proof}
    Lemma~\ref{lem:sensitivity_logistic} asserts that $s(y) = 1 + \frac{340(1+kA^2)}{\sqrt{\max(1,\log(kA^2))}}$ serves as an upper sensitivity function for $(\XX,\ P,\ \L^H_{\lgst,k})$ and  $(\XX, P, \XX, \bar{\ell}^H_{\lgst, k})$. Since it is a constant function, $s$-sensitivity sampling is equivalent to sampling according to $P$. It is worth noting that $\lgst$ is a $1$-Lipschitz, convex, and decreasing function. 
    For $E_1=A$,  
    we have 
     $$\expt\limits_{x\sim p}\lgst\left(\frac{\|x\|_H}{2E_1}\right)\geq 
     \lgst\left(\frac{1}{2}\right)\geq \frac{2}{5} = E_2.$$
    This implies that Definition~\ref{def:well_beaived_P1} is satisfied for $\L^H_{\lgst, k}$ and thus  
    Theorem~\ref{thm:monotonic_coreset1} concludes the statement for $m\geq \frac{2S}{\eps^2}\left(8C^2 + S\log\frac{2}{\delta}\right)$. 
\end{proof}

In line with the most compelling scenario, Hilbert space $\R^d$, utilizing the standard inner product as its kernel, we present the following theorem. 
\begin{theorem}\label{thm:coreset_for_sigmoid_R_1}
    Let $P$ be a probability measure over $\R^d$ such that 
    $P\left(\left\{x\in\XX\colon \|x\|_H\geq A\right\}\right)= 0$ and 
    $S=1 + \frac{340(1+kA^2)}{\sqrt{\max(1,\log(kA^2))}}$.
    There is an $m = O\left(\frac{S}{\eps^2}\left(d\log S + \log\frac{1}{\delta}\right)\right)$ such that 
    any iid sample $x_1,\ldots,x_m$ from $P$ with weights $u_i = \frac{1}{m}$ provides an $\eps$-coreset for $(\R^d, P, \R^d, \ell_{\lgst, k})$  {\rm (}respectively for $(\XX, P, \XX, \bar{\ell}^H_{\lgst, k})$ {\rm )} with probability at least $1-\delta$.
\end{theorem}
\begin{proof} 
    Because of similarity, we only handle the case $(\R^d, P, \R^d, \ell_{\lgst, k})$. Given that Lemma~\ref{lem:sensitivity_logistic} asserts the constant function 
    $s(x) = 1 + \frac{340(1+kA^2)}{\sqrt{\max(1,\log(kA^2))}}$ as an upper sensitivity function for $\L_{\lgst, k}$, we deduce $\ranges(\mathcal{T}_{\L_{\lgst, k}}, \succ) = \ranges(\L_{\lgst, k}, \succ)$. To clarify, every function in $\mathcal{T}_{\L_{\lgst, k}}$ is a function in $\L_{\lgst, k}$ scaled by a positive constant. This implies   
    $\vc(\ranges(\mathcal{T}_{\L_{\lgst, k}}, \succ)) = \vc(\ranges(\L_{\lgst, k}, \succ)).$ 
    We would like to remind that  
    $$\mathcal{T}_{_{\L_{\lgst, k}}} = 
\left\{f_w(\cdot) = \frac{\lgst(\langle  w, \cdot\rangle) + \frac{1}{k}\|w\|_2^2}{s(\cdot)} \colon  w \in\R^d\right\}.
$$ 
For an $f_w\in \L_{\lgst, k}$ and $r\geq 0$,
\begin{align*}
    \range(f_w, \succ, r) 
    & = \left\{x \in \R^d \colon f_w(x)> r\right\}\\
    & = \left\{x \in \R^d \colon \lgst(\langle x,w\rangle) + \frac{\|w\|^2}{k} > r\right\}\\
    & = \left\{x \in \R^d \colon \log\left(1+e^{-\langle x,w\rangle}\right) > \underbrace{r - \frac{\|w\|^2}{k}}_{=t}\right\}\\
    & = \left\{\begin{array}{ll}
      \R^d   &  t\leq 0\\ 
       \left\{x\in\R^d\colon \langle x,w\rangle > \log (e^t-1)\right\}  &  t > 0,
    \end{array}\right.
\end{align*}
which concludes that $\ranges(\L_{\lgst, k}, \succ)$ only includes half-spaces, empty set, and the whole space $\R^d$.  Therefore, by Radon's theorem, $\vc(\ranges(\L_{\lgst, k}, \succ))\leq d+1$ (for a proof, see Lemma 10.3.1 in ~\citealt{Matousek2002}). Leveraging Theorem~\ref{thm:main_braverman2022new}, we derive the statement for $m = O(\frac{S}{\eps^2}\left(d\log S + \log\frac{1}{\delta}\right))$, thereby completing the proof. 
\end{proof}
In the previous theorem, we assumed that the data is hard-bounded.
In the following results, we try to relax this assumption to a kind of soft bounding similar to what we have in Theorem~\ref{thm:coreset_for_sigmoid}. 
\begin{lemma}\label{lem:sensitivity_logistic_soft_bound}
    Assume that $H$ is a reproducing kernel Hilbert space of real-valued functions on $\XX$ with kernel $\K\colon \XX\times \XX\longrightarrow \R$ and $P$ is a probability measure over $\XX$ such that 
    $\expt\limits_{x\sim p}\|x\|_H^2 \leq E_1^2.$
    Then $$s(y) = 2+ \left(2 + \frac{\|y\|_H}{E_1}\right)
    \frac{680(1+kE_1^2)}{\sqrt{\max(1,\log(2E_1^2k)}}  + \frac{2kE^2_1 + 2kE_1\|y\|_H}{\sqrt{\max(1,\log(2E_1^2k)}}$$ 
    is an upper sensitivity function for $(\XX,P, \L^H_{\lgst,k})$ and $(\XX,P, \bar{\L}^H_{\lgst,k})$ with total sensitivity 
    $S = O\left(\frac{E_1^2k}{\sqrt{\log(2E_1^2k)}}\right).$
\end{lemma}
\begin{proof}
As $\expt_{x\sim P}(\|x\|_H^2)\leq E^2_1$,
Markov's inequality implies 
$\pr_{x\sim P}(\|x\|_H^2\geq 2E^2_1)\leq \frac{1}{2}$. 
For an arbitrary $y\in \XX$, define  
$\XX_y = \left\{x\in\XX\colon \|x\|_H\leq \sqrt{2}E_1\max(1, \frac{\|y\|_H}{E_1})\right\}$. 
Additionally, define $$\W = \left\{w\in R^d\colon \|w\|_H\leq \frac{1}{\sqrt{2}E_1}\sqrt{\max(1,\log(2E_1^2k)}\right\}.$$ 
Lemma~\ref{lem:beta_for_logistic} indicates that Lemma~\ref{lem:beta_property} 
is applicable with  
$$B_1 = \sqrt{2}E_1\max(1, \frac{\|y\|_H}{E_1}), B_2 = \frac{1}{\sqrt{2}E_1}\sqrt{\max(1,\log(2E_1^2k)},$$  
\begin{align*}
    M_y = \phi\left(-B_1B_2\right) & \leq 1 + \max(1, \frac{\|y\|_H}{E_1})\sqrt{\max(1,\log(2E_1^2k)}\\
    & \leq \left(2+\frac{\|y\|_H}{E_1}\right)\sqrt{\max(1,\log(2E_1^2k)}, 
\end{align*} and 
$$\beta(\alpha) = \beta_{\lgst}(\alpha)= \left\{
        \begin{array}{cc}
           \frac{85\alpha^2 k}{\log(\alpha^2 k)}  & \alpha^2 k\geq e\\ \\
           84  & \alpha^2 k \leq e.
        \end{array}
    \right.$$
Consequently, 
\begin{align}
    \sup_{\|w\|_H\leq B_2} \frac{\ell^H_{_{\lgst, k}}(y, w)}{\int_{x\in \XX} \ell^H_{_{\lgst, k}}(x, w)\d P(x)} 
    & = \sup_{\|w\|_H\leq B_2} \frac{1}{\int_{x\in \XX} \frac{\ell^H_{_{\lgst, k}}(x, w)}{\ell^H_{_{\lgst, k}}(y, w)}\d P(x)}\nonumber\\
    & \leq \sup_{\|w\|_H\leq B_2} \frac{1}{\int_{x\in \XX_y} \frac{\ell^H_{_{\lgst, k}}(x, w)}{\ell^H_{_{\lgst, k}}(y, w)}\d P(x)}\nonumber\\
    \text{(Lemma~\ref{lem:beta_property})}\quad
    & \leq \sup_{\|w\|_H\leq B_2} \frac{1}{\int_{x\in \XX_y} \frac{\lgst(0)}{M_y\beta(\|x\|_H)}\d P(x)}\nonumber\\
    & = \frac{M_y}{\lgst(0)}\frac{1}{\int_{x\in \XX_y} \frac{1}{\beta(\|x\|_H)}\d P(x)}\nonumber\\
    (*)\quad
    &\leq 2\left(2 + \frac{\|y\|_H}{E_1}\right)\sqrt{\max(1,\log(2E_1^2k)}
    \frac{340(1+kE_1^2)}{\max(1,\log(2kE_1^2))}.\nonumber\\
    & = \left(2 + \frac{\|y\|_H}{E_1}\right)
    \frac{680(1+kE_1^2)}{\sqrt{\max(1,\log(2E_1^2k)}}.\nonumber
\end{align}
To see ($*$), note that $M_y\leq\left(2+\frac{\|y\|_H}{E_1}\right)\sqrt{\max(1,\log(2E_1^2k)}$, $\lgst(2)\geq 0.5$, and 
\begin{align*}
    \int_{x\in \XX_y} \frac{1}{\beta(\|x\|_H)}\d P(x) 
    &\geq  \int_{\{x: e \leq k\|x\|_H^2\leq 2kE^2_1\}}\frac{\log(\|x\|_H^2 k)}{85\|x\|^2 k}\d P(x)
    + \int_{\{x: k\|x\|_H^2< e\}}\frac{1}{85}\d P(x)\\
    & \geq \frac{\log(2kE^2_1)}{170kE^2_1}P\left(\{x: e \leq k\|x\|_H^2\leq 2kE^2_1\}\right)
    + \frac{1}{85}P\left(\{x: k\|x\|_H^2< e\}\right)\\
    & \geq \frac{\max(1,\log(2kE_1^2))}{170(1+kE_1^2)} P\left(\{x: \|x\|_H^2\leq 2E^2_1\}
    \right)\\
    &\geq \frac{\max(1,\log(2kE_1^2))}{340(1+kE_1^2)}.
\end{align*}
Up to this point, our analysis was concentrated on the scenario where $\|w\|_H\leq B_2$.  Subsequently, we address the case that $\|w\|_H\geq B_2 = \frac{1}{\sqrt{2}E_1}\sqrt{\max(1,\log(2E_1^2k)}$. 
Since $\int_{x\in \XX} \ell^H_{_{\lgst, k}}(x, w)\d P(x)\geq \frac{1}{k}\|w\|_H^2$,  we have  
\begin{align*}
    \sup_{\|w\|_H\geq B_2} \frac{\ell^H_{_{\lgst, k}}(y, w)}{\int_{x\in \XX} \ell^H_{_{\lgst, k}}(x, w)\d P(x)} 
    & \leq \sup_{\|w\|_H\geq B_2}
    \frac{\frac{1}{k}\|w\|_H^2 + \log(1+e^{-\K(y, w)})}{\frac{1}{k}\|w\|_H^2}\\
    & \leq \sup_{\|w\|_H\geq B_2}\left[1 + \frac{k}{\|w\|_H^2}\log(1+e^{-\K(y, w)})\right]\\
    & \leq \sup_{\|w\|_H\geq B_2}\left[1 + \frac{k}{\|w\|_H^2}\log(1+e^{\|y\|_H\|w\|_H})\right]\\
    & \leq \sup_{\|w\|_H\geq B_2}\left[1 + \frac{k}{\|w\|_H^2}\left(1 + \|y\|_H\|w\|_H\right)\right]\\
    &\leq 2 + \frac{k}{B_2^2} + \frac{k\|y\|_H}{B_2}\\
    & = 2 + \frac{2kE_1^2}{\max(1,\log(2E_1^2k)} + \frac{\sqrt{2}kE_1\|y\|_H}{\sqrt{\max(1,\log(2E_1^2k)}}\\
    & \leq 2 + \frac{2kE^2_1 + 2kE_1\|y\|_H}{\sqrt{\max(1,\log(2E_1^2k)}}.
\end{align*}
Therefore, 
\begin{align*}
    \sup_{w\in \XX} \frac{\ell^H_{_{\lgst, k}}(y, w)}{\int_{x\in \XX} \ell^H_{_{\lgst, k}}(x, w)\d P(x)} 
     \leq &\  \max\Biggl\{
     \sup_{\|w\|\leq B_2} \frac{\ell^H_{_{\lgst, k}}(y, w)}{\int_{x\in \XX} \ell^H_{_{\lgst, k}}(x, w)\d P(x)},\\  & \qquad\quad\sup_{\|w\|\geq B_2} \frac{\ell^H_{_{\lgst, k}}(y, w)}{\int_{x\in \XX} \ell^H_{_{\lgst, k}}(x, w)\d P(x)}
    \Biggl\}\\
    \leq &\ \max\left\{\left(2 + \frac{\|y\|_H}{E_1}\right)
    \frac{680(1+kE_1^2)}{\sqrt{\max(1,\log(2E_1^2k)}},\ 2 + \frac{2kE^2_1 + 2kE_1\|y\|_H}{\sqrt{\max(1,\log(2E_1^2k)}}
    \right\}\\
    \leq & 2+ \left(2 + \frac{\|y\|_H}{E_1}\right)
    \frac{680(1+kE_1^2)}{\sqrt{\max(1,\log(2E_1^2k)}}  + \frac{2kE^2_1 + 2kE_1\|y\|_H}{\sqrt{\max(1,\log(2E_1^2k)}}
    = s(y).
\end{align*}
Consequently, $s(y)$ is an upper sensitivity for 
$(\XX,\ P,\  \L^H_{\lgst,k})$. Since $\expt_{y\sim P}(\|y\|_H)\leq \sqrt{\expt_{y\sim P}(\|y\|_H^2)}\leq E_1$, for the total sensitivity, we have
$$S=\int_{y\in\XX} s(y)\d P(y) = O\left(\frac{E_1^2k}{\sqrt{\max(1,\log(2E_1^2k)}}\right).$$
A similar argument concludes the statement for $(\XX, P,  \bar{\L}^H_{\lgst,k})$. 
\end{proof}

\begin{theorem}\label{thm:coreset_for_logistic_unbounded}
    Assume that $H$ is a reproducing kernel Hilbert space of real-valued functions on $\XX$ with kernel $\K\colon \XX\times \XX\longrightarrow \R$,  $P$ is probability measure over $\XX$ such that  
    $\expt\limits_{x\sim p}\|x\|_H^2 \leq E_1$, 
    $$s(x) =2+ \left(2 + \frac{\|y\|_H}{E_1}\right)
    \frac{680(1+kE_1^2)}{\sqrt{\max(1,\log(2E_1^2k)}}  + \frac{2kE^2_1 + 2kE_1\|y\|_H}{\sqrt{\max(1,\log(2E_1^2k)}},$$ and 
    $S=O\left(\frac{E_1^2k}{\sqrt{\max(1,\log(2E_1^2k)}}\right)$.  
    For $m\geq \frac{2S}{\eps^2}\left(8C^2 + S\log\frac{2}{\delta})\right)$, 
    any $s$-sensitivity sample $x_1,\ldots,x_m$ from $\XX$ with weights $u_i = \frac{S}{m s(x_i)}$
    provides an $\eps$-coreset for $(\XX, P, \XX, \ell_{\lgst, k})$ {\rm (}respectively for $(\XX, P, \XX, \bar{\ell}_{\lgst, k})${\rm )} with probability at least $1-\delta$, 
    where 
    $C= (2E_1+1)\max(4E_1k, 2.5) + 8E_1k + 1$.
\end{theorem}
\begin{proof}
    Lemma~\ref{lem:sensitivity_logistic_soft_bound} indicates that $s(\cdot)$ and $S$ are an  
    upper sensitivity function its corresponding total sensitivity for both $(\XX,P, \L^H_{\lgst,k})$ and  $(\XX,P, \bar{\L}^H_{\lgst,k})$. 
    Note that $\lgst(\cdot)$ is decreasing and convex. 
    using Equation~(\ref{eq:main_E_2_for_convex_phi}), we have 
     $$\expt\limits_{x\sim p}\lgst\left(\frac{\|x_i\|_H}{2E_1}\right)\geq 
     \lgst\left(\frac{1}{2}\right)\geq \frac{2}{5} = E_2.$$
    Now, the proof follows by applying Theorem~\ref{thm:monotonic_coreset}.
\end{proof}
Returning to the scenario where our Hilbert space is $\R^d$ equipped with the standard inner product as its kernel, to utilize Theorem~\ref{thm:main_braverman2022new}, the computation of $\vc(\ranges(\mathcal{T}_{{\L_{\lgst, k}}}, \succ))$ is necessary. The following theorem serves as a valuable instrument for handling the VC-dimension of linked-range spaces. 

For a positive integer $l$, a function $f:\R^d\longrightarrow [0,\infty)$ is called \emph{$l$-simply computable}, 
if the inequality $f(x)> r$ can be verified using $O(d^{l-1})$ steps via simple arithmetic operations 
$+, -, \times, /$, jumps conditioned on $>,\ \geq,\ <,\ \leq,\ =,\ \neq$ comparisons on real numbers, and $O(1)$ 
evaluations of the exponential function $t\rightarrow e^t$ on real number $t$. 
\begin{theorem}[\citealt{Martin_Peter2009}, Theorem~8.14]\label{thm:vc_simple_family}
    For a family of functions $\FF$, if each $f\in \FF$ is $l$-simply computable, then 
    $$\vc(\ranges(\FF, \succ)) = O(d^l).$$
\end{theorem}
For each $r\geq 0$, 
$$\range(w, \succ, r) = \left\{x \in \R^d \colon f_w(x)= \frac{\lgst(\langle  w, \cdot\rangle) + \frac{1}{k}\|w\|^2}{s(\cdot)}> r\right\}.$$
The inequality $f_w(x) > r$ is not simply computable since it contains the evaluation $\|x\|$ (to compute $s(x)$), which cannot be done by the allowed operations. To circumvent this issue, we propose a modification to the upper sensitivity function $s$ outlined in Lemma~\ref{lem:sensitivity_logistic_soft_bound} as follows:
\begin{align*}
    s(x) 
    & = 2+ \left(2 + \frac{\|y\|_H}{E_1}\right)
    \frac{680(1+kE_1^2)}{\sqrt{\max(1,\log(2E_1^2k)}}  + \frac{2kE^2_1 + 2kE_1\|y\|_H}{\sqrt{\max(1,\log(2E_1^2k)}}\\
    & =  2+ \left(3 + \frac{\|y\|^2_H}{E^2_1}\right)
    \frac{680(1+kE_1^2)}{\sqrt{\max(1,\log(2E_1^2k)}}  + \frac{2kE^2_1 + 2k(E_1^2+\|y\|^2_H)}{\sqrt{\max(1,\log(2E_1^2k)}}\\
    & = s_1(x).
\end{align*}
Note that $s_1$ is also an upper sensitivity function for $(\R^d,P,\R^d, \ell_{\lgst,k})$ with total sensitivity 
    $S_1 = O\left(\frac{E_1^2k}{\sqrt{\log(E_1^2k)}}\right).$
    Now, if we work with $s_1$ as the upper sensitivity function for $(\R^d,P,\R^d, \ell_{\lgst,k})$, then 
    $f_w(x) = \frac{\lgst(\langle  w, x\rangle) + \frac{1}{k}\|w\|_2^2}{s_1(\cdot)} > r$ is  $2$-simply computable. This is true since  
    $$f_w(x) > r \quad\quad\Longleftrightarrow \quad\quad 
    1+e^{-\langle  w, x\rangle} > \exp\left\{s_1(x)r-\frac{1}{k}\|w\|^2\right\}.$$
    So, using Theorems~\ref{thm:main_braverman2022new} and~\ref{thm:vc_simple_family}, we have the next statement. 
\begin{theorem}\label{thm:coreset_for_logistic_unbounded_d}
    Let $P$ be a probability measure over $\R^d$ such that 
    $\expt\limits_{x\sim p}\|x\|_2^2 \leq E_1^2$, 
    $$s(x) = 2+ \left(3 + \frac{\|y\|^2_H}{E^2_1}\right)
    \frac{680(1+kE_1^2)}{\sqrt{\max(1,\log(2E_1^2k)}}  + \frac{2kE^2_1 + 2k(E_1^2+\|y\|^2_H)}{\sqrt{\max(1,\log(2E_1^2k)}},$$  and   
    $S = O\left(\frac{E_1^2k}{\sqrt{\log(E_1^2k)}}\right).$  
    There is $m = O(\frac{2S}{\eps^2}\left(d^2\log S + \log1/\delta)\right))$ such that  
    any $s$-sensitivity sample $x_1,\ldots,x_m$ from $\XX$ with weights $u_i = \frac{S}{m s(x_i)}$
    provides an $\eps$-coreset for $(\R^d, P, \R^d, \ell_{\lgst, k})$ {\rm (}respectively for $(\R^d, P, \R^d, \bar{\ell}_{\lgst, k})${\rm )} with probability at least $1-\delta$.
\end{theorem}
One might question the necessity of having Theorems~\ref{thm:coreset_for_logistic} and~\ref{thm:coreset_for_sigmoid_R} while Theorems~\ref{thm:coreset_for_logistic_unbounded} and~\ref{thm:coreset_for_logistic_unbounded_d} appear stronger and broader. It is crucial to note that in Theorems~\ref{thm:coreset_for_logistic_unbounded} and~\ref{thm:coreset_for_logistic_unbounded_d}, the process demands re-sampling using $s$-sensitivity samples, which relies on evaluating the function $s$. In contrast, Theorems~\ref{thm:coreset_for_logistic} and~\ref{thm:coreset_for_sigmoid_R} permit the direct (via uniform sampling) use of given data points since they were originally sampled from $P$.

\subsection{Coreset for svm Function}
We would like to remind the reader that $\svm(t) = \max(0,1-t)$. The SVM loss, also referred to as Hinge Loss, serves as a loss function in training classifiers, such as in Support Vector Machines. Following a similar approach to the one outlined above, we initiate the process by identifying a function $\beta(\cdot)$ that satisfies the properties stated in Lemma~\ref{lem:beta_property} for $\phi = \svm$.
 
\begin{lemma}\label{lem:beta_for_svm}
    For $\svm(x) = \max(0,1-t)$ and $k>0$, we have 
    $$\frac{\svm(-\alpha z)+\frac{z^2}{k}}{\svm(\alpha z)+\frac{z^2}{k}}
    \leq 1+2\max(1, \alpha^2 k) \quad\quad \forall\alpha, z \geq 0.$$
\end{lemma}
\begin{proof}
    For a fixed $\alpha>0$, set $t =\alpha z$, $K = \alpha^2 k$, 
    and 
    $$\frac{\svm(-\alpha z)+\frac{z^2}{k}}{\svm(\alpha z)+\frac{z^2}{k}} = \frac{\svm(-t)+\frac{t^2}{K}}{\svm(t)+\frac{t^2}{K}} = f_K(t)\quad \quad t\geq 0.$$
    We can rewrite $f_K$ as 
    \begin{align}
        f_K(t) 
        & = \left\{
            \begin{array}{ll}
              \frac{1+t+\frac{t^2}{K}}{\frac{t^2}{K}}   &  1 \leq t\\ \\ 
              \frac{1+t+\frac{t^2}{K}}{1-t + \frac{t^2}{K}}   &  0\leq t\leq 1
            \end{array}\right.\\ 
        & = \left\{
            \begin{array}{ll}
              \frac{K}{t^2}+\frac{K}{t}+1  &  1 \leq t\\ \\ 
              1+\frac{2t}{1-t + \frac{t^2}{K}}   &  0\leq t\leq 1
            \end{array}\right.
    \end{align}
    One can simply check that $\max_{t\geq 1}f_K(t) = 1+2K$.  
    Also, if we assume that $K\geq 1$, then, by simple calculation, we can derive that 
    $\max_{0\leq t\leq 1}f_K(t) = f_K(1) =  1+2K$. So, using Lemma~\ref{lem:f_increase_K}, we obtain 
    $$\sup_{t\geq 0}f_K(t) \leq 1+2\max(1,K) \quad\quad \forall K>0.$$
    Replacing $K$ with $\alpha^2 z$ concludes the result.
\end{proof}
\begin{lemma}\label{lem:sensitivity_svm_hard_bound}
    Assume that $H$ is a reproducing kernel Hilbert space of real-valued functions on $\XX$ with kernel $\K\colon \XX\times \XX\longrightarrow \R$ and $P$ is a probability measure over $\XX$ such that  
    $P\left(\left\{x\in\XX\colon \|x\|_H\geq A\right\}\right)= 0$. 
    Then $s(y) = 6+4kA^2$ is an upper sensitivity function for $(\XX,\ P,\ \L^H_{\svm,k})$ and $(\XX,\ P,\ \bar{\L}^H_{\svm,k})$.
\end{lemma}
\begin{proof}
Since $P\left(\left\{x\in\XX\colon \|x\|_H\geq A\right\}\right)= 0$, our computation of $s(y)$ is limited to $\|y\|_H\leq A$. 
For $w\in \W = \{w\in \XX\colon \|w\|_H\leq \frac{1}{A}\}$, 
Lemma~\ref{lem:beta_for_svm} indicates that Lemma~\ref{lem:beta_property} 
is applicable with  $B_1 = A, B_2 = \frac{1}{A}, M = \phi(-1) =  2$, and 
$\beta_{\svm}(\alpha)= 1+2\max(1, \alpha^2 k).$
Using Lemma~\ref{lem:sensitivity_beta} with 
$\gamma(\|x\|_H) = \frac{\svm(0)}{M \beta_{\svm}(\|x\|_H)} = \frac{1}{2 \beta_{\svm}(\|x\|_H)}$,
we conclude that 
$$\sup_{\|w\|_H\leq \frac{1}{A}} \frac{\ell^H_{_{\svm, k}}(y, w)}{\int_{x\in \XX} \ell^H_{_{\svm, k}}(x, w)\d P(x)}\leq \frac{2}{\int_\XX\frac{1}{\beta_{\svm}(\|x\|_H)}\d P(x)}.$$ 
Note that  
\begin{align*}
    \int_{x\in \XX}\frac{1}{\beta_{\svm}(\|x\|_H)}\d P(x) 
    &\geq  \int_{\{x: 1 \leq k\|x\|_H^2\}}\frac{1}{1+2k\|x\|_H^2}\d P(x)
    + \int_{\{x: k\|x\|^2< 1\}}\frac{1}{3}\d P(x)\\
    & \geq \frac{1}{1+2kA^2}P\left(\left\{x: \frac{1}{k} \leq \|x\|_H^2\right\}\right)
    + \frac{1}{3}P\left(\left\{x: \|x\|_H^2< \frac{1}{k}\right\}\right)\\
    & \geq \frac{1}{3+2kA^2}.
\end{align*}
This concludes 
$$\sup_{\|w\|_H\leq \sqrt{k}} \frac{\ell^H_{_{\svm, k}}(y, w)}{\int_{x\in \XX} \ell^H_{_{\svm, k}}(x, w)\d P(x)}\leq
 6+4kA^2.$$
Since  
$\int_{x\in \R^d} \ell_{_{\svm, k}}(x, w)\d P\geq \frac{1}{k}\|w\|^2$,  we have  
\begin{align*}
    \sup_{\|w\|_H\geq \frac{1}{A}} \frac{\ell_{_{\svm, k}}(y, w)}{\int_{x\in \XX} \ell_{_{\svm, k}}(x, w)\d P} 
    & \leq \sup_{\|w\|_H\geq \frac{1}{A}}\frac{\frac{1}{k}\|w\|_H^2 + \max(0, 1-\langle y, w\rangle_H)}{\frac{1}{k}\|w\|_H^2}\\
    & = \sup_{\|w\|_H\geq \frac{1}{A}} \left[1 + \frac{k}{\|w\|_H^2}\max(0, 1-\langle y, w\rangle_H)\right]\\
    & \leq \sup_{\|w\|_H\geq \frac{1}{A}} \left[1 + \frac{k}{\|w\|_H^2}\max(0, 1+\|y\|_H\|w\|_H)\right]\\
    & \leq \sup_{\|w\|_H\geq \frac{1}{A}} \left[1 + \frac{k}{\|w\|_H^2}\left(1 + \|y\|_H\|w\|_H\right)\right]\\
    &\leq 1 + 2A^2 k.
\end{align*}
So, $s(y) = 6+4kA^2$ is an upper sensitivity 
function for $(\XX,\ P,\ \L^H_{\svm,k})$.
\end{proof}

\begin{theorem}[]\label{thm:coreset_for_svm}
    Assume that $H$ is a reproducing kernel Hilbert space of real-valued functions on $\XX$ with kernel $\K\colon \XX\times \XX\longrightarrow \R$ and $P$ is a probability measure over $\XX$ such that  
    $P\left(\left\{x\in\XX\colon \|x\|_H\geq A\right\}\right)= 0$. 
    For $S= 6+4kA^2$, $C= (4A+2)\max(2Ak, 1)+8Ak+1$, and $m = \frac{2S}{\eps^2}(8C^2 + S\log\frac{2}{\delta})$, any iid sample $x_1,\ldots,x_m$ from $\XX$ with weights $u_i = \frac{1}{m}$ provides an $\eps$-coreset for $(\XX, P, \XX, \L^H_{\svm, k})$ 
    {\rm (}respectively for $(\XX, P, \XX, \bar{\ell}^H_{\svm, k})${\rm )} with probability at least $1-\delta$.
\end{theorem}
\begin{proof}
    Lemma~\ref{lem:sensitivity_svm_hard_bound} indicates that  $s(y) = 6+4kA^2$ is an upper sensitivity function. 
    As it is a constant function, $s$-sensitivity sampling is the same as  sampling according to $P$.
    Notice that $\svm$ is convex and non-increasing. 
    For $E_1=A$, using Equation~\ref{eq:main_E_2_for_convex_phi}, we have 
     $$\expt\limits_{x\sim p}\svm\left(\frac{\|x_i\|_2}{2E_1}\right)\geq 
     \svm\left(\frac{1}{2}\right)\geq \frac{1}{2} = E_2.$$
    This implies that Definition~\ref{def:well_beaived_P1} is satisfied for $\L^H_{\svm, k}$ and thus  
    Theorem~\ref{thm:monotonic_coreset1} concludes the statement for $m\geq \frac{2S}{\eps^2}\left(8C^2 + \log\frac{1}{\delta}\right)$.
\end{proof}
Following directly from this theorem, we observe the  following corollary.
\begin{corollary}[Theorem~\ref{thm:kde}, Restated]\label{cor:kde}
    Let $\K:\R^d\times \R^d\longrightarrow (0,1]$ be a reproducing kernel, i.e., a kernel associated with an RKHS, and $P$ be a probability measure over $\R^d$. For $\eps,\delta\in(0,1)$, there exists a universal constant $C$ (independent of $d$ and $\K$) such that if $m \geq \frac{C}{\eps^2}\log \frac{1}{\delta}$, then, with probability at least $1-\delta$, for any random sample $X = \{x_1,\ldots,x_m\}$ based on $P$, we have 
    $$\sup_{w\in\R^d}|\kde_P(w) - \kde_X(w)|\leq \eps.$$
\end{corollary}
\begin{proof}
    Note that $\K$ and $\phi(t) = \max(0,1-t)$ satisfy Definition~\ref{def:well_beaived_P1} with $E_1 = 1$, $E_2 = \frac{1}{2}$, and 
    $$\ell^H_{\phi, k}(x,w) = \phi(\K(x,w)) + \frac{1}{k}\|w\|_H^2 = 1 - \K(x,w) + \frac{1}{k}\|w\|_H^2\leq  1 - \K(x,w) + \frac{1}{k}.$$ 
    If set $k=1$, then, with probability at least $1-\delta$, for each $w\in\R^d$, we have 
    \begin{align*}
        \left|\int_{x\in \R^d} \K(x,w)\d P(x)- \frac{1}{m}\sum_{i=1}^m \K(x_i,w)\right|
        & = \left|\int_{x\in \R^d} \ell^H_{\phi, 1}(x,w)\d P(x)- \frac{1}{m}\sum_{i=1}^m \ell^H_{\phi, 1}(x_i,w)\right| \\
        \text{(by Theorem~\ref{thm:coreset_for_svm})}\quad\quad \quad 
        &\leq \frac{\eps}{2} \int_{x\in \R^d} \ell^H_{\phi, 1}(x,w)\d P(x)\\
        &\leq \frac{\eps}{2}\int_{x\in \R^d} (2 - \K(x,w))\d P(x)\\
        \text{(since $\K(x,w)\leq 1$)}\quad\quad \quad
        &\leq \eps.
    \end{align*}
\end{proof}

\begin{theorem}\label{thm:coreset_for_svm_d}
    Let $P$ be a probability measure over $\R^d$ such that 
    $P\left(\left\{x\in\XX\colon \|x\|_H\geq A\right\}\right)= 0$ and $S= 6+4kA^2$. 
    There exists $m = O\left(\frac{S}{\eps^2}(d\log S + \log\frac{1}{\delta})\right)$ such that any iid sample $x_1,\ldots,x_m$ from $P$ with weights $u_i = \frac{1}{m}$
provides an $\eps$-coreset for $(\R^d, P, \R^d, \ell_{\svm, k})$ 
{\rm (}respectively for $(\XX, P, \XX, \bar{\ell}^H_{\svm, k})${\rm }with probability at least $1-\delta$.
\end{theorem}
\begin{proof}
    Given that Lemma~\ref{lem:sensitivity_svm_hard_bound} asserts the constant function $s(x) = 3+2kA^2 + \sqrt{k}A$ as an upper sensitivity function for $\L_{\svm, k}$, we deduce $\ranges(\mathcal{T}_{\L_{\svm, k}}, \succ) = \ranges(\L_{\svm, k}, \succ)$. Notice $\ranges(\mathcal{T}_{\L_{\svm, k}}, \succ) = \ranges(\L_{\svm, k}, \succ)$ 
    since each function in $\mathcal{T}_{\L_{\svm, k}}$ is a function in $\L_{\svm, k}$ scaled by a positive constant. For an $f_w\in \L_{\svm, k}$ and $r\geq 0$,
\begin{align*}
    \range(f_w, \succ, r) 
    & = \left\{x \in \R^d \colon f_w(x)> r\right\}\\
    & = \left\{x \in \R^d \colon \max(0, 1-\langle x,w\rangle) + \frac{\|w\|^2}{K} > r\right\}\\
    & = \left\{x \in \R^d \colon \max(0, 1-\langle x,w\rangle) > \underbrace{r - \frac{\|w\|^2}{K}}_{=t}\right\}\\
    & = \left\{\begin{array}{ll}
      \R^d   &  t<0\\ \\
       \left\{x\in\R^d\colon \langle x,w\rangle< 1-t\right\}  & t\geq 0,
    \end{array}\right.
\end{align*}
which concludes that $\ranges(\L_{\svm, k}, \succ)$ only includes half-spaces and the whole space $\R^d$.  Therefore, by Radon's theorem, $\vc(\ranges(\L_{\svm, k}, \succ))\leq d+1$ (for a proof, see Lemma 10.3.1 in ~\cite{Matousek2002}).  
    Using Theorem~\ref{thm:main_braverman2022new}, we also have the statement of Theorem~\ref{thm:coreset_for_logistic} for $m \geq O(\frac{S}{\eps^2}\left(d\log S + \log\frac{1}{\delta}\right))$, completing the proof. A similar argument establishes the result for $(\XX,P, \bar{\L}^H_{\svm,k})$.
\end{proof}

Continuing in line with prior sections, we aim to establish a theorem that addresses-soft bounding on data instead of employing a hard boundary. In the subsequent discussion, we substantiate such a result.

\begin{lemma}\label{lem:sensitivity_svm_hard_bound}
    Assume that $H$ is a reproducing kernel Hilbert space of real-valued functions on $\XX$ with kernel $\K\colon \XX\times \XX\longrightarrow \R$ and $P$ is a probability measure over $\XX$ such that 
    $\expt_{x\sim p} \|x\|_H\leq E_1^2$. 
    Then $s(x) = (6+8kE_1^2)\left(2 + \frac{\|y\|_H}{E_1}\right)$ is an upper sensitivity function for $(\XX, P, \L^H_{\svm,k})$ and $(\XX,P, \bar{\L}^H_{\svm,k})$ with total sensitivity at most $18+24kE_1^2$.
\end{lemma}
\begin{proof}
As $\expt_{x\sim P}(\|x\|_H^2)\leq E^2_1$,
Markov's inequality implies 
$\pr_{x\sim P}(\|x\|_H^2\geq 2E^2_1)\leq \frac{1}{2}$. 
For an arbitrary $y\in \XX$, define  
$\XX_y = \left\{x\in\XX\colon \|x\|_H\leq \sqrt{2}\max(E_1, \|y\|_H)\right\}$. Additionally, define $\W = \{w\in \XX\colon \|w\|_H\leq \frac{1}{\sqrt{2}E_1}\}$. 
Lemma~\ref{lem:beta_for_logistic} indicates that Lemma~\ref{lem:beta_property} 
is applicable with  $B_1 = \sqrt{2}\max(E_1, \|y\|_H)$, $B_2 = \frac{1}{\sqrt{2}E_1}$, $\svm(-B_1B_2) = 1 + \max\left(1, \frac{\|y\|_H}{E_1}\right)\leq 2 + \frac{\|y\|_H}{E_1} = M_y$, and  
$$\beta(\alpha) = \beta_{\svm}(\alpha)= 1+2\max(1, \alpha^2 k).$$
Consequently,  
\begin{align}
    \sup_{\|w\|_H\leq B_2} \frac{\ell^H_{_{\svm, k}}(y, w)}{\int_{x\in \XX} \ell^H_{_{\svm, k}}(x, w)\d P(x)} 
    & = \sup_{\|w\|_H\leq B_2} \frac{1}{\int_{x\in \XX} \frac{\ell^H_{_{\svm, k}}(x, w)}{\ell^H_{_{\svm, k}}(y, w)}\d P(x)}\nonumber\\
    & \leq \sup_{\|w\|_H\leq B_2} \frac{1}{\int_{x\in \XX_y} \frac{\ell^H_{_{\svm, k}}(x, w)}{\ell^H_{_{\svm, k}}(y, w)}\d P(x)}\nonumber\\
    \text{(Lemma~\ref{lem:beta_property})}\quad
    & \leq \sup_{\|w\|_H\leq B_2} \frac{1}{\int_{x\in \XX} \frac{\svm(0)}{M_y\beta(\|x\|_H)}\d P(x)}\nonumber\\
    & = \frac{M_y}{\int_{x\in \XX_y} \frac{1}{\beta(\|x\|_H)}\d P(x)}\nonumber\\
    & \leq \frac{2 + \frac{\|y\|_H}{E_1}}{\int_{x\in \XX_y} \frac{1}{\beta(\|x\|_H)}\d P(x)}\nonumber\\
    (*)\quad
    &\leq (6+8kE_1^2)\left(2 + \frac{\|y\|_H}{E_1}\right).\nonumber
\end{align}
To see ($*$), note that 
\begin{align*}
    \int_{x\in \XX_y} \frac{1}{\beta(\|x\|_H)}\d P 
    &\geq  \int_{\{x: 1 \leq k\|x\|_H^2\leq 2kE^2_1\}}\frac{1}{1+2k\|x\|_H^2}\d P
    + \int_{\{x: k\|x\|_H^2< 1\}}\frac{1}{3}\d P\\
    & \geq \frac{1}{1+4kE_1^2}P\left(\{x: 1 \leq k\|x\|_H^2\leq 2kE^2_1\}\right)
    + \frac{1}{3}P\left(\{x: k\|x\|_H^2< 1\}\right)\\
    & \geq \min\left(\frac{1}{1+4kE_1^2},\frac{1}{3} \right)P\left(\{x: \|x\|_H^2\leq 2E^2_1\}\right)\\
    & \geq 
    \frac{1}{6+8kE_1^2}.
\end{align*}
As we have only considered the case that $\|w\|_H\leq B_2$, 
to complete the proof, we should work with $\|w\|_H\geq B_2$ as well. 
Since $\int_{x\in \XX} \ell^H_{_{\svm, k}}(x, w)\d P\geq \frac{1}{k}\|w\|_H^2$,  we have  
\begin{align*}
    \sup_{\|w\|_H\geq B_2} \frac{\ell^H_{_{\svm, k}}(y, w)}{\int_{x\in \XX} \ell^H_{_{\svm, k}}(x, w)\d P} 
    & \leq \sup_{\|w\|_H\geq B_2}
    \frac{\frac{1}{k}\|w\|_H^2 + \max(0, 1-\langle y, w\rangle_H)}{\frac{1}{k}\|w\|_H^2}\\
    & \leq \sup_{\|w\|_H\geq B_2} \left[1 + \frac{k}{\|w\|_H^2}\max(0, 1-\langle y, w\rangle_H)\right]\\
    & \leq \sup_{\|w\|_H\geq B_2} \left[1 + \frac{k}{\|w\|_H^2}\max(0, 1+\|y\|_H\|w\|_H)\right]\\
    & = \sup_{\|w\|_H\geq B_2} \left[1 + \frac{k}{\|w\|_H^2}\left(1 + \|y\|_H\|w\|_H\right)\right]\\
    &\leq 1 + 2kE_1^2 + \sqrt{2}kE_1\|y\|_H.
\end{align*}
Therefore, 
\begin{align*}
    \sup_{w\in \XX} \frac{\ell^H_{_{\svm, k}}(y, w)}{\int_{x\in \R^d} \ell^H_{_{\svm, k}}(x, w)\d P} 
     \leq &  \max\Biggl\{
     \sup_{\|w\|\leq \sqrt{k}} \frac{\ell^H_{_{\svm, k}}(y, w)}{\int_{x\in \R^d} \ell^H_{_{\svm, k}}(x, w)\d P},\\  & \qquad\quad\sup_{\|w\|\geq \sqrt{k}} \frac{\ell^H_{_{\svm, k}}(y, w)}{\int_{x\in \R^d} \ell^H_{_{\svm, k}}(x, w)\d P}
    \Biggl\}\\
    \leq & \max\left\{(6+8kE_1^2)\left(2 + \frac{\|y\|_H}{E_1}\right),\ 1 + 2kE_1^2 + \sqrt{2}kE_1\|y\|_H
    \right\}\\
    = & (6+8kE_1^2)\left(2 + \frac{\|y\|_H}{E_1}\right) = s(y).
\end{align*}
Consequently, 
$s(y)$ is an upper sensitivity for $(\R^d,P,\R^d, \ell_{\svm,k})$.
Since $\expt_{y\sim P}(\|y\|_H)\leq \sqrt{\expt_{y\sim P}(\|y\|_H^2)}\leq E_1$, for the total sensitivity, we have
$$S=\int s(y)\d P\leq (18+24kE_1^2)=O(E_1^2k).$$
\end{proof}
With a similar proof, we can restate the following versions of Theorems~\ref{thm:coreset_for_logistic_unbounded} and~\ref{thm:coreset_for_logistic_unbounded_d} replacing $\lgst$ by $\svm$. 
\begin{theorem}\label{thm:coreset_for_svm_unbounded}
    Assume that $H$ is a reproducing kernel Hilbert space of real-valued functions on $\XX$ with kernel $\K\colon \XX\times \XX\longrightarrow \R$ and $P$ is a probability measure over $\XX$ such that 
    $\expt\limits_{x\sim p}\|x\|_H^2 \leq E_1$, $$s(x) = (6+8kE_1^2)\left(2 + \frac{\|y\|_H}{E_1}\right),$$ and  
    $S=O(E_1^2k)$.  
    For $m\geq \frac{2S}{\eps^2}\left(8C^2 + S\log\frac{2}{\delta}\right)$, 
    any $s$-sensitivity sample $x_1,\ldots,x_m$ from $\XX$ with weights $u_i = \frac{S}{m s(x_i)}$
    provides an $\eps$-coreset for $(\XX, P, \XX, \ell^H_{\svm, k})$ {\rm )} respectively for $(\XX, P, \XX, \bar{\ell}^H_{\svm, k})${\rm )} with probability at least $1-\delta$, 
    where 
    $C= (4E_1+2)\max\left(E_1k, 1\right) + 8E_1k+1 $.
\end{theorem}
\begin{theorem}\label{thm:coreset_for_svm_unbounded_d}
    Let $P$ be a probability measure over $\R^d$ such that 
    $\expt\limits_{x\sim p}\|x\|_2^2 \leq E_1^2$, $s(x) = (6+8kE_1^2)\left(2 + \frac{\|y\|_H}{E_1}\right),$ and 
    $S=O\left(E_1^2k\right)$.  
    For $m\geq O\left(\frac{S}{\eps^2}\left(d^2\log S + \log\frac{1}{\delta}\right)\right)$, 
    any $s$-sensitivity sample $x_1,\ldots,x_m$ from $\XX$ with weights $u_i = \frac{S}{m s(x_i)}$
    provides an $\eps$-coreset for $(\R^d, P, \R^d, \ell_{\svm, k})$ {\rm )} respectively for $(\R^d, P, \R^d, \bar{\ell}^H_{\svm, k})${\rm )} with probability at least $1-\delta$.
\end{theorem}
\subsection{Coreset for $\relu$ Function}\label{sec:relu}
For $\relu(t) = \max(0, t)$, we derive our results from those in previous section using a simple trick.
Assume that $H$ is a reproducing kernel Hilbert space of real-valued functions on $\XX$ with kernel $\K\colon \XX\times \XX\longrightarrow \R$ and $P$ is a probability measure over $\XX$ such that $\expt\limits_{x\sim p}\|x\|_H^2 \leq E^2_1$ or 
$P\left(\left\{x\in\XX\colon \|x\|_H\geq A\right\}\right)= 0$. 
It is easy to verify that $\K'(\cdot) = 1+\K(\cdot)$ is also a reproducing kernel associated with $H'$  with $\expt\limits_{x\sim p}\|x\|_{H'}^2 \leq (1+E_1)^2$ or $P\left(\left\{x\in\XX\colon \|x\|_{H'}\geq 1+A\right\}\right)= 0$. As $\svm(1+t) = \max(0,-t)$, if we plug $\K'$ in the results given in the previous section, we can replicate them for $\phi(t) = \max(0,-t)$ by replacing $E_1$ and $A$ with $(1+E_1)$ and $(1+A)$. While $\phi$ is not precisely the $\relu$ function, for the case that the kernel is the standard linear kernel in Euclidean space, 
$$\ell_{\phi, k}(x,w) = \max(0,-\langle x,w\rangle) + \frac{1}{k}\|w\|_2^2 = \max(0,\langle x, -w\rangle) + \frac{1}{k}\|w\|_2^2 = \ell_{\relu, k}(x, -w).$$
This implies that if we have a coreset for $(\R^d, P, \R^d, \ell_{\phi, k})$, it automatically serves as a coreset for  $(\R^d, P, \R^d, \ell_{\relu, k})$ as well.
It is noteworthy that this conclusion cannot be extended to general kernels as $-\K(x,w) = \K(x,-w)$ is not generally true.

\end{document}